\documentclass[11pt]{article}

\date{}
\usepackage[utf8]{inputenc} 
\usepackage[T1]{fontenc}    
\usepackage{url}            
\usepackage{booktabs}       
\usepackage{amsfonts}       
\usepackage{nicefrac}       
\usepackage{microtype}      

\usepackage{amsmath}
\usepackage{amssymb}
\usepackage{nccmath}
\usepackage{mathtools}
\usepackage{xcolor}
\usepackage{graphicx} 
\usepackage{natbib}

\usepackage{wrapfig}
\usepackage{amsthm}
\usepackage{enumitem}
\usepackage{mwe}
\usepackage{subfig}
\usepackage[export]{adjustbox}
\usepackage{times}
\usepackage{algorithmicx}
\usepackage{algorithm} %
\usepackage[noend]{algpseudocode}
\usepackage[breaklinks=true]{hyperref}
\usepackage{breakcites}

\newtheorem{theorem}{Theorem}
\newcommand{\bigzero}{\mbox{\normalfont\Large\bfseries 0}}
\newcommand\x{\times}
\usepackage{mdframed}

\usepackage{tikz}

\usetikzlibrary{shapes,arrows}
\usetikzlibrary{arrows,calc,positioning}
\usetikzlibrary{fit}
\tikzstyle[nome]=[anchor=west]
\tikzset{
    block/.style = {draw, rectangle,
        minimum height=0.4cm,
        minimum width=0.5cm},
    input/.style = {coordinate,node distance=1cm},
    output/.style = {coordinate,node distance=4cm},
    arrow/.style={draw, -latex,node distance=2cm},
    pinstyle/.style = {pin edge={latex-, black,node distance=2cm}},
    sum/.style = {draw, circle,minimum size=1pt, node distance=1cm},}

\makeatletter

\newcommand{\LQG}{\textsc{\small{LQG}}\xspace}

\newcommand{\Hes}{H_{es}}

\newcommand{\sigmaW}{\sigma_{w}}
\newcommand{\sigmaZ}{\sigma_{z}}
\newcommand{\sigmaWup}{\wb\sigma_{w}}
\newcommand{\sigmaZup}{\wb\sigma_{z}}
\newcommand{\sigmaWlow}{\underline\sigma_{w}}
\newcommand{\sigmaZlow}{\underline\sigma_{z}}

\newcommand{\strong}{\underline\alpha_{loss}}
\newcommand{\smooth}{\wb\alpha_{loss}}
\newcommand{\Markov}{\mathbf{G}}
\newcommand{\Mcontrol}{\mathbf{M}}
\newcommand{\Mcontrolset}{\mathcal{M}}
\newcommand{\proj}{proj}

\newcommand{\nature}{b}
\newcommand{\Tburn}{{T_{burn}}}
\newcommand{\Tbase}{{T_{base}}}
\newcommand{\Tmax}{{T_{\max}}}

\newcommand{\Sys}{\textsc{\small{SysId}}\xspace}

\newcommand{\Gol}{\mathcal{G}^{ol}}

\newcommand{\modelearn}{\mathbf{\mathcal{G}_{yu}}}
\newcommand{\modelearnf}{\mathbf{\widehat{\mathcal{G}}_{yu1}}}
\newcommand{\modelearni}{\mathbf{\widehat{\mathcal{G}}_{yui}}}

\newcommand{\RL}{\textsc{\small{RL}}\xspace}
\newcommand{\LDC}{\textsc{\small{LDC}}\xspace}
\newcommand{\DFC}{\textsc{\small{DFC}}\xspace}

\DeclareMathOperator*{\argmin}{arg\,min}

\newcommand{\sig}{\Sigma}
\newcommand{\reg}{\textsc{\small{Regret}}\xspace}
\newcommand{\OO}{\mathcal{O}}

\DeclareMathOperator{\Tr}{Tr}

\newcommand{\alg}{\textsc{\small{AdaptOn}}\xspace}

\newcommand{\T}{\mathcal T}

\newcommand{\R}{\mathbb{R}}

\newcommand{\boldR}{\mathbb R}

\newcommand{\wt}{\widetilde}
\newcommand{\wh}{\widehat}
\newcommand{\wb}{\overline}

\newtheorem{lemma}{Lemma}[section]

\newtheorem{theorem*}{Theorem}[section]
\newtheorem{corollary}{Corollary}[section]
\newtheorem*{condition}{Persistence of Excitation of $\Mcontrol \in \Mcontrolset$ on System $\Theta$}

\newtheorem{definition}{Definition}[section]

\title{Logarithmic Regret Bound in Partially Observable Linear Dynamical Systems}

\author{%
  Sahin Lale$^1$, Kamyar Azizzadenesheli$^2$, Babak Hassibi$^1$, Anima Anandkumar$^2$\\
  $^1$~Department of Electrical Engineering\\
  $^2$~Department of Computing and Mathematical Sciences\\
  California Institute of Technology, Pasadena\\
  \texttt{\{alale,kazizzad,hassibi,anima\}@caltech.edu}
}

\begin{document}

\maketitle

\begin{abstract}%
We study the problem of system identification and adaptive control in partially observable linear dynamical systems. Adaptive and closed-loop system identification is a challenging problem due to correlations introduced in data collection. In this paper, we present the first model estimation method with finite-time guarantees in both open and closed-loop system identification. Deploying this estimation method, we propose adaptive control online learning (\alg), an efficient reinforcement learning algorithm that adaptively learns the system dynamics and continuously updates its controller through online learning steps. \alg estimates the model dynamics by occasionally solving a linear regression problem through interactions with the environment. Using policy re-parameterization and the estimated model, \alg constructs counterfactual loss functions to be used for updating the controller through online gradient descent. Over time, \alg improves its model estimates and obtains more accurate gradient updates to improve the controller. We show that  \alg achieves a regret upper bound of $\text{polylog}\left(T\right)$, after $T$ time steps of agent-environment interaction. To the best of our knowledge, \alg is the first algorithm that achieves $\text{polylog}\left(T\right)$ regret in adaptive control of \emph{unknown} partially observable linear dynamical systems which includes linear quadratic Gaussian (\LQG) control. 
\end{abstract}


\section{Introduction}\label{Introduction}

Reinforcement learning (\RL) in unknown partially observable linear dynamical systems with the goal of minimizing a cumulative cost is one of the central problems in adaptive control~\citep{bertsekas1995dynamic}. In this setting, a desirable \RL agent needs to efficiently \textit{explore} the environment to learn the system dynamics, and \textit{exploit} the gathered experiences to minimize overall cost~\citep{lavalle2006planning}. However, since the underlying states of the systems are not fully observable, learning the system dynamics with finite time guarantees brings substantial challenges, making it a long-lasting problem in adaptive control. In particular, when the latent states of a system are not fully observable, future observations are correlated with the past inputs and observations through the latent states. These correlations are even magnified when closed-loop controllers, those that naturally use past experiences to come up with control inputs, are deployed. Therefore, more sophisticated estimation methods that consider these complicated and unknown correlations are required for learning the dynamics.

In recent years, a series of works have studied this learning problem and presented a range of novel methods with finite-sample learning guarantees. These studies propose to employ i.i.d. Gaussian excitation as the control input, collect system outputs and estimate the model parameters using the data collected. The use of i.i.d. Gaussian noise as the open-loop control input (not using past experiences) mitigates the correlation in the inputs and the output observations. For stable systems, these methods provide efficient ways to learn the model dynamics with confidence bounds of $\Tilde{\OO}(1/\sqrt{T})$, after $T$ times step of agent-environment interaction~\citep{oymak2018non,sarkar2019finite,tsiamis2019finite, simchowitz2020improper}. Here $\Tilde{\OO}(\cdot)$ denotes the order up to logarithmic factors. Deploying i.i.d. Gaussian noise for a long period of time to estimate the model parameters has been the common practice in adaptive control since incorporating closed-loop controller introduces significant challenges to learn the model dynamics \citep{qin2006overview}.


These estimation techniques later have been deployed to propose explore-then-commit based \RL algorithms to minimize regret, \textit{i.e.}, how much more cost an agent suffers compared to the cost of a baseline policy~\citep{lai1982least}.
These algorithms deploy i.i.d. Gaussian noise as the control input to learn the model parameters in the \textit{explore} phase and then exploit these estimates during the \textit{commit} phase to minimize regret. Among these works, \citet{lale2020regret} and \citet{simchowitz2020improper} respectively propose to use optimism~\citep{auer2002using} and online convex optimization~\citep{anava2015online}
during the commit phase. These works attain regret of $\Tilde{\mathcal{O}}(T^{2/3})$ in the case of convex cost functions. Moreover, in the case of strongly convex cost functions, \citet{mania2019certainty,simchowitz2020improper} show that exploiting the strong convexity allows to guarantee regret of $\Tilde{\mathcal{O}}(\sqrt{T})$. These methods heavily rely on the lack of correlation achieved by using i.i.d. Gaussian noise as the open-loop control input to estimate the model. Therefore, they do not generalize to the adaptive settings, where the past observations are used to continuously improve both model estimates and the controllers. These challenges pose the following two open problems:
\begin{center}
    ``\textit{Can we estimate the model parameters in closed-loop setting with finite-time guarantees?}''

``\textit{Can we utilize such an estimation method, and propose an RL algorithm to significantly improve regret in partially observable linear dynamical systems?}''

\end{center}

In this paper, we give \textbf{affirmative} answers to both of these questions:
\begin{itemize}[wide, labelindent=0pt]
    \item \textbf{Novel closed-loop estimation method:} We introduce the first system identification method that allows to estimate the model parameters with finite-time guarantees in both open and closed-loop setting. We exploit the classical predictive form representation of the system that goes back to \citet{kalman1960new} and reformulate each output as a linear function of previous control inputs and outputs with an additive i.i.d. Gaussian noise. This reformulation allows to address the limitations of the prior open-loop estimation methods in handling the correlations in inputs and outputs. We state a novel least squares problem to recover the model parameters. We show that when the controllers persistently excite the system, the parameter estimation error is $\Tilde{\OO}(1/\sqrt{T})$ after $T$ samples. Our method allows updating the model estimates while controlling the system with an adaptive controller.

    
    \item \textbf{Novel \RL algorithm for partially observable linear dynamical systems:} Leveraging our novel estimation method, we propose \textbf{adapt}ive control \textbf{on}line learning algorithm (\alg) that \emph{adaptively} learns the model dynamics and efficiently uses the model estimates to continuously optimize the controller and reduce the cumulative cost. \alg operates in growing size epochs and in the beginning of each epoch estimates the model parameters using our novel model estimation method. During each epoch, following the online learning procedure introduced by \citet{simchowitz2020improper}, \alg utilizes a convex policy reparameterization of linear controllers and the estimated model dynamics to construct counterfactual loss functions. \alg then deploys online gradient descent on these loss functions to gradually optimize the controller. We show that as the model estimates improve, the gradient updates become more accurate, resulting in improved controllers. 
    
    We show that \alg attains a regret upper bound of $\text{polylog}(T)$ after $T$ time steps of agent-environment interaction, when the cost functions are strongly convex. To the best of our knowledge, this is the first logarithmic regret bound for partially observable linear dynamical systems with unknown dynamics which includes the canonical \LQG setting. The presented regret bound improves $\Tilde{\OO}(\sqrt{T})$ regret of \citet{simchowitz2020improper,mania2019certainty} in stochastic setting with the help of novel estimation method which allows updating model estimates during control~(Table \ref{table:1}).
\end{itemize}

\begin{table}
\centering
\caption{Comparison with prior works in partially observable linear dynamical systems.}
 \begin{tabular}{l l l  l l } 
 \toprule
 \textbf{Work} & \textbf{Regret} &  \textbf{Cost} &
 \textbf{Identification} & \textbf{Noise}  \\ 
 \midrule
 \citet{lale2020regret} & $T^{2/3}$ & Convex & Open-Loop & Stochastic  \\ 
  \citet{simchowitz2020improper} & $T^{2/3}$ & Convex & Open-Loop & Adversarial \\
 \citet{mania2019certainty} & $\sqrt{T}$ & Strongly Convex & Open-Loop & Stochastic \\
  \citet{simchowitz2020improper} & $\sqrt{T}$ & Strongly Convex & Open-Loop & Semi-adversarial  \\
 \textbf{This work} & polylog$(T)$ & Strongly Convex & Closed-Loop & Stochastic   \\
 \bottomrule
\end{tabular}
\label{table:1}
\end{table}

\section{Preliminaries}\label{prelim}
We denote the Euclidean norm of a vector $x$ as $\|x\|_2$. For a given matrix $A$, $\| A \|_2$ is its spectral norm, $\| A\|_F$ is its Frobenius norm, $A^\top$ is its transpose, $A^{\dagger}$ is its Moore-Penrose inverse, and $\Tr(A)$ is the trace of $A$. $\rho(A)$ denotes the spectral radius of $A$, \textit{i.e.}, the largest absolute value of its eigenvalues. The j-th singular value of a rank-$n$ matrix $A$ is denoted by $\sigma_j(A)$, where $\sigma_{\max}(A )\!\coloneqq \!\!\sigma_1(A) \!\geq\! \sigma_2(A) \!\geq\! \ldots \!\geq\! \sigma_n(A) \!\coloneqq\! \sigma_{\min}(A) \!>\! 0$. $I$ is the identity matrix with appropriate dimensions. $\mathcal{N}(\mu, \Sigma)$ denotes a multivariate normal distribution with mean vector $\mu$ and covariance matrix $\Sigma$. 

\textbf{State space form:} Consider an unknown discrete-time linear time-invariant system $\Theta$,
\begin{align}
    x_{t+1}& = A x_t + B u_t + w_t,  \qquad 
    y_t = C x_t + z_t \label{output},
\end{align}
where $x_t \in \mathbb{R}^{n}$ is the (latent) state of the system, $u_t \in \mathbb{R}^{p}$ is the control input, and the observation $y_t \in \mathbb{R}^{m}$ is the output of the system. At each time step $t$, the system is at state $x_t$ and the agent observes $y_t$, \textit{i.e.}, an imperfect state information. Then, the agent applies a control input $u_t$, observes the loss function $\ell_t$, pays the cost of $c_t=\ell_t(y_{t},u_{t})$, and the system evolves to a new $x_{t+1}$ at time step $t+1$. Let $\left(\mathcal{F}_{t} ; t \geq 0 \right)$ be the corresponding filtration. For any $t$, conditioned on $\mathcal{F}_{t-1}$, $w_{t}$ and $z_{t}$ are $\mathcal{N}(0,\sigmaW^{2}I)$ and $\mathcal{N}(0,\sigmaZ^{2}I)$ respectively. In this paper, in contrast to the standard assumptions in \LQG literature that the algorithm is given the knowledge of both $\sigmaW^{2}$ and $\sigmaZ^{2}$, we only assume the knowledge of their upper and lower bounds, i.e., $ \sigmaWup^{2}, \sigmaWlow^{2}, \sigmaZup^{2}$, and $\sigmaZlow^{2}$, such that, $0< \sigmaWlow^{2}\leq \sigmaW^{2}\leq \sigmaWup^{2}$ and $0< \sigmaZlow^{2}\leq \sigmaZ^{2}\leq\sigmaZup^{2}$, for some finite $\sigmaWup^{2},\sigmaZup^{2}$. For the system $\Theta$, let $\Sigma$ be the unique positive semidefinite solution to the following DARE (Discrete Algebraic Riccati Equation),
\begin{equation}\label{DARE}
    \sig = A \sig   A^\top - A \sig  C^\top \left( C \sig  C^\top + \sigma_z^2 I \right)^{-1} C \sig  A^\top + \sigma_w^2 I.
\end{equation} 
$\Sigma$ can be interpreted as the steady state error covariance matrix of state estimation under $\Theta$.

\textbf{Predictor form:} An equivalent and common representation of the system $\Theta$ in (\ref{output}), is its predictor form representation introduced by \citet{kalman1960new} and characterized as,
\begin{align}
    \hat{x}_{t+1}&=\Bar{A} \hat{x}_t+B u_t+F y_t, \qquad
    y_{t}=C \hat{x}_t+e_t, \label{predictor}
\end{align}
where $F\!=\! A \Sigma  C^\top \!\! \left( C \Sigma  C^\top \!\!\!+\! \sigma_z^2 I \right)^{-1}\!$ is the Kalman filter gain in the observer form, $e_t$ is the zero mean white innovation process and $\Bar{A} = A - F C$. In this equivalent representation of system, the state $\hat{x}_t$ can be seen as the estimate of the state in (\ref{output}). 
\begin{definition}[Controllability \& Observability]\label{c_o_def}
A system is $(A,B)$ controllable if the controllability matrix $[B \enskip AB \enskip A^2B \ldots A^{n-1}B]$ has full row rank. As the dual, a system is $(A,C)$ observable if the observability matrix $[C^\top \enskip (CA)^\top \enskip (CA^2)^\top \ldots (CA^{n-1})^\top]^\top$ has full column rank.
\end{definition}
We assume that the unknown system $\Theta$ is $(A, B)$ controllable, $(A, C)$ observable and $(A, F)$ controllable. This provides exponential convergence of the Kalman filter to the steady-state. Thus, without loss of generality, for the simplicity of analysis, we assume that $x_0 \sim \mathcal{N}(0, \Sigma)$, \textit{i.e.}, the system starts at the steady-state. In the steady state, $e_t \sim \mathcal{N}\left(0,C \sig  C^\top + \sigma_z^2 I \right)$. 

We assume that the unknown system $\Theta$ is order $n$ and $\rho(A) < 1$.
Let $\Phi(A) = \sup _{\tau \geq 0} \left\|A^{\tau}\right\|/\rho(A)^{\tau}$. In the following we consider the standard setting where $\Phi(A)$ is finite. The above mentioned construction is the general setting for the majority of literature on both estimation and regret minimization~\citep{oymak2018non, sarkar2019finite,tsiamis2019finite,simchowitz2020improper,lale2020regret, mania2019certainty} for which the main challenge is the estimation and the controller design. 


%

\section{System Identification}\label{dynamics}
In this section, we first explain the estimation methods that use state-space representation of the system (\ref{output}) to recover the model parameters and discuss the reason why they cannot provide reliable estimates in closed-loop estimation problems. Then we present our novel estimation method which provides reliable estimates in both open and closed-loop estimation. 

\textbf{Challenges in using the state-space representation for system identification:} Using the state-space representation in (\ref{output}), for any positive integer $H$, one can rewrite the output at time $t$ as follows,
\begin{equation} \label{markov_rollout}
    y_t = \sum\nolimits_{i=1}^{H} \!C A^{i-1} B u_{t-i} + CA^{H}x_{t-H} + z_t + \sum\nolimits_{i=0}^{H-1} CA^{i}w_{t-i-1}.
\end{equation}
\begin{definition} [Markov Parameters] \label{def:markov}
The set of matrices that maps the inputs to the output in (\ref{markov_rollout}) is called Markov parameters of the system $\Theta$. They are the first $H$ parameters of the Markov operator, $\Markov \!=\!\lbrace G^{[i]}\rbrace_{i\geq 0}$ with $G^{[0]} \!=\! 0_{m \times p}$, and $\forall i\!>\!0$,  $G^{[i]}\!=\!CA^{i-1}B$ that uniquely describes the system behavior. Moreover, $\Markov(H) \!=\![G^{[0]} ~\! G^{[1]}\! ~\!\ldots\!~ G^{[H-1]}] \!\in\! \R^{m\! \times \! Hp}$ denotes the $H$-length Markov parameters matrix. 
\end{definition}
For $\kappa_\Markov \! \geq\! 1$, let the Markov operator of $\Theta$ be bounded, \textit{i.e.}, $\sum_{i\geq 0}\|{G}^{[i]}\|\!\leq\!\kappa_\Markov$. Due to stability of $A$, the second term in (\ref{markov_rollout}) decays exponentially and for large enough $H$ it becomes negligible. Therefore, using Definition \ref{def:markov}, we obtain the following for the output at time $t$,
\begin{equation} \label{markov_rollout_approx}
    y_t \approx \sum\nolimits_{i=0}^{H} \!G^{[i]} u_{t-i} + z_t + \sum\nolimits_{i=0}^{H-1} CA^{i}w_{t-i-1}.
\end{equation}
From this formulation, a least squares estimation problem can be formulated using outputs as the dependent variable and the concatenation of $H$ input sequences $\bar{u}_t = [u_t,\ldots,u_{t-H}]$ as the regressor to recover the Markov parameters of the system:
\begin{equation}\label{markov_least}
    \wh{\Markov}(H) = \argmin_{X} \sum\nolimits_{t=H}^T \|y_t - X \bar{u}_t \|^2_2.
\end{equation}
Prior finite-time system identification algorithms propose to use i.i.d. zero-mean Gaussian noise for the input, to make sure that the two noise terms in (\ref{markov_rollout_approx}) are not correlated with the inputs~\textit{i.e.} open-loop estimation. This lack of correlation allows them to 
solve (\ref{markov_least}), estimate the Markov parameters and develop finite-time estimation error guarantees~\citep{oymak2018non,sarkar2019finite,lale2020regret,simchowitz2019learning}. From Markov parameter estimates, they recover the system parameters $(A,B,C)$ up to similarity transformation using singular value decomposition based methods like Ho-Kalman algorithm \citep{ho1966effective}.

However, when a controller designs the inputs based on the history of inputs and observations, the inputs become highly correlated with the past process noise sequences, $\{w_i\}_{i=0}^{t-1}$. This correlation prevents consistent and reliable estimation of Markov parameters using (\ref{markov_least}). Therefore, these prior open-loop estimation methods do not generalize to the systems that adaptive controllers generates the inputs for estimation, \textit{i.e.}, closed-loop estimation. In order to overcome this issue, we exploit the predictor form of the system $\Theta$ and design a novel system identification method that provides consistent and reliable estimates both in closed and open-loop estimation problems.

\textbf{Novel estimation method for partially observable linear dynamical systems:} Using the predictor form representation (\ref{predictor}), for a positive integer $\Hes$, the output at time $t$ can be rewritten as follows,
\begin{equation} \label{predictor_rollout}
    y_t = \sum\nolimits_{k=0}^{\Hes-1} C\Bar{A}^k \left(Fy_{t-k-1} \!+\! Bu_{t-k-1} \right) + e_t + C\Bar{A}^{\Hes} x_{t-\Hes}.
\end{equation}
Using the open or closed-loop generated input-output sequences up to time $\tau$, $\{y_t,u_t\}_{t=1}^\tau$, we construct subsequences of $\Hes$ input-output pairs for $\Hes \! \leq \!t \!\leq \tau$,
\begin{equation*}
    \phi_t \!=\! \!\left[ y_{t-1}^\top \ldots y_{t-\Hes}^\top u_{t-1}^\top \ldots u_{t-\Hes}^\top \right]^\top \!\!\!\!\in\! \mathbb{R}^{(m+p)\Hes}.
\end{equation*}
The output of the system, $y_t$ can be represented using $\phi_t$ as:
\begin{equation}\label{arx_rollout}
    y_{t} \!=\! \modelearn \phi_{t} + e_{t} + C \Bar{A}^{\Hes} x_{t-\Hes} \!\enskip \text{for} \enskip \modelearn \!=\! \left[CF \enskip\! C\Bar{A}F \!\enskip\!  \ldots\! \enskip\! C \Bar{A}^{\Hes\!-\!1}\!F \enskip\! CB \enskip\! C\Bar{A}B \!\enskip\! \ldots\! \enskip\! C\Bar{A}^{\Hes\!-\!1}\! B \right].
\end{equation}
Notice that $\Bar{A}$ is stable due to $(A,F)$-controllability of $\Theta$~\citep{kailath2000linear}. Therefore, with a similar argument used in (\ref{markov_rollout}), for $\Hes = O(\log(T))$, the last term in (\ref{predictor_rollout}) is negligible. This yields into a linear model of the dependent variable $y_t$ and the regressor $\phi_t$ with additive i.i.d. Gaussian noise $e_t$:
\begin{equation}
    y_t \approx \modelearn \phi_{t} + e_{t}.
\end{equation} 
For this model, we achieve consistent and reliable estimates by solving the following regularized least squares problem,
\begin{equation}\label{new_lse}
    \wh{\mathcal{G}}_{\mathbf{y}} = \argmin_{X} \lambda \|X\|_F^2 +\! \sum\nolimits_{t=\Hes}^\tau \|y_t - X \phi_t\|^2_2.
\end{equation}
This problem does not require any specification on how the inputs are generated and therefore can be deployed in both open and closed-loop estimation problems. Exploiting the specific structure of $\modelearn$ in (\ref{arx_rollout}), we design a procedure named \Sys, which recovers model parameters from $\wh{\mathcal{G}}_{\mathbf{y}}$ (see Appendix \ref{apx:Identification} for details). To give an overview, \Sys forms two Hankel matrices from the blocks of $\wh{\mathcal{G}}_{\mathbf{y}}$ which corresponds to the product of observability and controllability matrices. \Sys applies a variant of Ho-Kalman procedure and similar to this classical algorithm, it uses singular value decomposition of these Hankel matrices to recover model parameter estimates $(\hat{A},\hat{B},\hat{C})$ and finally constructs $\wh{\Markov}(H)$. For the persistently exciting inputs, the following gives the first finite-time system identification guarantee in both open and closed-loop estimation problems (see Appendix \ref{apx:generalsysidproof} for the proof).
\begin{theorem}[System Identification]\label{theo:sysid} If the inputs are persistently exciting, then for $T$ input-output pairs, as long as $T$ is large enough, solving the least squares problem in (\ref{new_lse}) and deploying \Sys procedure provides model parameter estimates $(\hat{A},\hat{B},\hat{C},\wh{\Markov}(H))$ that with high probability,
\begin{equation}
    \|\hat{A} - A \|, \|\hat{B} - B \|, \|\hat{C} - C \|,  \|\wh{\Markov}(H) - \Markov(H) \| = \Tilde{\OO}(1/\sqrt{T})
\end{equation}
\end{theorem}


\section{Controller and Regret }\label{Policyclass}
In this section we describe the class of linear dynamic controllers (\LDC) and show how a convex policy reparameterization can provide accurate approximations of the \LDC controllers. Then we provide the regret definition, \textit{i.e.} performance metric, of the adaptive control task.

\textbf{Linear dynamic controller (\LDC):}
An \LDC policy $\pi$ is a linear controller with an internal state dynamics of 
\begin{equation}\label{eq:LDC}
    s^\pi_{t+1} = A_\pi s_t^\pi + B_\pi y_t, \qquad u^\pi_{t} = C_\pi s_t^\pi + D_\pi y_t,
\end{equation}
where $s^\pi_t\in \boldR^s$ is the state of the controller, $y_t$ is the input to the controller, \textit{i.e.} observation from the system that controller is designing a policy for, and $u^\pi_t$ is the output of the controller. \LDC controllers provide a large class of controller. For instance, when the problem is canonical \LQG setting, the optimal policy is known to be a \LDC policy~\citep{bertsekas1995dynamic}. 

\textbf{Nature's output:} Using (\ref{markov_rollout}), we can further decompose the generative components of $y_t$ to obtain,
\begin{align*}
  y_t = z_t + CA^tx_0 + \sum\nolimits_{i=0}^{t-1} CA^{t-i-1}w_{i} + \sum\nolimits_{i=0}^{t} \!G^{[i]} u_{t-i}  
\end{align*}
Notice that first three components generating $y_t$ are derived from the uncontrollable noise processes in the system, while the last one is a linear combination of control inputs. The first three components are known as Nature's $y$, \textit{i.e.}, Nature's output~\citep{simchowitz2020improper,youla1976modern}, of the system,
\begin{align}\label{eq:nature}
\nature_t(\Markov) \coloneqq y_t - \sum\nolimits_{i=0}^{t-1} G^{[i]} u_{t-i} = z_t+CA^tx_0+\sum\nolimits_{i=0}^{t-1}CA^{t-i-1}w_i.
\end{align}
The ability to define Nature's $y$ is a unique characteristics of linear dynamical systems. At any time step $t$, after following a sequence of control inputs $\{ u_{i}\}_{i=0}^t$, and observing $y_t$, we can compute $\nature_t(\Markov)$ using (\ref{eq:nature}). This quantity allows for counterfactual reasoning about the outcome of the system. Particularly, having access to $\lbrace\nature_{\tau-t}(\Markov)\rbrace_{t\geq 0}$, we can reason what the outputs $y_{\tau-t}'$ of the system would have been, if the agent, instead, had taken other sequence of control inputs $\{ u_{i}'\}_{i=0}^{\tau-t}$, \textit{i.e.},
\begin{align*}
y_{\tau-t}' = \nature_{\tau-t}(\Markov) + \sum\nolimits_{i=0}^{\tau-t-1} G^{[i]} u_{\tau-t-i}'.
\end{align*}
This property indicates that we can use $\lbrace\nature_{\tau-t}(\Markov)\rbrace_{t\geq 0}$ to evaluate the quality of any other potential input sequences, and build a desirable controller, as elaborated in the following.



\textbf{Disturbance feedback control (\DFC):}  In this work, we adopt a convex policy reparametrization called \DFC  introduced by \citet{simchowitz2020improper}. A \DFC policy of length $H'$ is defined as a set of parameters,  $\Mcontrol(H'):=\lbrace M^{[i]} \rbrace_{i=0}^{H'-1}$, prescribing the control input of 
\begin{equation}
    u^\Mcontrol_t = \sum\nolimits_{i=0}^{H'-1}M^{[i]}\nature_{t-i}(\Markov).
\end{equation}
for Nature's $y$, $\lbrace\nature_{t-i}(\Markov)\rbrace^{H'-1}_{i= 0}$. The \DFC policy construction is in parallel with the classical Youla parametrization \citep{youla1976modern} which states that any linear controller can be prescribed as acting on past noise sequences. Thus, \DFC policies can be regarded as truncated approximations to \LDC{}s. More formally, any stabilizing \LDC policy can be well-approximated as a \DFC policy (see Appendix \ref{apx:lem:deviation}).

Define the convex compact sets of \DFC{}s, $\Mcontrolset_\psi$ and $\Mcontrolset$ such that all the controllers in these sets persistently excite the system $\Theta$. The precise definition of persistence of excitation condition is given in Appendix \ref{subsec:persistence}. The controllers $\Mcontrol(H_0') \in \Mcontrolset_\psi$ are bounded \textit{i.e.}, $\sum\nolimits_{i\geq 0}^{H_0'-1}\|M^{[i]}\| \leq \kappa_\psi$ and $\Mcontrolset$ is an $r$-expansion of $\Mcontrolset_\psi$, \textit{i.e.}, $\Mcontrolset = \lbrace \Mcontrol(H')=\Mcontrol(H_0')+ \Delta : \Mcontrol(H_0') \in \Mcontrolset_\psi, \sum\nolimits_{i\geq 0}^{H'-1}\|\Delta^{[i]}\| \leq r \kappa_\psi \rbrace$ where $H'_0 = \lfloor \frac{H'}{2} \rfloor - H$. 
%
%
Thus, all controllers $\Mcontrol(H') \in \Mcontrolset$ are also bounded $\sum\nolimits_{i\geq 0}^{H'-1}\|M^{[i]}\| \leq \kappa_\Mcontrolset$ where $\kappa_\Mcontrolset = \kappa_\psi (1+r)$. Throughout the interaction with the system, the agent has access to $\Mcontrolset$. 

\textbf{Loss function:} The loss function at time $t$, $\ell_t$, is smooth and strongly convex for all $t$, i.e., $0\prec\strong I\preceq\nabla^2\ell_t(\cdot,\cdot)\preceq\smooth I$ for a finite constant $\smooth$. Note that the standard quadratic regulatory costs of $\ell_t(y_{t},u_{t})=y_t^\top Q_ty_t+u_t^\top R_tu_{t}^\top$ with bounded positive definite matrices $Q_t$ and $R_t$ are special cases of the mentioned setting. For all $t$, the unknown lost function $c_t = \ell_t(\cdot, \cdot)$ is non-negative strongly convex and associated with a parameter $L$, such that for any $R$ with $\|u\|,\|u'\|, \|y\|,\|y'\| \leq R$, we have,

\begin{align}
\label{asm:lipschitzloss}|\ell_t(y,u)-\ell_t(y',u') |\leq LR(\|y-y'\|+\|u-u'\|) \enskip \text{ and } \enskip |\ell_t(y,u)|\leq LR^2.\end{align}



\textbf{Regret definition:} We evaluate the agent's performance by its regret with respect to $\Mcontrol_\star$, the optimal, in hindsight, \DFC policy in the given set $\Mcontrolset_\psi$, \textit{i.e.}, $\Mcontrol_\star \!=\!  \argmin_{\Mcontrol \in \Mcontrolset_\psi} \sum_{t=1}^{T} \ell_t(y_{t}^\Mcontrol,u_{t}^\Mcontrol).$
After $T$ step of interaction, the agent's regret is denoted as 
\begin{equation}\label{regret_def}
    \reg(T) = \sum\nolimits_{t=1}^T c_t- \ell_t(y^{\Mcontrol_\star},u^{\Mcontrol_\star}).
\end{equation}

\section{\textsc{AdaptOn}} \label{sec:algorithm}
\begin{figure}[t]
    \centering
    \includegraphics[width=0.95\textwidth]{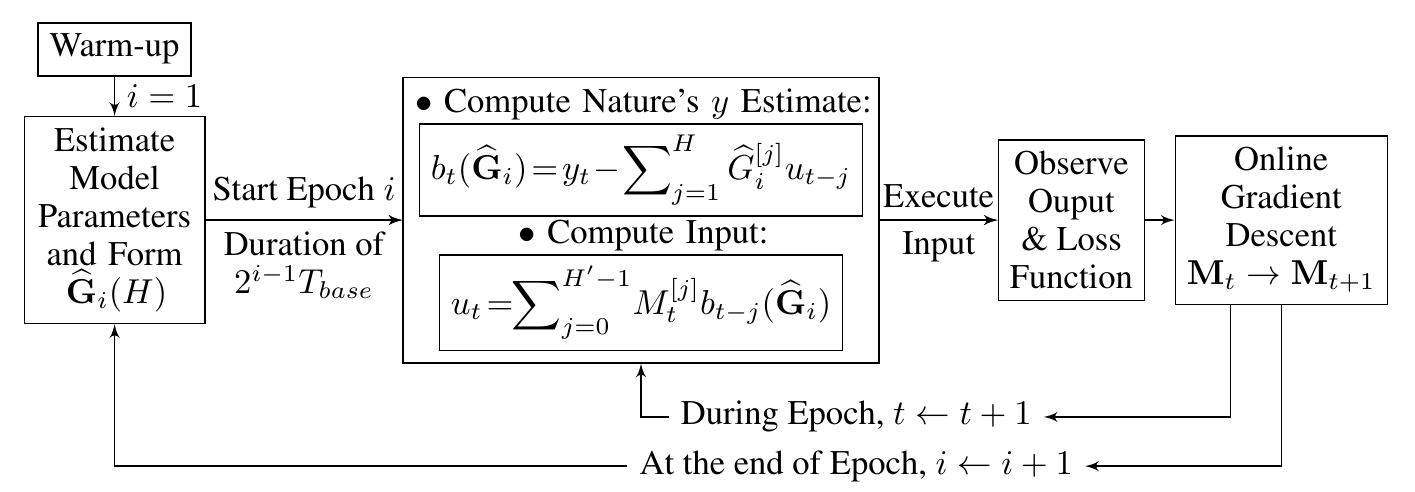}
    \caption{\alg}\label{alg:fig}
\end{figure}

In this section, we present \alg, a sample efficient \textbf{adapt}ive control \textbf{on}line learning algorithm which learns the model dynamics through interaction with the environment and continuously deploys online convex optimization to improve the control policy. \alg is illustrated in Figure~\ref{alg:fig} and the detailed pseudo-code is provided in Appendix \ref{apx:adapton}. 
  
\textbf{Warm-up:}  \alg starts with a fixed warm-up period and applies $u_t \sim \mathcal{N}(0,\sigma_u^2 I)$ for the first $\Tburn$ time steps. The length of the warm-up period is chosen to guarantee an accountable first estimate of the system, the persistence of excitation during the adaptive control period and the stability of the online learning algorithm on the underlying system.

\textbf{Adaptive control in epochs:} After warm-up, \alg starts controlling the system and operates in epochs with doubling length. \alg sets the base period $\Tbase$ to the initial value $\Tbase=\Tburn$ and for each epoch $i$, it runs for $2^{i-1}\Tbase$ time steps. 

\textbf{Model estimation in the beginning of epochs:} At the beginning of each epoch $i$, \alg exploits the past experiences up to epoch $i$. It deploys the proposed closed-loop estimation method and solves (\ref{new_lse}) to estimate $\modelearn$. \alg then exploits the construction of true $\modelearn$ to estimate model parameter estimates $\hat A_i, \hat B_i, \hat C_i$ and constructs an estimate of $H$-length Markov parameters matrix, $\wh{\Markov}_i(H)$, via \Sys described in Section \ref{dynamics} and provided in Appendix \ref{apx:Identification}. 


\textbf{Control input, output and loss during the epochs:} \alg utilizes $\wh{\Markov}_i(H)$ and the past inputs to estimate the Nature's outputs, $\nature_t(\wh \Markov_i) = y_t- \sum_{j=1}^{H} \wh G_i^{[j]} u_{t-j}$. Using these estimates, \alg executes a \DFC policy $\Mcontrol_t \in \Mcontrolset$ such that  $u_t^{\Mcontrol_t} = \sum_{j=0}^{H'-1}M_t^{[j]} \nature_{t-j}(\wh \Markov_i)$ and observes the output of $y_t^{\Mcontrol_t}$. Finally, \alg receives the loss function $\ell_t$, pays a cost of $\ell(y_t^{\Mcontrol_t}, u_t^{\Mcontrol_t})$.

\textbf{Counterfactual input, output, loss:} \alg uses counterfactual reasoning introduced in \citet{simchowitz2020improper} to update its controller. After observing the loss function $\ell_t$, it constructs,
\begin{equation}
    \tilde{u}_{t-j}(\Mcontrol_t, \wh \Markov_i) = \sum\nolimits_{l=0}^{H'-1}M_t^{[l]}\nature_{t-j-l}(\wh \Markov_i),
\end{equation}
the counterfactual inputs, which are the recomputations of past inputs as if the current DFC policy is applied using Nature's $y$ estimates. Then, \alg reasons about the counterfactual output of the system. Using the current Nature's $y$ estimate and the counterfactual inputs, the agent approximates what the output of the system could be, if counterfactual inputs had been applied,
\begin{equation}
    \tilde{y}_t(\Mcontrol_t, \wh \Markov_i) = \nature_t(\wh \Markov_i) + \sum\nolimits_{j=1}^H \hat{G}_i^{[j]} \tilde{u}_{t-j}(\Mcontrol_t, \wh \Markov_i).
\end{equation}
Using the counterfactual inputs, output and the revealed loss function $\ell_t$, \alg finally constructs,
\begin{equation}
    f_t(\Mcontrol_t, \wh \Markov_i) = \ell_t(\tilde{y}_t(\Mcontrol_t, \wh \Markov_i), \tilde{u}_{t}(\Mcontrol_t, \wh \Markov_i) ).
\end{equation}
which is termed as the counterfactual loss. It is \alg's approximation of what the cost would have been at time $t$, if the current \DFC policy was applied until time $t$. It gives a performance evaluation of the current \DFC policy to \alg for updating the policy. Note that the Markov parameter estimates are crucial in the accuracy of this performance evaluation.

\textbf{Online convex optimization:} In order to optimize the controller during the epoch, at each time step, \alg runs online gradient descent on the counterfactual loss function $f_t(\Mcontrol_t, \wh \Markov_i)$ while keeping the updates in the set $\Mcontrolset$ via projection \citep{simchowitz2020improper},
\begin{equation*}
    \Mcontrol_{t+1}=\proj_\Mcontrolset\left(\Mcontrol_t-\eta_t\nabla_\Mcontrol f_t\left(\Mcontrol,\wh \Markov_i\right)\Big|_{\Mcontrol_t}\right).
\end{equation*}
Notice that if \alg had access to the underlying Markov operator $\Markov$, the counterfactual loss would have been the true loss of applying the current \DFC policy until time $t$, up to truncation. By knowing the exact performance of the \DFC policy, online gradient descent would have obtained accurate updates. Using the counterfactual loss for optimizing the controller causes an error in the gradient updates which is a function of estimation error of $\wh \Markov_i$. Therefore, as the Markov estimates improve via our closed-loop estimation method, the gradient updates get more and more accurate.

\section{Regret Analysis}\label{analysis}
In this section, we first provide the closed-loop learning guarantee of \alg, then show that \alg maintains stable system dynamics and present the regret upper bound for \alg. 

\textbf{Closed-loop learning guarantee:} In the beginning of adaptive control epochs, \alg guarantees that Markov parameter estimates are accurate enough that deploying any controller from set $\Mcontrolset$ provides persistence of excitation in inputs (see Appendix \ref{apx:persistence}). Under this guarantee, using our novel estimation method at the beginning of any epoch $i$ ensures that during the epoch, $\|\wh{\Markov}_i(H) - \Markov(H)  \| = \Tilde{\OO}(1/\sqrt{2^{i-1}\Tbase})$, due to Theorem \ref{theo:sysid} and doubling length epochs.

\textbf{Stable system dynamics with \alg:} Since $w_t$ and $z_t$ are Gaussian disturbances, from standard concentration results we have that Nature's $y$ is bounded with high probability for all $t$ (see Appendix \ref{apx:Boundedness}). Thus, let $\|\nature_t(\Markov)\|\leq \kappa_\nature$ for some $\kappa_\nature$. The following lemma shows that during \alg, Markov parameter estimates are well-refined such that the inputs, outputs and the Nature's $y$ estimates of \alg are uniformly bounded with high probability. The proof is in Appendix \ref{apx:Boundedness}.
\begin{lemma}
For all $t$ during the adaptive control epochs, $\|u_t\| \leq \kappa_\Mcontrolset \kappa_\nature$, $\|y_t\| \leq \kappa_\nature (1+ \kappa_\Markov \kappa_\Mcontrolset) $ and $\|\nature_t(\wh{\Markov})\| \leq 2\kappa_\nature$ with high probability.
\end{lemma}

\textbf{Regret upper bound of \alg:} The regret decomposition of \alg includes 3 main pieces: $\!(R_1)\!$ Regret due to warm-up, $\!(R_2)\!$ Regret due to online learning controller, $\!(R_3)\!$ Regret due to lack of system dynamics knowledge (see Appendix \ref{apx:Regret} for exact expressions and proofs). $R_1$ gives constant regret for the short warm-up period. $R_2$ results in $\OO(\log(T))$ regret. Note that this regret decomposition and these results follow and adapt Theorem 5 of \citet{simchowitz2020improper}. The key difference is in $R_3$, which scales quadratically with the Markov parameter estimation error $\|\wh{\Markov}_i(H) \!-\! \Markov(H)\|$. \citet{simchowitz2020improper} deploys open-loop estimation and does not update the model parameter estimates during adaptive control and attains $R_3 = \Tilde{\OO}(\sqrt{T})$ which dominates the regret upper bound. However, using our novel system identification method with the closed-loop learning guarantees of Markov parameters and the doubling epoch lengths \alg gets $R_3 = \OO(\text{polylog}(T))$.



\begin{theorem}\label{thm:polylogregret}
Given $\Mcontrolset$, a closed, compact and convex set of \DFC policies with persistence of excitation, with high probability, \alg achieves logarithmic regret, \textit{i.e.}, $\reg(T)=polylog(T).$
\end{theorem}

In minimizing the regret, \alg competes against the best \DFC policy in the given set $\Mcontrolset$. Recall that any stabilizing \LDC policy can be well-approximated as a \DFC policy. Therefore, for any \LDC policy $\pi$ whose \DFC approximation lives in the given $\Mcontrolset$, Theorem \ref{thm:polylogregret} can be extended to achieve the first logarithmic regret in \LQG setting.

\begin{corollary}\label{regretLQG}
Let $\pi_\star$ be the optimal linear controller for \LQG setting. If the \DFC approximation of $\pi_\star$ is in $\Mcontrolset_\psi$, then the regret of \alg with respect to $\pi_\star$  is $\sum\nolimits_{t=1}^T c_t- \ell_t(y_t^{\pi_\star},u_t^{\pi_\star}) = polylog(T)$.
\end{corollary}

Note that without any estimation updates during the adaptive control, \alg reduces to a variant of the algorithm given in \citet{simchowitz2020improper}. While the update rule in \alg results in $\OO(\log(T))$ updates in adaptive control period, one can follow different update schemes as long as \alg obtains enough samples in the beginning of the adaptive control period to obtain persistence of excitation. The following is an immediate corollary of Theorem \ref{thm:polylogregret} which considers the case when number of epochs or estimations are limited during the adaptive control period.

\begin{corollary}\label{corr:slowupdate}
If enough samples are gathered in the adaptive control period, \alg with any update scheme less than $\log(T)$ updates has $\reg(T) \in [polylog(T),\Tilde{O}(\sqrt{T})].$
    
\end{corollary}

\section{Related Works}\label{relatedworks}
\textbf{Classical results in system identification:} The classical open or closed-loop system identification methods either consider the asymptotic behavior of the estimators or demonstrate the positive and negative empirical performances without theoretical guarantees \citep{van1994n4sid,verhaegen1994identification,chen1992integrated,forssell1999closed,van1997closed,ljung1996subspace,ljung1999system,qin2003closed}. Most of prior work exploits the state-space form or the innovations form representation of the system. Among all closed-loop estimation methods only a handful consider the predictor form representation for system identification \citep{chiuso2005consistency, jansson2003subspace}. For an extensive overview of classical system identification is provided in \citet{qin2006overview}.

\textbf{Finite-time system identification for partially observable linear dynamical systems:} In partially observable linear systems, most of the prior works focus on open-loop system identification guarantees~\citep{oymak2018non,sarkar2019finite,tsiamis2019finite,simchowitz2019learning,lee2019non,tsiamis2019sample,tsiamis2020online}. Among all, only \citet{lee2019non} considers finite-time closed-loop system identification. However, they analyze the output estimation error rather than explicitly recovering the model parameters as presented in this work. Another body of work aims to extend the problem of estimation and prediction to online convex optimization where a set of strong guarantees on cumulative prediction errors are provided~\citep{abbasi2014tracking, hazan2017learning,arora2018towards,hazan2018spectral,lee2019robust,ghai2020no}

\textbf{Regret in fully observable linear dynamical systems:} Efforts in regret analysis of adaptive control in fully observable linear dynamical systems is initiated by seminal work of \citet{abbasi2011regret}. They present $\Tilde{\OO}(\sqrt{T})$ regret upper bound for linear quadratic regulators (LQR) which are fully observable counterparts of \LQG. 
This work sparked the flurry of research with different directions in the regret analysis in LQRs \citep{faradonbeh2017optimism,abeille2017thompson, abeille2018improved, dean2018regret,faradonbeh2018input,cohen2019learning,mania2019certainty,abbasi2019model}. Recently, \citet{cassel2020logarithmic} show that logarithmic regret is achievable if only $A$ or $B$ is unknown, and \citet{simchowitz2020naive} provide $\Tilde{\OO}(\sqrt{T})$ regret lower bound for LQR setting with fully unknown systems. However, due to the persistent noise in the observations of the hidden states, the mentioned lower bound does not carry to the partially observable linear dynamical systems.

\textbf{Regret in adversarial noise setting:} In adversarial noise setting, most of the works consider full information of the underlying system \citep{agarwal2019online, cohen2018online, agarwal2019logarithmic, foster2020logarithmic}. Recent efforts extend to adaptive control in the adversarial setting for the unknown system model \citep{hazan2019nonstochastic, simchowitz2020improper}.


\section{Conclusion}\label{conclusion}
In this paper, we propose the first system identification algorithm that provides consistent and reliable estimates with finite-time guarantees in both open and closed-loop estimation problems. We deploy this estimation technique in \alg, a novel adaptive control that efficiently learns the model parameters of the underlying dynamical system and deploys projected online gradient descent to design a controller.
We show that in the presence of convex set of persistently exciting linear controllers and strongly convex loss functions, \alg achieves a regret upper bound that is polylogarithmic in number of agent-environment interactions. 
The unique nature of \alg which combines occasional model estimation with continual online convex optimization allows the agent to achieve significantly improved regret in the challenging setting of adaptive control in partially observable linear dynamical systems. 
%

\newpage
\section*{Acknowledgements}
S. Lale is supported in part by DARPA PAI. K. Azizzadenesheli gratefully acknowledge the financial support of Raytheon and Amazon Web Services. B. Hassibi is supported in part by the National Science Foundation under grants CNS-0932428, CCF-1018927, CCF-1423663 and CCF-1409204, by a grant from Qualcomm Inc., by NASA’s Jet Propulsion Laboratory through the President and Director’s Fund, and by King Abdullah University of Science and Technology. A. Anandkumar is supported in part by Bren endowed chair, DARPA PAIHR00111890035 and LwLL grants, Raytheon, Microsoft, Google, and Adobe faculty fellowships.

\bibliography{main}
\bibliographystyle{unsrtnat}
\newpage
\appendix

\begin{center}
{\huge Appendix}
\end{center}
In the beginning of this Appendix, we will provide the overall organization of the Appendix and notation table for the paper. Then we will include description of lower bound on the warm-up duration and briefly comment on their goal in helping to achieve the regret result.

\textbf{Appendix organization and notation:}
In Appendix \ref{apx:Policies}, we revisit the precise definitions of \LDC and \DFC policies and introduce the technical properties that will be used in the proofs. We also provide Lemma \ref{lem:deviation} that shows that any stabilizing \LDC policy can be well-approximated by an \DFC policy. In Appendix \ref{apx:Identification}, we provide the details of \Sys procedure that recovers model parameter estimates from the estimate of $\modelearn$ obtained by (\ref{new_lse}). Appendix \ref{apx:adapton} provides the detailed pseudocode of \alg. In Appendix \ref{apx:generalsysidproof}, we provide the proof Theorem \ref{theo:sysid} step by step. In Appendix \ref{apx:persistence}, we provide the persistence of excitation guarantees for \alg which enables us to achieve consistent estimates by using the new system identification method. In particular, in Appendix \ref{subsec:warmuppersist} we show the persistence of excitation during the warm-up, in Appendix \ref{subsec:persistence} we formally define the persistence of excitation property of the controllers in $\Mcontrolset$, \textit{i.e.} (\ref{precise_persistence}), and finally in Appendix \ref{subsec:theo:persist_adaptive}, we show that the control policies of \alg achieve persistence of excitation during the adaptive control. Appendix \ref{apx:Boundedness} shows that execution of \alg results in stable system dynamics. In Appendix \ref{apx:Regret}, we state the formal regret result of the paper, Theorem \ref{thm:main} and provides its proof. Appendix \ref{apx:convex} briefly looks into the case where the loss functions are convex. Finally, in Appendix \ref{technical}, we provide the supporting technical theorems and lemmas. Table \ref{table:notation} provides the notations used throughout the paper.

\begin{table}
\centering
\caption{Useful Notations for the Analysis}
\vspace{-0.7em}
 \begin{tabular}{l | l } 
 \toprule
 \textbf{System Notations} & \textbf{Definition} \\ 
 \midrule
 $\Theta$ & Unknown discrete-time linear time invariant system with dynamics of (\ref{output})  \\ 
  $x_t$ & Unobserved (latent) state  \\
$y_t$ & Observed output \\
 $u_t$ & Input to the system  \\
 $w_t$ & Process Noise, $w_t \sim \mathcal{N}(0,\sigmaW^{2}I)$  \\
 $z_t$ & Measurement noise,  $z_t \sim \mathcal{N}(0,\sigmaZ^{2}I)$ \\
 $\ell_t(y_t,u_t)$ & Revealed loss function after taking action $u_t$  \\
 $\Sigma$ & Steady state error covariance matrix of state estimation under $\Theta$  \\
 $F$ & Kalman filter gain in the observer form, $F\!=\! A \Sigma  C^\top \!\! \left( C \Sigma  C^\top \!\!\!+\! \sigma_z^2 I \right)^{-1}\!$\\
 $e_t$ & Innovation process, $e_t \sim \mathcal{N}\left(0,C \sig  C^\top + \sigma_z^2 I \right)$ \\
 $\Phi(A)$ & Growth rate of powers of $A$,  $\Phi(A) = \sup_{\tau \geq 0} \left\|A^{\tau}\right\|/\rho(A)^{\tau}$ \\
 $\Markov$ & Markov operator, $\Markov \!=\!\lbrace G^{[i]}\rbrace_{i\geq 0}$ with $G^{[0]} \!=\! 0$, and $\forall i\!>\!0$,  $G^{[i]}\!=\!CA^{i-1}B$
 \\
 $\modelearn$ & System parameter to be estimated
 \\
 $\phi_t$ & Regressor in the estimation, concatenation of $\Hes$ input-output pairs \\
 $\nature_t(\Markov)$ & Nature's $y$, $\nature_t(\Markov) \coloneqq y_t - \sum\nolimits_{i=0}^{t-1} G^{[i]} u_{t-i}$
 \\
 $\Gol$ & Operator that maps history of noise and open-loop inputs to $\phi$ \\
 \toprule
 \textbf{Policy Notations} &  \\ 
 \midrule
 $\pi$ & \LDC policy  \\
 $\psi$ & Proper decay function, nonincreasing, $\lim_{h'\rightarrow \infty}\psi(h')=0$ \\
 $\Pi(\psi)$ & \LDC class with proper decay function $\psi$, $\forall \pi \in \Pi(\psi) $, $\sum_{i\geq h}\|{G'_\pi}^{[i]}\|\leq \psi(h)$, $\forall h$ \\
 $\Markov_{\pi}'$ & Markov operator of induced closed-loop system by $\pi$ \\
 $\Mcontrol(H')$ & \DFC policy of length $H'$, $\Mcontrol(H'):=\lbrace M^{[i]} \rbrace_{i=0}^{H'-1}$, 
    \\
 $u^\Mcontrol_t$ & Input designed by \DFC policy $ u^\Mcontrol_t = \sum\nolimits_{i=0}^{H'-1}M^{[i]}\nature_{t-i}(\Markov)$ \\
 $\Mcontrolset$ & Given set of persistently exciting \DFC controllers  \\
 $\Mcontrol_\star$ & Optimal policy in hindsight in $\Mcontrolset$ \\
 $\tilde{u}_{t-j}(\Mcontrol_t, \wh \Markov_i)$ & Counterfactual input, $ \sum\nolimits_{l=0}^{H'-1}M_t^{[l]}\nature_{t-j-l}(\wh \Markov_i)$  \\
 $\tilde{y}_{t-j}(\Mcontrol_t, \wh \Markov_i)$ & Counterfactual output, $  \nature_t(\wh \Markov_i) + \sum\nolimits_{j=1}^H \hat{G}_i^{[j]} \tilde{u}_{t-j}(\Mcontrol_t, \wh \Markov_i)$  \\
 $ f_t(\Mcontrol_t, \wh \Markov_i)$ & Counterfactual loss, $\ell_t(\tilde{y}_t(\Mcontrol_t, \wh \Markov_i), \tilde{u}_{t}(\Mcontrol_t, \wh \Markov_i) )$  \\
 \toprule
 \textbf{Quantities } &  \\ 
 \midrule
 $H$ & Length of Markov parameter matrix, $\Markov(H)$   \\
 $\Hes$ & Length of estimated operator $\modelearn$  \\
 $H'$ & Length of \DFC policy \\
 $\kappa_\Markov$ & Bound on the Markov operator, $\sum_{i\geq0} \|G^[i]\| \leq  \kappa_\Markov$  \\
 $\kappa_\nature$ & Bound on Nature's $y$, $\|\nature_t(\Markov)\| \leq \kappa_\nature $ \\
 $\kappa_u$ & Bound on input $\|u_t\|$ \\
 $\kappa_y$ & Bound on output $\|y_t\|$ \\
 $\kappa_\Mcontrolset$ & Bound on the \DFC policy, $\sum\nolimits_{i\geq 0}^{H'-1}\|M^{[i]}\| \leq \kappa_\Mcontrolset = (1+r)\kappa_\psi $ \\
 $\kappa_\psi$ & Maximum of decay function, $\psi(0)$ \\
 $\psi_\Markov(h)$ & Induced decay function on $\Markov$,  $\sum_{i\geq h}\|{G}^{[i]}\|$\\
 $\Tbase = \Tburn$ & Base length of first adaptive control epoch and warm-up duration  \\
 \toprule
 \textbf{Estimates} &  \\ 
 \midrule
 $\modelearni$ & Estimate of $\modelearn$ at epoch $i$   \\
 $\wh \Markov (H)$ & Estimate of Markov parameter matrix   \\
 $\nature_t(\wh \Markov_i)$ & Estimate of Nature's $y$,   \\
 $\epsilon_\Markov(i,\delta)$ & Estimation error for the Markov operator estimate at epoch $i$, $\Tilde{\OO}(1/\sqrt{2^{i-1}\Tbase})$ \\
 \bottomrule
\end{tabular}
\label{table:notation}
\end{table}

\textbf{Warm-up duration:}
The duration of warm-up is chosen as $\Tburn \geq \Tmax$, where,
\begin{equation}
    \Tmax \coloneqq \max \{H, H', T_o, T_A, T_B, T_c, T_{\epsilon_G}, T_{cl}, T_{cx}, T_r\}.
\end{equation}
This duration guarantees an accountable first estimate of the system ($T_o$, see Appendix \ref{subsec:warmuppersist}), the stability of the online learning algorithm on the underlying system ($T_A, T_B$, see Appendix \ref{apx:generalsysidproof}), the stability of the inputs and outputs ($T_{\epsilon_\Markov}$, see Appendix \ref{subsec:theo:persist_adaptive}), the persistence of excitation during the adaptive control period ($T_{cl}$, see Appendix \ref{subsec:theo:persist_adaptive}), an accountable estimate at the first epoch of adaptive control ($T_c$, see Appendix \ref{subsec:theo:persist_adaptive}), the conditional strong convexity of expected counterfactual losses ($T_{cx}$, see Appendix \ref{subsec:markovbound}) and the existence of a good comparator \DFC policy in $\Mcontrolset$ ($T_r$, see Appendix \ref{subsec:markovbound}) . The precise expressions are given throughout the Appendix in stated sections.

\newpage
\section{Details on \LDC Policies and \DFC Policies} \label{apx:Policies}
\textbf{Linear Dynamic Controller (\LDC):}
An \LDC policy, $\pi$, is a $s$ dimensional linear controller on a state $s^\pi_t\in \boldR^s$ of a linear dynamical system $(A_\pi,B_\pi,C_\pi,D_\pi)$, with input $y^\pi_t$ and output $u^\pi_t$, 
where the state dynamics and the controller evolves as follows,
\begin{equation}\label{eq:LDC}
    s^\pi_{t+1} = A_\pi s_t^\pi + B_\pi y_t^\pi, \qquad u^\pi_{t} = C_\pi s_t^\pi + D_\pi y_t^\pi.
\end{equation}
Deploying a \LDC policy $\pi$ on the system $\Theta =(A,B,C)$ induces the following joint dynamics of the $x_t^\pi,s_t^\pi$ and the observation-action process: 
\begin{align*}\label{eq:PolicyDynamics}
\begin{bmatrix}
x^\pi_{t+\!1}\\
s^\pi_{t+\!1}
\end{bmatrix}
&\!\!=\!\!
\underbrace{\begin{aligned}
\begin{bmatrix}
A \!+\! BD_\pi C \! \!&\! \! BC_{\pi}\\
B_{\pi}C \! \!&\!\! A_{\pi}
\end{bmatrix}\end{aligned}}_\text{$A'_{\pi}$}\!
\begin{bmatrix}
x^\pi_{t}\\
s^\pi_{t}
\end{bmatrix}
\!\!+\!\!\underbrace{\begin{aligned}\begin{bmatrix}
I_{n} \! \!\!&\!\! \! BD_{\pi}\\
0_{s\!\times\! n} \! \!\!&\!\! \! B_{\pi}
\end{bmatrix}\end{aligned}}_\text{$B'_{\pi}$} \!
\begin{bmatrix}
w_{t}\\
z_{t}
\end{bmatrix}\!, \enskip\!
\begin{bmatrix}
y^\pi_{t}\\
u^\pi_{t}
\end{bmatrix} \!\!=\!\!
\underbrace{\begin{aligned}\begin{bmatrix}
C \! \!&\! \! 0_{s\!\times\! d}\\
D_{\pi}C \! \!&\! \! C_{\pi}
\end{bmatrix}\end{aligned}}_\text{$C'_{\pi}$}\!
\begin{bmatrix}
x^\pi_{t}\\
s^\pi_{t}
\end{bmatrix}
\!\!+\!\!\underbrace{\begin{aligned}\begin{bmatrix}
0_{d\!\times\! n} \! \!\!&\!\! \! I_{d}\\
0_{m\!\times\! n} \!\!\!&\!\! \! D_{\pi}
\end{bmatrix}\end{aligned}}_\text{$D'_{\pi}$}\!
\begin{bmatrix}
w_{t}\\
z_{t}
\end{bmatrix}
\end{align*}
where $(A_\pi',B_\pi',C_\pi',D_\pi')$ are the associated parameters of induced closed-loop system. We define the Markov operator for the system $(A_\pi',B_\pi',C_\pi',D_\pi')$, as $\mathbf{G_\pi'}=\lbrace {G'_\pi}^{[i]}\rbrace_{i=0}$, where ${G'_\pi}^{[0]}=D_\pi'$, and $\forall i>0$,  ${G'_\pi}^{[i]}=C_\pi'A_\pi'^{i-1}B_\pi'$. Let 
$B'_{\pi,z}\coloneqq [D_{\pi}^\top B^\top \enskip
B_{\pi}^\top]^\top$ and $C'_{\pi,u} \coloneqq \begin{bmatrix}
D_{\pi}C & C_{\pi}
\end{bmatrix}$. 

\textbf{Proper Decay Function:} Let $\psi : \mathbb{N} \rightarrow \mathbb{R}_{\geq 0}$ be a proper decay function, such that $\psi$ is non-increasing and $\lim_{h'\rightarrow \infty}\psi(h')=0$. For a Markov operator $\Markov$, $\psi_\Markov(h)$ defines the induced decay function on $\Markov$, \textit{i.e.}, $\psi_\Markov(h) \coloneqq \sum_{i\geq h}\|{G}^{[i]}\|$. $\Pi(\psi)$ denotes the class of \LDC policies associated with a proper decay function $\psi$, such that for all $\pi\in\Pi(\psi)$, and all $h\geq 0$, $\sum_{i\geq h}\|{G'_\pi}^{[i]}\|\leq \psi(h)$. Let $\kappa_\psi \coloneqq \psi(0)$.

\textbf{\DFC Policy:} A \DFC policy of length $H'$ is defined with a set of parameters,  $\Mcontrol(H'):=\lbrace M^{[i]} \rbrace_{i=0}^{H'-1}$,
prescribing the control input of $u^\Mcontrol_t = \sum_{i=0}^{H'-1}M^{[i]}\nature_{t-i}(\Markov)$, and resulting in state $x^\Mcontrol_{t+1}$ and observation $y^\Mcontrol_{t+1}$. In the following, directly using the analysis in \citet{simchowitz2020improper}, we show that for any $\pi\in\Pi(\psi)$ and any input $u^\pi_t$ at time step $t$, there is a set of parameters $\Mcontrol(H')$, such that $u^\Mcontrol_t$ is sufficiently close to $u^\pi_t$, and the resulting $y^\pi_t$ is sufficiently close to $y^\Mcontrol_t$.

\begin{lemma} \label{lem:deviation}
Suppose $\|\nature_t(\Markov)\| \!\leq\! \kappa_\nature$  for all $t\leq T$ for some $\kappa_\nature$. For any \LDC policy $\pi\in\Pi(\psi)$, there exist a $H'$ length \DFC policy $\Mcontrol(H')$ such that $\|u^\pi_t \!-\! u^\Mcontrol_t\| \!\leq\! \psi(H') \kappa_\nature,$ and $\|y^\pi_t \!-\! y^\Mcontrol_t\| \!\leq\! \psi(H') \kappa_\Markov\kappa_\nature$.
One of the \DFC policies that satisfy this is $M^{[0]}\!=\!D_\pi$, and $M^{[i]}\!=\!C'_{\pi,u}{A'_{\pi}}^{\!i-1}B'_{\pi,z}$ for $0\!<\!i\!<\!H'$.  
\end{lemma}

\begin{proof}\label{apx:lem:deviation}

Let 
$B'_{\pi,w}\coloneqq [I_{n\times n}^\top \enskip
0_{s\times n}^\top]^\top$, and 
$B'_{\pi,z}\coloneqq [D_{\pi}^\top B^\top \enskip
B_{\pi}^\top]^\top$, 
the columns of $B'_{\pi}$ applied on process noise, and measurement noise respectively. Similarly $C'_{\pi,y} \coloneqq \begin{bmatrix}
C & 0_{s\times d}\\
\end{bmatrix}$ and $C'_{\pi,u} \coloneqq \begin{bmatrix}
D_{\pi}C & C_{\pi}
\end{bmatrix}$ are rows of $C'_{\pi}$ generating the observation and action. 

Rolling out the dynamical system defining a policy $\pi$ in (\ref{eq:LDC}), we can restate the action $u_t^\pi$ as follows,
\begin{align*}
u_t^\pi &= D_\pi z_t +\sum_{i=1}^{t-1}C'_{\pi,u}{A'_{\pi}}^{i-1}B'_{\pi,z}z_{t-i}+\sum_{i=1}^{t-1} C'_{\pi,u}{A'_{\pi}}^{i-1}B'_{\pi,w}w_{t-i}\\
&=D_\pi z_t +\sum_{i=1}^{t-1}C'_{\pi,u}{A'_{\pi}}^{i-1}B'_{\pi,z}z_{t-i}+C'_{\pi,u}B'_{\pi,w}w_{t-1}+\sum_{i=2}^{t-1} C'_{\pi,u}{A'_{\pi}}^{i-1}B'_{\pi,w}w_{t-i}\\
&=D_\pi z_t +\sum_{i=1}^{t-1}C'_{\pi,u}{A'_{\pi}}^{i-1}B'_{\pi,z}z_{t-i}+D_\pi C w_{t-1}+\sum_{i=2}^{t-1} C'_{\pi,u}{A'_{\pi}}^{i-1}B'_{\pi,w}w_{t-i}\\
\end{align*}
Note that ${A'_{\pi}}^{i-1}B'_{\pi,w}$ is equal to $\begin{bmatrix}
A + BD_\pi C \\
B_{\pi}C 
\end{bmatrix}$. Based on the definition of $A'_\pi$ in  (\ref{eq:LDC}), we restate $A'_\pi$ as follows,
\begin{align*}
A'_\pi=\begin{bmatrix}
A+BD_\pi C & BC_{\pi}\\
B_{\pi}C & A_{\pi}
\end{bmatrix}=\begin{bmatrix}
BD_\pi C & BC_{\pi}\\
B_{\pi}C & A_{\pi}
\end{bmatrix}+\begin{bmatrix}
A & 0_{n\times s}\\
0_{s\times n} & 0_{s\times s}
\end{bmatrix}
\end{align*}
For any given bounded matrices $A'_\pi$ and $A$, and any integer $i>0$, we have
\begin{align*}
{A'_\pi}^i\!\!=\!\!\begin{bmatrix}
A+BD_\pi C \!&\! BC_{\pi}\\
B_{\pi}C \!&\! A_{\pi}
\end{bmatrix}^i&=\begin{bmatrix}
A+BD_\pi C \!&\! BC_{\pi}\\
B_{\pi}C \!&\! A_{\pi}
\end{bmatrix}^{i-1}\begin{bmatrix}
BD_\pi C \!&\! BC_{\pi}\\
B_{\pi}C \!&\! A_{\pi}
\end{bmatrix}\!+\!\begin{bmatrix}
A+BD_\pi C \!&\! BC_{\pi}\\
B_{\pi}C \!&\! A_{\pi}
\end{bmatrix}^{i-1}\begin{bmatrix}
A \!\!&\!\! 0_{n\times s}\\
0_{s\times n} \!\!&\!\! 0_{s\times s}
\end{bmatrix}\\
&=\begin{bmatrix}
A+BD_\pi C & BC_{\pi}\\
B_{\pi}C & A_{\pi}
\end{bmatrix}^{i-1}\begin{bmatrix}
BD_\pi C & BC_{\pi}\\
B_{\pi}C & A_{\pi}
\end{bmatrix}\\
&\quad\quad\quad+\begin{bmatrix}
A+BD_\pi C & BC_{\pi}\\
B_{\pi}C & A_{\pi}
\end{bmatrix}^{i-2}\begin{bmatrix}
BD_\pi C & BC_{\pi}\\
B_{\pi}C & A_{\pi}
\end{bmatrix}\begin{bmatrix}
A & 0_{n\times s}\\
0_{s\times n} & 0_{s\times s}
\end{bmatrix}\\
&\quad\quad\quad\quad\quad\quad+\begin{bmatrix}
A+BD_\pi C \!&\! BC_{\pi}\\
B_{\pi}C \!&\! A_{\pi}
\end{bmatrix}^{i-2}\begin{bmatrix}
A^2 & 0_{n\times s}\\
0_{s\times n} & 0_{s\times s}
\end{bmatrix}\\
&\quad\vdots\\
&=\begin{bmatrix}
A^i & 0_{n\times s}\\
0_{s\times n} & 0_{s\times s}
\end{bmatrix}+ \sum_{j=1}^i{A'_\pi}^{j-1}\begin{bmatrix}
BD_\pi C & BC_{\pi}\\
B_{\pi}C & A_{\pi}
\end{bmatrix}\begin{bmatrix}
A^{i-j} & 0_{n\times s}\\
0_{s\times n} & 0_{s\times s}
\end{bmatrix}
\end{align*}
We use this decomposition to relate $u^\pi_t$ and $u^\Mcontrol_t$. 
Now considering ${A'_{\pi}}^{i-1}B'_{\pi,w}$, for $i-1>0$ we have 

\begin{align*}
{A'_\pi}^{i-1}B'_{\pi,w}&=\begin{bmatrix}
A^{i-1} \\
0_{s\times n}
\end{bmatrix}+ \sum_{j=1}^{i-1}{A'_\pi}^{j-1}\begin{bmatrix}
BD_\pi C & BC_{\pi}\\
B_{\pi}C & A_{\pi}
\end{bmatrix}\begin{bmatrix}
A^{i-1-j} \\
0_{s\times n}
\end{bmatrix}
=\begin{bmatrix}
A^{i-1} \\
0_{s\times n}
\end{bmatrix}+ \sum_{j=1}^{i-1}{A'_\pi}^{j-1}B'_{\pi,z}
CA^{i-1-j} \\
\end{align*}
Using this equality in the derivation of $u_t^\pi$ we derive,

\begin{align*}
u_t^\pi &=D_\pi z_t +\sum_{i=1}^{t-1}C'_{\pi,u}{A'_{\pi}}^{i-1}B'_{\pi,z}z_{t-i}+D_\pi C w_{t-1}\\
&\quad\quad\quad\quad\quad\quad+\sum_{i=2}^{t-1}\begin{bmatrix}
D_{\pi}C & C_{\pi}
\end{bmatrix}\begin{bmatrix}
A^{i-1} \\
0_{s\times n}
\end{bmatrix} w_{t-i}+\sum_{i=2}^{t-1} C'_{\pi,u}\sum_{j=1}^{i-1}{A'_\pi}^{j-1}B'_{\pi,z}CA^{i-1-j}w_{t-i}\\
&=D_\pi z_t +\sum_{i=1}^{t-1}C'_{\pi,u}{A'_{\pi}}^{i-1}B'_{\pi,z}z_{t-i}+\sum_{i=1}^{t-1}D_{\pi}C
A^{i-1}
 w_{t-i}+\sum_{i=2}^{t-1}\sum_{j=1}^{i-1}C'_{\pi,u}{A'_\pi}^{j-1}B'_{\pi,z}CA^{i-1-j}w_{t-i}\\
\end{align*}
Note that $\nature_t(\Markov)=z_t+\sum_{i=1}^{t-1}CA^{t-i-1}w_i=z_t+\sum_{i=1}^{t-1}CA^{i-1}w_{t-i}$. Inspired by this expression, we rearrange the previous sum as follows:

\begin{align*}
u_t^\pi &=D_\pi \left(z_t+\sum_{i=1}^{t-1}C
A^{i-1}
 w_{t-i}\right) +\sum_{i=1}^{t-1}C'_{\pi,u}{A'_{\pi}}^{i-1}B'_{\pi,z}z_{t-i}+\sum_{i=2}^{t-1}\sum_{j=1}^{i-1}C'_{\pi,u}{A'_\pi}^{j-1}B'_{\pi,z}CA^{i-1-j}w_{t-i}\\
&=D_\pi \left(z_t+\sum_{i=1}^{t-1}C
A^{i-1}
 w_{t-i}\right) +\sum_{i=1}^{t-1}C'_{\pi,u}{A'_{\pi}}^{i-1}B'_{\pi,z}z_{t-i}+\sum_{j=1}^{t-2}\sum_{i=j+1}^{t-1}C'_{\pi,u}{A'_\pi}^{j-1}B'_{\pi,z}CA^{i-1-j}w_{t-i}\\
&=D_\pi \left(z_t+\sum_{i=1}^{t-1}C
A^{i-1}
 w_{t-i}\right) +\sum_{i=1}^{t-1}C'_{\pi,u}{A'_{\pi}}^{i-1}B'_{\pi,z}z_{t-i}+\sum_{j=1}^{t-2}C'_{\pi,u}{A'_\pi}^{j-1}B'_{\pi,z}\sum_{i=1}^{t-j-1}CA^{t-j-i-1}w_{i}\\
&=D_\pi \nature_t +\sum_{i=1}^{t-1}C'_{\pi,u}{A'_{\pi}}^{i-1}B'_{\pi,z}\nature_{t-i}
\end{align*}
Now setting $M^{[0]}=D_\pi$, and $M^{[i]}=C'_{\pi,u}{A'_{\pi}}^{i-1}B'_{\pi,z}$ for all $0<i<H'$, we conclude that for any \LDC policy $\pi\in\Pi$, there exists at least one length $H'$ \DFC policy $\Mcontrol(H')$ such that 
\begin{align*}
    u^\pi_t - u^\Mcontrol_t = \sum_{i=H'}^{t}C'_{\pi,u}{A'_{\pi}}^{i-1}B'_{\pi,z}\nature_{t-i}   
\end{align*}
Using Cauchy Schwarz inequality we have
\begin{align*}
    \|u^\pi_t - u^\Mcontrol_t\| \leq \left \|\sum_{i=H'}^{t}C'_{\pi,u}{A'_{\pi}}^{i-1}B'_{\pi,z}\nature_{t-i} \right \|\leq \psi(H') \kappa_\nature 
\end{align*}
which states the first half of the Lemma. 

Using the definition of $y_t^\pi$ in (\ref{eq:LDC}), we have
\begin{align*}
    y_t^\pi = z_t + \sum_{i=1}^{t-1} CA^{t-i-1}w_{i} + \sum_i^{t-1} G^{[i]} u^\pi_{t-i}.
\end{align*}
Similarly for $y^\Mcontrol_t$ we have,
\begin{align*}
    y^\Mcontrol_t = z_t + \sum_{i=1}^{t-1} CA^{t-i-1}w_{i} + \sum_i^{t-1} G^{[i]} u^\pi_{t-i}.
\end{align*}
Subtracting these two equations, we derive,
\begin{align*}
    y^\pi_t - y^\Mcontrol_t = \sum_i^{t-1} G^{[i]} u^\pi_{t-i}-\sum_i^{t-1} G^{[i]} u^\Mcontrol_{t-i} =\sum_i^{t-1} G^{[i]} (u^\pi_{t-i}-u^\Mcontrol_{t-i})
\end{align*}
resulting in
\begin{align*}
    \|y^\pi_t - y^\Mcontrol_t\| \leq \psi(H') \kappa_\Markov\kappa_\nature 
\end{align*}
which states the second half of the Lemma. 
\end{proof}
This lemma further entails that any stabilizing \LDC can be well approximated by a \DFC that belongs to the following set of \DFC{}s,
\begin{equation*}
    \Mcontrolset \left( H', \kappa_\psi \right) = \Big \{\Mcontrol(H'):=\lbrace M^{[i]} \rbrace_{i=0}^{H'-1} : \sum\nolimits_{i\geq 0}^{H'-1}\|M^{[i]}\|\leq \kappa_\psi \Big\},
\end{equation*}
indicating that using the class of \DFC policies as an approximation to \LDC policies is justified. 

\section{Model Parameters Identification Procedure, \textsc{{SysId}}\xspace }\label{apx:Identification}
Algorithm \ref{SYSID} gives the model parameters identification algorithm, \Sys, that is executed after recovering $\modelearni$ in the beginning of each epoch $i$. \Sys is similar to Ho-Kalman method~\citep{ho1966effective} and estimates the system parameters from $\modelearni$. 

First of all, notice that $\modelearn = [\mathbf{F}, \mathbf{G}]$ where 
\begin{align*}
    \mathbf{F} &=  \left[CF, \enskip C\Bar{A}F, \enskip \ldots, \enskip C \Bar{A}^{\Hes-1} F \right] \in \mathbb{R}^{m \times m\Hes}, \\
    \mathbf{G} &= \left[ CB, \!\enskip C\Bar{A}B, \enskip \ldots, \enskip C \Bar{A}^{\Hes-1} B \right] \in \mathbb{R}^{m \times p\Hes}.
\end{align*}
Given the estimate for the truncated ARX model
\begin{align*}
   \modelearni = [\mathbf{\hat{F}_{i,1}},\ldots, \mathbf{\hat{F}_{i,\Hes}}, \mathbf{\hat{G}_{i,1}},\ldots, \mathbf{\hat{G}_{i,\Hes}}],
\end{align*}
where $\mathbf{\hat{F}_{i,j}}$ is the $j$'th $m \times m$ block of $\mathbf{\hat{F}_{i}}$, and $\mathbf{\hat{G}_{i,j}}$ is the $j$'th $m \times p$ block of $\mathbf{\hat{G}_{i}}$ for all $1 \leq j \leq \Hes$. \Sys constructs two $d_1 \times (d_2+1)$ Hankel matrices $\mathbf{\mathcal{H}_{\hat{F}_i}}$ and $\mathbf{\mathcal{H}_{\hat{G}_i}}$ such that $(j,k)$th block of Hankel matrix is $\mathbf{\hat{F}_{i,(j+k-1)}}$ and $\mathbf{\hat{G}_{i,(j+k-1)}}$ respectively. Then, it forms the following matrix $\hat{\mathcal{H}}_i$.
\begin{align*}
    \hat{\mathcal{H}}_i = \left[ \mathbf{\mathcal{H}_{\hat{F}_i}}, \enskip \mathbf{\mathcal{H}_{\hat{G}_i}} \right].
\end{align*}
Recall that from $(A,F)$-controllability of $\Theta$, we have that $\| T \Bar{A} T^{-1}\| \leq \upsilon <1$ for some similarity transformation $T$ and the dimension of latent state, $n$, is the order of the system for the observable and controllable system. For $\Hes\geq \max \left\{2n+1, \frac{\log(c_H T^2 \sqrt{m} / \sqrt{\lambda})}{\log(1/\upsilon)} \right\}$ for some problem dependent constant $c_H$, we can pick $d_1 \geq n$ and $d_2 \geq n$ such $d_1+d_2+1 = \Hes$. This guarantees that the system identification problem is well-conditioned. Using  Definition \ref{c_o_def}, if the input to the \Sys was $\modelearn = [\mathbf{F}, \mathbf{G}]$ then constructed Hankel matrix, $\mathcal{H}$ would be rank $n$, 
\begin{align*}
     \mathcal{H} &=  [ C^\top, ~\ldots, \enskip (C\Bar{A}^{d_1-1}) ^\top] ^\top [F,~\ldots,~ \Bar{A}^{d_2}F,~B,~\ldots,~\Bar{A}^{d_2}B]\\
     &=\mathbf{O}(\bar{A},C,d_1) \enskip [\mathbf{C}(\bar{A},F,d_2+1),\quad \Bar{A}^{d_2}F, \quad \mathbf{C}(\bar{A},B,d_2+1), \quad \Bar{A}^{d_2}B] \\
     &= \mathbf{O}(\bar{A},C,d_1) \enskip [F,\quad \Bar{A}\mathbf{C}(\bar{A},F,d_2+1), \quad B, \quad \Bar{A}\mathbf{C}(\bar{A},B,d_2+1)].
\end{align*}
where for all $k\geq n$, $\mathbf{C}(A,B,k)$ defines the extended $(A, B)$ controllability matrix and $\mathbf{O}(A,C,k)$ defines the extended $(A, C)$ observability matrix. Notice that $\modelearn$ and $\mathcal{H}$ are uniquely identifiable for a given system $\Theta$, whereas for any invertible $\mathbf{T}\in \R^{n \times n}$, the system resulting from
\begin{align*}
    A' = \mathbf{T}^{-1}A\mathbf{T},~ B' = \mathbf{T}^{-1}B,~ C' = C \mathbf{T},~ F' = \mathbf{T}^{-1}F 
\end{align*}
gives the same $\modelearn$ and $\mathcal{H}$. Similar to Ho-Kalman algorithm, \Sys computes the SVD of $\modelearni$ and estimates the extended observability and controllability matrices and eventually system parameters up to similarity transformation. To this end, \Sys constructs $\hat{\mathcal{H}}_i^-$ by discarding $(d_2 + 1)$th and $(2d_2 + 2)$th block columns of $\hat{\mathcal{H}}_i$, \textit{i.e.} if it was $\mathcal{H}$ then we have,
\begin{align*}
    \mathcal{H}^- = \mathbf{O}(\bar{A},C,d_1) \enskip [\mathbf{C}(\bar{A},F,d_2+1), \quad \mathbf{C}(\bar{A},B,d_2+1)].
\end{align*}
The \Sys procedure then calculates $\hat{\mathcal{N}}_i$, the best rank-$n$ approximation of $\hat{\mathcal{H}}_i^-$, obtained by setting its all but top $n$ singular values to zero. The estimates of $\mathbf{O}(\bar{A},C,d_1)$,  $\mathbf{C}(\bar{A},F,d_2+1)$ and $\mathbf{C}(\bar{A},B,d_2+1)$ are given as 
\begin{equation*}
    \hat{\mathcal{N}}_i = \mathbf{U_i} \mathbf{\Sigma_i}^{1/2}~ \mathbf{\Sigma_i}^{1/2}\mathbf{V_i}^\top =  \mathbf{\hat{O}_i}(\bar{A},C,d_1) \enskip [\mathbf{\hat{C}_i}(\bar{A},F,d_2+1), \quad \mathbf{\hat{C}_i}(\bar{A},B,d_2+1)].
\end{equation*}
From these estimates \Sys recovers $\hat{C}_i$ as the first $m \times n$ block of $\mathbf{\hat{O}_i}(\bar{A},C,d_1)$, $\hat{B}_i$ as the first $n \times p$ block of $\mathbf{\hat{C}_i}(\bar{A},B,d_2+1)$ and $\hat{F}_i$ as the first $n \times m$ block of $\mathbf{\hat{C}_i}(\bar{A},F,d_2+1)$. Let $\hat{\mathcal{H}}_i^+$ be the matrix obtained by discarding $1$st and $(d_2+2)$th block columns of $\hat{\mathcal{H}}_i$, \textit{i.e.} if it was $\mathcal{H}$ then
\begin{align*}
    \mathcal{H}^+ = \mathbf{O}(\bar{A},C,d_1) \enskip \Bar{A} \enskip [\mathbf{C}(\bar{A},F,d_2+1), \quad \mathbf{C}(\bar{A},B,d_2+1)].
\end{align*}
Therefore, \Sys recovers $\hat{\Bar{A}}_i$ as,
\begin{align*}
    \hat{\Bar{A}}_i = \mathbf{\hat{O}_i}^\dagger(\bar{A},C,d_1) \enskip \hat{\mathcal{H}}_t^+ \enskip [\mathbf{\hat{C}_t}(\bar{A},F,d_2+1), \quad \mathbf{\hat{C}_t}(\bar{A},B,d_2+1)]^\dagger.
\end{align*}
Using the definition of $\Bar{A} = A - FC$, the algorithm obtains $\hat{A}_t = \hat{\Bar{A}}_t + \hat{F}_t \hat{C}_t$. 

\begin{algorithm}[H] 
 \caption{\Sys}
  \begin{algorithmic}[1] 
  \State {\bfseries Input:} $\modelearni$, $\Hes$, system order $n$, $d_1, d_2$ such that $d_1 + d_2 + 1 = \Hes$ 
  \State Form two $d_1 \times (d_2+1)$ Hankel matrices $\mathbf{\mathcal{H}_{\hat{F}_i}}$ and $\mathbf{\mathcal{H}_{\hat{G}_i}}$  from $\modelearni = [\mathbf{\hat{F}_{i,1}},\ldots, \mathbf{\hat{F}_{i,\Hes}}, \mathbf{\hat{G}_{i,1}},\ldots, \mathbf{\hat{G}_{i,\Hes}}],$ and construct $\hat{\mathcal{H}}_i = \left[ \mathbf{\mathcal{H}_{\hat{F}_i}}, \enskip \mathbf{\mathcal{H}_{\hat{G}_i}} \right] \in \R^{md_1 \times (m+p)(d_2+1)}$
  \State Obtain $\hat{\mathcal{H}}_i^-$ by discarding $(d_2 + 1)$th and $(2d_2 + 2)$th block columns of $\hat{\mathcal{H}}_i$
  \State Using SVD obtain $\hat{\mathcal{N}}_i \in \R^{m d_1 \times (m+p) d_2}$, the best rank-$n$ approximation of $\hat{\mathcal{H}}_i^-$
  \State Obtain  $\mathbf{U_i},\mathbf{\Sigma_i},\mathbf{V_i} = \text{SVD}(\hat{\mathcal{N}}_i)$
  \State Construct $\mathbf{\hat{O}_i}(\bar{A},C,d_1) = \mathbf{U_i}\mathbf{\Sigma_t}^{1/2} \in \R^{md_1 \times n}$
  \State Construct $[\mathbf{\hat{C}_i}(\bar{A},F,d_2+1), \enskip \mathbf{\hat{C}_i}(\bar{A},B,d_2+1)] = \mathbf{\Sigma_i}^{1/2}\mathbf{V_i} \in \R^{n \times (m+p)d_2}$
  \State Obtain $\hat{C}_i\in \R^{m\times n}$, the first $m$ rows of $\mathbf{\hat{O}_i}(\bar{A},C,d_1)$
  \State Obtain $\hat{B}_i\in \R^{n\times p}$, the first $p$ columns of $\mathbf{\hat{C}_i}(\bar{A},B,d_2+1)$
  \State Obtain $\hat{F}_i\in \R^{n\times m}$, the first $m$ columns of $\mathbf{\hat{C}_i}(\bar{A},F,d_2+1)$
  \State Obtain $\hat{\mathcal{H}}_i^+$ by discarding $1$st and $(d_2+2)$th block columns of $\hat{\mathcal{H}}_i$
  \State Obtain $\hat{\Bar{A}}_i = \mathbf{\hat{O}_i}^\dagger(\bar{A},C,d_1) \enskip \hat{\mathcal{H}}_i^+ \enskip [\mathbf{\hat{C}_i}(\bar{A},F,d_2+1), \quad \mathbf{\hat{C}_i}(\bar{A},B,d_2+1)]^\dagger$
  \State Obtain $\hat{A}_i = \hat{\Bar{A}}_i + \hat{F}_i \hat{C}_i$
  \end{algorithmic}
 \label{SYSID}  
\end{algorithm}

\section{\textsc{{AdaptOn}}\xspace}\label{apx:adapton}

\begin{algorithm}[H]
\caption{\alg}
  \begin{algorithmic}[1]
 \State \textbf{Input:} $T$, $H$, $H'$, $\Tburn$, $\Mcontrolset$ \\

 ------ \textsc{\small{Warm-Up}} ------------------------------------------------------------------------------------
\For{$t = 1, \ldots, \Tburn$}
\State Deploy $u_t \!\sim\! \mathcal{N}(0,\sigma_u^2 I)$
\EndFor
\State Store $\mathcal{D}_\Tburn \!=\! \lbrace y_t,u_t \rbrace_{t=1}^{\Tburn}$, set $t_1\!=\!\Tbase\!=\!\Tburn$, $t\!=\!\Tbase\!+\!1$, and $\Mcontrol_t$ as any member of $\Mcontrolset$\\
------ \textsc{\small{Adaptive Control}} --------------------------------------------------------------------------
\For{$i = 1, 2, \ldots$}
    \State Solve (\ref{new_lse}) using $\mathcal{D}_{t}$, estimate $\hat{A}_i, \hat{B}_i, \hat{C}_i$ using \Sys(Alg. \ref{SYSID}) and construct $\wh \Markov_i(H)$  
    \State Compute $\nature_\tau(\wh \Markov_i):=y_\tau - \sum_{j=1}^{H} \wh G_i^{[j]} u_{\tau-j}$, $\forall \tau\leq t$ 
    \While{ $t\leq t_i + 2^{i-1}\Tbase~~~ \textit{and} ~~t\leq T$}
        \State Observe $y_t$, and compute $\nature_t(\wh \Markov_i):=y_t - \sum_{j=1}^{H} \wh G_i^{[j]} u_{t-j}$
        \State Commit to $u_t = \sum_{j=0}^{H'-1}M_t^{[j]} \nature_{t-j}(\wh \Markov_i)$, observe $\ell_t$, and pay a cost of $\ell_t(y_t,u_t)$
        \State Update $\Mcontrol_{t+1}=\proj_\Mcontrolset\left(\Mcontrol_t-\eta_t\nabla f_t\left(\Mcontrol_t,\wh \Markov_i\right)
        \right)$,~ $\mathcal{D}_{t+1}=\mathcal{D}_{t}\cup\lbrace y_t,u_t\rbrace$, set $t\!=\!t\!+\!1$
    \EndWhile
    \State $t_{i+1}=t_i + 2^{i-1}\Tbase$
\EndFor
  \end{algorithmic}
 \label{algo} 
\end{algorithm}

\section{Proof of Theorem \ref{theo:sysid}}
\label{apx:generalsysidproof}
In this section, we provide the proof of Theorem \ref{theo:sysid} with precise expressions. In Appendix \ref{apx:selfnorm}, we show the self-normalized error bound on the (\ref{new_lse}), Theorem \ref{theo:closedloopid}. In Appendix \ref{apx:2normbound}, assuming persistence of excitation, we convert the self-normalized bound into a Frobenius norm bound to be used for parameter estimation error bounds in Appendix \ref{apx:modelestbound}, Theorem \ref{ConfidenceSets}. Finally, we consider the Markov parameter estimates constructed via model parameter estimates in Appendix \ref{apx:markovestimate}, which concludes the proof of Theorem \ref{theo:sysid}.

\subsection{Self-Normalized Bound on Finite Sample Estimation Error of (\ref{new_lse})} \label{apx:selfnorm}
First consider the effect of truncation bias term, $C\Bar{A}^{\Hes} x_{t-\Hes}$.
Notice that $\Bar{A}$ is stable due to $(A,F)$-controllability of $\Theta$~\citep{kailath2000linear}, \textit{i.e.,} $\| T \Bar{A} T^{-1}\| \leq \upsilon <1$ for some similarity transformation $T$. Thus, $C\Bar{A}^{\Hes} x_{t-\Hes}$ is order of $\upsilon^H$. In order to get consistent estimation, for some problem dependent constant $c_H$, we set $\Hes \geq \frac{\log(c_H T^2 \sqrt{m} / \sqrt{\lambda})}{\log(1/\upsilon)}$, resulting in a negligible bias term of  order $1/T^2$. Using this we first obtain a self-normalized finite sample estimation error of (\ref{new_lse}):

\begin{theorem}[Self-normalized Estimation Error]
\label{theo:closedloopid}
Let $\wh{\mathcal{G}}_{\mathbf{y}}$ be the solution to (\ref{new_lse}) at time $\tau$. For the given choice of $\Hes$, define 
\begin{align*}
    V_\tau = \lambda I + \sum_{i=\Hes}^{\tau} \phi_i \phi_i^\top.
\end{align*}
Let $\|\mathbf{M}\|_F \leq S$. For $\delta \in (0,1)$, with probability at least $1-\delta$, for all $t\leq \tau$, $\modelearn$ lies in the set $\mathcal{C}_{\modelearn}(t)$, where 
\begin{equation*}
    \mathcal{C}_{\modelearn}(t) = \{ \modelearn': \Tr((\wh{\mathcal{G}}_{\mathbf{y}} - \modelearn')V_t(\wh{\mathcal{G}}_{\mathbf{y}}-\modelearn')^{\top}) \leq \beta_\tau \},
\end{equation*}
for $\beta_\tau$ defined as follows,
\begin{equation*}
    \beta_\tau = \left(\sqrt{m\| C \sig  C^\top + \sigma_z^2 I \| \log  \left(\frac{\operatorname{det}\left(V_\tau \right)^{1 / 2}}{\delta \operatorname{det}(\lambda I)^{1 / 2} }\right)} + S\sqrt{\lambda} +\frac{\tau \sqrt{\Hes}}{T^2} \right)^2.
\end{equation*}
\end{theorem}
\begin{proof}

For a single input-output trajectory $\{y_t, u_t \}^{\tau}_{t=1}$, where $\tau \leq T$, using the representation in (\ref{arx_rollout}), we can write the following for the given system,
\begin{equation} 
    Y_\tau = \Phi_\tau \modelearn^{\top} + \underbrace{ E_\tau + N_\tau}_{ \text{Noise}} \qquad \text{where}
\end{equation}
\vspace{-0.4cm}
\begin{align*}
    \modelearn & = \left[CF,\enskip C\Bar{A}F,\enskip \ldots, \enskip C \Bar{A}^{\Hes-1}F, \enskip CB, \enskip C\Bar{A}B, \enskip \ldots, \enskip C\Bar{A}^{\Hes-1} B \right] \in \R^{m \times (m+p)\Hes} \\
    Y_\tau &= \left[ y_{\Hes},~y_{\Hes+1},~ \ldots,~y_\tau \right]^\top \in \mathbb{R}^{(\tau-\Hes) \times m} \\
    \Phi_\tau &= \left[ \phi_{\Hes},~\phi_{\Hes+1},~\ldots,~\phi_\tau \right]^\top \in \mathbb{R}^{(\tau-\Hes) \times (m+p)\Hes} \\
    E_\tau &= \left[ e_{\Hes},~e_{\Hes+1},~\ldots,~e_\tau \right]^\top \in \mathbb{R}^{(\tau-\Hes) \times m} \\
    N_\tau\!&=\! \left[C \Bar{A}^{\Hes}\!x_{0},~C \Bar{A}^{\Hes}\!x_{1},\ldots,C \Bar{A}^{\Hes} \!x_{\tau-\Hes} \right]^\top \!\!\in\! \mathbb{R}^{(\tau-\Hes) \times m}. 
\end{align*}
$\wh{\mathcal{G}}_{\mathbf{y}}$ is the solution to (\ref{new_lse}), \textit{i.e.,} $\min_X \|Y_{\tau} - \Phi_{\tau} X^{\top}\|^2_F + \lambda \|X\|^2_F $. Hence, we get $\wh{\mathcal{G}}_{\mathbf{y}}^\top = (\Phi_\tau^\top \Phi_\tau + \lambda I)^{-1} \Phi_\tau^\top Y_\tau $.

\begin{align*}
    \wh{\mathcal{G}}_{\mathbf{y}}  & = \big[(\Phi_\tau^\top \Phi_\tau + \lambda I)^{-1} \Phi_\tau^\top(\Phi_\tau \modelearn^{\top} + E_\tau + N_\tau)\big]^{\top} \\
    &= \big[(\Phi_\tau^\top \Phi_\tau + \lambda I)^{-1} \Phi_\tau^\top \left(E_\tau + N_\tau \right) + (\Phi_\tau^\top \Phi_\tau + \lambda I)^{-1} \Phi_\tau^\top \Phi_\tau \modelearn^{\top} \\
    &\qquad + \lambda (\Phi_\tau^\top \Phi_\tau + \lambda I)^{-1}\modelearn^{\top} - \lambda (\Phi_\tau^\top \Phi_\tau + \lambda I)^{-1}\modelearn^{\top} \big]^{\top} \\
    &= \big[(\Phi_\tau^\top \Phi_\tau + \lambda I)^{-1} \Phi_\tau^\top E_\tau + (\Phi_\tau^\top \Phi_\tau + \lambda I)^{-1} \Phi_\tau^\top N_\tau + \modelearn^{\top} - \lambda (\Phi_\tau^\top \Phi_\tau + \lambda I)^{-1}\modelearn^{\top} \big]^{\top}
\end{align*}
Using $\wh{\mathcal{G}}_{\mathbf{y}}$, we get 
\begin{align}
    &|\Tr(X(\wh{\mathcal{G}}_{\mathbf{y}}-\modelearn)^{\top})| \\
    &= |\Tr(X (\Phi_\tau^\top \Phi_\tau + \lambda I)^{-1} \Phi_\tau^\top E_\tau) + \Tr(X (\Phi_\tau^\top \Phi_\tau + \lambda I)^{-1} \Phi_\tau^\top N_\tau)   - \lambda \Tr(X (\Phi_\tau^\top \Phi_\tau + \lambda I)^{-1}\modelearn^{\top})| \nonumber \\
    &\leq |\Tr(X (\Phi_\tau^\top \Phi_\tau + \lambda I)^{-1} \Phi_\tau^\top E_\tau)| + |\Tr(X (\Phi_\tau^\top \Phi_\tau + \lambda I)^{-1} \Phi_\tau^\top N_\tau)| + \lambda |\Tr(X (\Phi_\tau^\top \Phi_\tau + \lambda I)^{-1}\modelearn^{\top})| \nonumber \\
    &\leq \sqrt{\Tr(X(\Phi_\tau^\top \Phi_\tau + \lambda I)^{-1} X^{\top})\Tr(E_\tau^{\top}\Phi_\tau(\Phi_\tau^{\top}\Phi_\tau + \lambda I)^{-1}\Phi_\tau^{\top}E_\tau )}  \label{CS_trace} \\
    &\quad + \sqrt{\Tr(X(\Phi_\tau^\top \Phi_\tau + \lambda I)^{-1} X^{\top})\Tr(N_\tau^{\top}\Phi_\tau(\Phi_\tau^{\top}\Phi_\tau + \lambda I)^{-1}\Phi_\tau^{\top}N_\tau )} \nonumber \\
    &\quad + \lambda \sqrt{\Tr(X(\Phi_\tau^{\top}\Phi_\tau + \lambda I)^{-1}X^{\top})\Tr(\modelearn(\Phi_\tau^{\top}\Phi_\tau + \lambda I)^{-1}\modelearn^{\top})} \nonumber \\
    &= \sqrt{\Tr(X(\Phi_\tau^{\top}\Phi_\tau + \lambda I)^{-1}X^{\top})} \enskip \times \nonumber \\
    &\bigg[\!\sqrt{\Tr(E_\tau^{\top}\Phi_\tau(\Phi_\tau^{\top}\Phi_\tau \!+\! \lambda I)^{-1}\Phi_\tau^{\top}E_\tau )} \!+\!\! \sqrt{\Tr(N_\tau^{\top}\Phi_\tau(\Phi_\tau^{\top}\Phi_\tau \!+\! \lambda I)^{-1}\Phi_\tau^{\top}N_\tau )} \!+\! \lambda \sqrt{\Tr(\modelearn(\Phi_\tau^{\top}\Phi_\tau \!+\! \lambda I)^{-1}\modelearn^{\top})} \bigg] \nonumber
\end{align}
where (\ref{CS_trace}) follows from $|\Tr(ABC^{\top})| \leq \sqrt{\Tr(ABA^{\top})\Tr(CBC^{\top}) }$ for positive definite B due to Cauchy Schwarz (weighted inner-product).  For $X = (\wh{\mathcal{G}}_{\mathbf{y}}-\modelearn)(\Phi_\tau^{\top}\Phi_\tau + \lambda I)$, we get
\begin{align}
  \sqrt{\Tr((\wh{\mathcal{G}}_{\mathbf{y}}-\modelearn)V_\tau(\wh{\mathcal{G}}_{\mathbf{y}}-\modelearn)^{\top})} &\leq  \sqrt{\Tr(E_\tau^{\top}\Phi_\tau V_\tau^{-1}\Phi_\tau^{\top}E_\tau )} + \sqrt{\Tr(N_\tau^{\top}\Phi_\tau V_\tau^{-1}\Phi_\tau^{\top}N_\tau )} + \sqrt{\lambda} \|\modelearn\|_F \label{estimationterms}
\end{align}
where $V_\tau$ is the regularized design matrix at time $\tau$. Let $\max_{i\leq \tau }\|\phi_i \| \leq \Upsilon \sqrt{\Hes} $ and $\max_{\Hes \leq i\leq \tau} \|x_i \| \leq \mathcal{X}$, \textit{i.e.,} in data collection bounded inputs are used. The first term on the right hand side of (\ref{estimationterms}) can be bounded using Theorem \ref{selfnormalized} since $e_t$ is $\| C \sig  C^\top + \sigma_z^2 I \|$-sub-Gaussian vector. Therefore, for $\delta \in (0,1)$, with probability at least $1-\delta$, 
\begin{equation} \label{esti_first}
    \sqrt{\Tr(E_\tau^{\top}\Phi_t V_\tau^{-1}\Phi_\tau^{\top}E_\tau ) } \leq \sqrt{m\| C \sig  C^\top \!\!\!+\! \sigma_z^2 I \| \log \! \left(\!\frac{\operatorname{det}\left(V_\tau\right)^{1 / 2}}{\delta \operatorname{det}(\lambda I)^{1 / 2} }\!\!\right)}
\end{equation}
For the second term,
\begin{align*}
    \sqrt{\Tr(N_\tau^{\top}\Phi_\tau V_\tau^{-1}\Phi_\tau^{\top}N_\tau )} \leq \frac{1}{\sqrt{\lambda}} \|N_\tau^\top \Phi_\tau \|_F &\leq \sqrt{\frac{m}{\lambda}} \left \|\sum_{i=\Hes}^\tau \phi_i (C \Bar{A}^{\Hes}\!x_{i-\Hes})^\top  \right\| \\
    &\leq \tau \sqrt{\frac{m}{\lambda}}  \max_{i\leq \tau} \left\|\phi_i (C \Bar{A}^{\Hes}\!x_{i-\Hes})^\top  \right\| \\
    &\leq \tau \sqrt{\frac{m}{\lambda}} \|C\| \upsilon^{\Hes} \max_{i\leq \tau} \|\phi_i \| \|x_{i-\Hes}\| \\
    &\leq \tau \sqrt{\frac{m}{\lambda}} \|C\| \upsilon^{\Hes} \Upsilon \sqrt{\Hes} \mathcal{X}.
\end{align*}
Picking $\Hes = \frac{2\log(T) + \log(\Upsilon \mathcal{X}) + 0.5\log (m/ \lambda ) + \log(\|C\|)}{\log(1/\upsilon)}$ gives 
\begin{align} \label{esti_second}
    \sqrt{\Tr(N_\tau^{\top}\Phi_\tau V_\tau^{-1}\Phi_\tau^{\top}N_\tau )} \leq \frac{\tau}{T^2} \sqrt{\Hes}.
\end{align}
Combining (\ref{esti_first}) and (\ref{esti_second}) gives the self-normalized estimation error bound state in the theorem.
\end{proof}
\subsection{Frobenius Norm Bound on Finite Sample Estimation Error of (\ref{new_lse}) } \label{apx:2normbound}
Using this result, we have 
\begin{align*}
    \sigma_{\min}(V_\tau)\| \wh{\mathcal{G}}_{\mathbf{y}}-\modelearn \|^2_F &\leq \Tr((\wh{\mathcal{G}}_{\mathbf{y}} - \modelearn)V_t(\wh{\mathcal{G}}_{\mathbf{y}}-\modelearn)^{\top}) \\
    &\leq \left(\sqrt{m\| C \sig  C^\top + \sigma_z^2 I \| \log  \left(\frac{\operatorname{det}\left(V_\tau \right)^{1 / 2}}{\delta \operatorname{det}(\lambda I)^{1 / 2} }\right)} + S\sqrt{\lambda} +\frac{\tau \sqrt{\Hes}}{T^2} \right)^2 
\end{align*}
For persistently exciting inputs, \textit{i.e.,} $\sigma_{\min}(V_\tau) \geq \sigma_\star^2 \tau $ for $\sigma_\star > 0$, using the boundedness of $\phi_i$, \textit{i.e.,} $\max_{i\leq \tau }\|\phi_i \| \leq \Upsilon \sqrt{\Hes} $, we get, 
\begin{align*}
    \| \wh{\mathcal{G}}_{\mathbf{y}}-\modelearn \|_F \leq \frac{\sqrt{m\| C \sig  C^\top + \sigma_z^2 I \|\left( \log(1/\delta) + \frac{\Hes(m+p)}{2} \log  \left(\frac{\lambda(m+p) + \tau \Upsilon^2}{\lambda(m+p) }\right)\right)} + S\sqrt{\lambda} +\frac{\sqrt{\Hes}}{T}}{\sigma_\star \sqrt{\tau}}
\end{align*}
This result shows that under persistent of excitation, the novel least squares problem provides consistent estimates and the estimation error is $\Tilde{\OO}(1/\sqrt{T})$ after $T$ samples.
\subsection{Bound on Model Parameters Estimation Errors} \label{apx:modelestbound}
After estimating $\wh{\mathcal{G}}_{\mathbf{y}}$, we deploy \Sys (Appendix \ref{apx:Identification}) to recover the unknown system parameters. The outline of the algorithm is given in Algorithm \ref{SYSID}. Note that the system is order $n$ and
minimal in the sense that the system cannot be described by a state-space model of order less than $n$. Define $T_{\modelearn}$ such that at $T_{\modelearn}$, $\|\wh{\mathcal{G}}_{\mathbf{y}} - \modelearn\| \leq 1$. Let $T_N = T_{\modelearn} \frac{8\Hes}{\sigma_n^2(\mathcal{H})}$, $T_B = T_{\modelearn} \frac{20n\Hes}{\sigma_n(\mathcal{H})}$. We have the following result on the model parameter estimates:

\begin{theorem}[Model Parameters Estimation Error] \label{ConfidenceSets}
Let $\mathcal{H}$ be the concatenation of two Hankel matrices obtained from $\modelearn$. Let $\bar{A}, \bar{B}, \bar{C}, \bar{F}$ be the system parameters that \Sys provides for $\modelearn$. At time step $t$, let $\hat{A}_t, \hat{B}_t, \hat{C}_t, \hat{F}_t$ denote the system parameters obtained by \Sys using $\wh{\mathcal{G}}_{\mathbf{y}}$. For the described system in the main text $\mathcal{H}$ is rank-$n$ \textit{i.e.,} due to controllability-observability assumption. For $t \geq \max \{ T_{\modelearn}, T_N, T_B \}$, for the given choice of $\Hes$, there exists a unitary matrix $\mathbf{T} \in \R^{n \times n}$ such that, $\bar{\Theta}=(\bar{A}, \bar{B}, \bar{C}, \bar{F}) \in (\mathcal{C}_A \times \mathcal{C}_B \times \mathcal{C}_C \times \mathcal{C}_F) $ where
\begin{align}
    &\mathcal{C}_A(t) = \left \{A' \in \R^{n \times n} : \|\hat{A}_t - \mathbf{T}^\top A' \mathbf{T} \| \leq \beta_A(t) \right\}, \nonumber \\
    &\mathcal{C}_B(t) = \left \{B' \in \R^{n \times p} : \|\hat{B}_t - \mathbf{T}^\top B' \| \leq  \beta_B(t) \right\}, \nonumber \\
    &\mathcal{C}_C(t) = \left\{C' \in \R^{m \times n} : \|\hat{C}_t -  C'  \mathbf{T} \| \leq \beta_C(t) \right\}, \nonumber \\
    &\mathcal{C}_L(t) = \left \{L' \in \R^{p \times m} : \|\hat{F}_t - \mathbf{T}^\top F' \| \leq  \beta_F(t) \right\}, \nonumber
\end{align}
for
\begin{align}
&\beta_A(t) = c_1\left( \frac{\sqrt{n\Hes}(\|\mathcal{H}\| + \sigma_n(\mathcal{H}))}{\sigma_n^2(\mathcal{H})} \right)\|\wh{\mathcal{G}}_{\mathbf{y}} - \modelearn \|, \quad \beta_B(t) = \beta_C(t) = \beta_F(t)  =  \sqrt{\frac{20n\Hes }{\sigma_n(\mathcal{H})}}\|\wh{\mathcal{G}}_{\mathbf{y}} - \modelearn\|, \nonumber
\end{align}
for some problem dependent constant $c_1$.
\end{theorem}
\begin{proof}
The result follows similar steps with \citet{oymak2018non}. The following lemma is from \citet{oymak2018non}, it will be useful in proving error bounds on system parameters and we provide it for completeness. 
\begin{lemma}[\citep{oymak2018non}] \label{hokalmanstability lemma}
$\mathcal{H}$, $\hat{\mathcal{H}}_t$ and $\mathcal{N}, \hat{\mathcal{N}}_t$ satisfies the following perturbation bounds,

\begin{align*} 
\max \left\{\left\|\mathcal{H}^+-\hat{\mathcal{H}}_t^+\right\|,\left\|\mathcal{H}^- -\hat{\mathcal{H}}_t^-\right\|\right\} \leq \|\mathcal{H}-\hat{\mathcal{H}}_t\| &\leq \sqrt{\min \left\{d_{1}, d_{2}+1\right\}}\|\wh{\mathcal{G}}_{\mathbf{y}} - \modelearn\| \\ \|\mathcal{N}-\hat{\mathcal{N}}_t\| \leq 2\left\|\mathcal{H}^- -\hat{\mathcal{H}}_t^-\right\| &\leq 2 \sqrt{\min \left\{d_{1}, d_{2}\right\}}\|\wh{\mathcal{G}}_{\mathbf{y}} - \modelearn\|
\end{align*}
\end{lemma}

The following lemma is a slight modification of Lemma B.1 in \citep{oymak2018non}.

\begin{lemma}[\citep{oymak2018non}]\label{lem:ranknperturb} Suppose $\sigma_{\min}(\mathcal{N}) \geq 2\|\mathcal{N}-\hat{\mathcal{N}}\|$ where $\sigma_{\min }(\mathcal{N})$ is the smallest nonzero singular value (i.e. $n$th largest singular value) of $N$. Let rank n matrices $\mathcal{N}, \hat{\mathcal{N}}$ have singular value decompositions $\mathbf{U} \mathbf{\Sigma} \mathbf{V}^{\top}$ and $\mathbf{\hat{U}} \mathbf{\hat{\Sigma}} \mathbf{\hat{V}}^{\top}$
There exists an $n \times n$ unitary matrix $\mathbf{T}$ so that
\begin{equation*}
\left\|\mathbf{U} \mathbf{\Sigma}^{1/2}-\mathbf{\hat{U}} \mathbf{\hat{\Sigma}}^{1/2} \mathbf{T} \right\|_{F}^{2}+\left\|\mathbf{V} \mathbf{\Sigma}^{1/2}-\mathbf{\hat{V}} \mathbf{\hat{\Sigma}}^{1/2} \mathbf{T} \right\|_{F}^{2} \leq \frac{5n \| \mathcal{N} - \hat{\mathcal{N}}\|^2}{\sigma_n(\mathcal{N}) - \| \mathcal{N} - \hat{\mathcal{N}}\| }
\end{equation*}

\end{lemma}

For brevity, we have the following notation $\mathbf{O} = \mathbf{O}(\bar{A},C,d_1)$,  $\mathbf{C_F} = \mathbf{C}(\bar{A},F,d_2+1)$, $\mathbf{C_B} = \mathbf{C}(\bar{A},B,d_2+1)$, 
$\mathbf{\hat{O}_t} = \mathbf{\hat{O}_t}(\bar{A},C,d_1)$, $\mathbf{\hat{C}_{F_t}} = \mathbf{\hat{C}_t}(\bar{A},F,d_2+1)$,
$\mathbf{\hat{C}_{B_t}} = \mathbf{\hat{C}_t}(\bar{A},B,d_2+1)$. In the definition of $T_N$, we use $\sigma_n(H)$, due to the fact that singular values of submatrices by column partitioning are interlaced, \textit{i.e.} $\sigma_n(\mathbf{N}) = \sigma_n(\mathbf{H}^-) \geq \sigma_n(\mathbf{H})$. 
Directly applying Lemma \ref{lem:ranknperturb} with the condition that for given $t \geq T_N$, $\sigma_{\min}(\mathcal{N}) \geq 2\|\mathcal{N}-\hat{\mathcal{N}}\|$, we can guarantee that there exists a unitary transform $\mathbf{T}$ such that 
\begin{equation} \label{boundonSVD}
    \left\|\mathbf{\hat{O}_t} - \mathbf{O}\mathbf{T}  \right\|_F^2 + \left\|[\mathbf{\hat{C}_{F_t}} \enskip \mathbf{\hat{C}_{B_t}}] -  \mathbf{T}^\top  [\mathbf{C_F} \enskip \mathbf{C_B}] \right\|_F^2 \leq \frac{10n \| \mathcal{N} - \hat{\mathcal{N}}_t\|^2}{\sigma_n(\mathcal{N}) }.
\end{equation}
Since $\hat{C}_t - \bar{C}\mathbf{T}$ is a submatrix of $\mathbf{\hat{O}_t} - \mathbf{O}\mathbf{T}$, $\hat{B}_t - \mathbf{T}^\top \bar{B}$ is a submatrix of $\mathbf{\hat{C}_{B_t}} - \mathbf{T}^\top \mathbf{C_{B}}$ and $\hat{F}_t - \mathbf{T}^\top \bar{F}$ is a submatrix of $\mathbf{\hat{C}_{F_t}} - \mathbf{T}^\top \mathbf{C_{F}}$, we get the same bounds for them stated in (\ref{boundonSVD}). Using Lemma \ref{hokalmanstability lemma}, with the choice of $d_1,d_2 \geq \frac{\Hes}{2}$, we have 
\begin{equation*}
    \| \mathcal{N} - \hat{\mathcal{N}}_t\| \leq \sqrt{2\Hes}\|\wh{\mathcal{G}}_{\mathbf{y}} - \modelearn\|.
\end{equation*}
This provides the advertised bounds in the theorem:
\begin{align*}
    \|\hat{B}_t - \mathbf{T}^\top \bar{B} \|, \|\hat{C}_t - \bar{C}\mathbf{T}\|, \|\hat{F}_t - \mathbf{T}^\top \bar{F}\|  \leq \frac{\sqrt{20n\Hes} \|\wh{\mathcal{G}}_{\mathbf{y}} - \modelearn\| }{\sqrt{\sigma_n(\mathcal{N})}}
\end{align*}

Notice that for $t\geq T_B$, we have all the terms above to be bounded by 1. In order to determine the closeness of $\hat{A}_t$ and $\Bar{A}$ we first consider the closeness of $\hat{\Bar{A}}_t - \mathbf{T}^\top \Bar{\Bar{A}}\mathbf{T}$, where $\Bar{\Bar{A}}$ is the output obtained by \Sys for $\Bar{A}$ when the input is $\modelearn$. Let $X = \mathbf{O}\mathbf{T}$ and $Y = \mathbf{T}^\top  [\mathbf{C_F} \enskip \mathbf{C_B}]$. Thus, we have
\begin{align*}
    \|\hat{\Bar{A}}_t - \mathbf{T}^\top \Bar{\Bar{A}}\mathbf{T} \|_F &= \| \mathbf{\hat{O}_t}^\dagger \hat{\mathcal{H}}_t^+ [\mathbf{\hat{C}_{F_t}} \enskip \mathbf{\hat{C}_{B_t}}]^\dagger - X^\dagger \mathcal{H}^+ Y^\dagger \|_F \\
    &\leq \left \| \left( \mathbf{\hat{O}_t}^\dagger - X^\dagger \right) \hat{\mathcal{H}}_t^+ [\mathbf{\hat{C}_{F_t}} \enskip \mathbf{\hat{C}_{B_t}}]^\dagger \right \|_F + \left \|X^\dagger \left(\hat{\mathcal{H}}_t^+ - \mathcal{H}^+ \right)[\mathbf{\hat{C}_{F_t}} \enskip \mathbf{\hat{C}_{B_t}}]^\dagger \right \|_F\\
    &\quad + \left \|X^\dagger \mathcal{H}^+ \left([\mathbf{\hat{C}_{F_t}} \enskip \mathbf{\hat{C}_{B_t}}]^\dagger - Y^\dagger \right) \right \|_F
\end{align*}
For the first term we have the following perturbation bound \citep{meng2010optimal,wedin1973perturbation},
\begin{align*}
    \|\mathbf{\hat{O}_t}^\dagger - X^\dagger \|_F &\leq \| \mathbf{\hat{O}_t} - X \|_F \max \{ \|X^\dagger \|^2, \| \mathbf{\hat{O}_t}^\dagger \|^2  \} \leq \| \mathcal{N} - \hat{\mathcal{N}}_t\| \sqrt{\frac{10n}{\sigma_n(\mathcal{N}) }} \max \{ \|X^\dagger \|^2, \| \mathbf{\hat{O}_t}^\dagger \|^2  \}
\end{align*}
Since we have $\sigma_{n}(\mathcal{N}) \geq 2\|\mathcal{N}-\hat{\mathcal{N}}\|$, we have $\|\hat{\mathcal{N}}\| \leq 2 \|\mathcal{N}\| $ and $2\sigma_{n}(\hat{\mathcal{N}}) \geq \sigma_{n}(\mathcal{N})$. Thus,
\begin{equation} \label{daggersingular}
    \max \{ \|X^\dagger \|^2, \| \mathbf{\hat{O}_t}^\dagger \|^2  \} = \max \left\{ \frac{1}{\sigma_{n}(\mathcal{N})}, \enskip \frac{1}{\sigma_{n}(\hat{\mathcal{N}})}  \right\} \leq \frac{2}{\sigma_{n}(\mathcal{N})}
\end{equation}
Combining these and following the same steps for $\|[\mathbf{\hat{C}_{F_t}} ~ \mathbf{\hat{C}_{B_t}}]^\dagger \!-\! Y^\dagger \|_F$, we get 
\begin{equation} \label{perturbationbounds}
    \left \|\mathbf{\hat{O}_t}^\dagger - X^\dagger \right \|_F,\enskip \left \|[\mathbf{\hat{C}_{F_t}} ~ \mathbf{\hat{C}_{B_t}}]^\dagger \!-\! Y^\dagger \right \|_F\leq \left \| \mathcal{N} - \hat{\mathcal{N}}_t \right \| \sqrt{\frac{40n}{\sigma_n^{3}(\mathcal{N}) }} 
\end{equation}
The following individual bounds obtained by using (\ref{daggersingular}), (\ref{perturbationbounds}) and triangle inequality: 
\begin{align*}
    \left \| \left( \mathbf{\hat{O}_t}^\dagger - X^\dagger \right) \hat{\mathcal{H}}_t^+ [\mathbf{\hat{C}_{F_t}} \enskip \mathbf{\hat{C}_{B_t}}]^\dagger \right \|_F &\leq \| \mathbf{\hat{O}_t}^\dagger - X^\dagger \|_F \|\hat{\mathcal{H}}_t^+ \| \|[\mathbf{\hat{C}_{F_t}} \enskip \mathbf{\hat{C}_{B_t}}]^\dagger \| \\
    &\leq \frac{4\sqrt{5n} \left \| \mathcal{N} - \hat{\mathcal{N}}_t \right \|}{\sigma_n^{2}(\mathcal{N})} \left( \|\mathcal{H}^+ \| + \|\hat{\mathcal{H}}_t^+ - \mathcal{H}^+ \| \right)  \\
    \left \|X^\dagger \left(\hat{\mathcal{H}}_t^+ - \mathcal{H}^+ \right)[\mathbf{\hat{C}_{F_t}} \enskip \mathbf{\hat{C}_{B_t}}]^\dagger \right \|_F &\leq \frac{2\sqrt{n}\|\hat{\mathcal{H}}_t^+ - \mathcal{H}^+ \|}{\sigma_n(\mathcal{N})}
    \\
    \left \|X^\dagger \mathcal{H}^+ \left([\mathbf{\hat{C}_{F_t}} \enskip \mathbf{\hat{C}_{B_t}}]^\dagger - Y^\dagger \right) \right \|_F &\leq \|X^\dagger \| \|\mathcal{H}^+ \| \|[\mathbf{\hat{C}_{F_t}} \enskip \mathbf{\hat{C}_{B_t}}]^\dagger - Y^\dagger \| \\
    &\leq \frac{2\sqrt{10n} \left \| \mathcal{N} - \hat{\mathcal{N}}_t \right \|}{\sigma_n^{2}(\mathcal{N})}  \|\mathcal{H}^+ \|
\end{align*}
Combining these we get 
\begin{align*}
    \|\hat{\Bar{A}}_t - \mathbf{T}^\top \Bar{\Bar{A}}\mathbf{T} \|_F &\leq \frac{31\sqrt{n}  \|\mathcal{H}^+ \|  \left \| \mathcal{N} - \hat{\mathcal{N}}_t \right \|}{2\sigma_n^{2}(\mathcal{N})}  + \|\hat{\mathcal{H}}_t^+ - \mathcal{H}^+ \| \left( \frac{4\sqrt{5n} \left \| \mathcal{N} - \hat{\mathcal{N}}_t \right \|}{\sigma_n^{2}(\mathcal{N})} + \frac{2\sqrt{n}}{\sigma_n(\mathcal{N})}  \right) \\
    &\leq \frac{31\sqrt{n}  \|\mathcal{H}^+ \|  }{2\sigma_n^{2}(\mathcal{N})} \left \| \mathcal{N} - \hat{\mathcal{N}}_t \right \| +  \frac{13 \sqrt{n}}{2\sigma_n(\mathcal{N})} \|\hat{\mathcal{H}}_t^+ - \mathcal{H}^+ \|
\end{align*}
Now consider $\hat{A}_t = \hat{\Bar{A}}_t + \hat{F}_t\hat{C}_t$. Using Lemma \ref{hokalmanstability lemma},
\begin{align*}
    &\|\hat{A}_t -  \mathbf{T}^\top \Bar{A}\mathbf{T} \|_F \\
    &= \|\hat{\Bar{A}}_t + \hat{F}_t\hat{C}_t - \mathbf{T}^\top \Bar{\Bar{A}}\mathbf{T} - \mathbf{T}^\top \bar{F} \bar{C} \mathbf{T} \|_F \\
    &\leq \|\hat{\Bar{A}}_t - \mathbf{T}^\top \Bar{\Bar{A}}\mathbf{T} \|_F + \|(\hat{F}_t -\mathbf{T}^\top \bar{F}) \hat{C}_t  \|_F + \|\mathbf{T}^\top \bar{F} (\hat{C}_t - \bar{C} \mathbf{T}) \|_F \\
    &\leq \|\hat{\Bar{A}}_t - \mathbf{T}^\top \Bar{\Bar{A}}\mathbf{T} \|_F + \|(\hat{F}_t -\mathbf{T}^\top \bar{F})\|_F \| \hat{C}_t -\bar{C} \mathbf{T} \|_F + \|(\hat{F}_t -\mathbf{T}^\top \bar{F})\|_F \| \bar{C} \| + \|\bar{F}\| \|(\hat{C}_t - \bar{C} \mathbf{T}) \|_F \\
    &\leq \frac{31\sqrt{n}  \|\mathcal{H}^+ \|  }{2\sigma_n^{2}(\mathcal{N})} \! \left \| \mathcal{N} \!-\! \hat{\mathcal{N}}_t \right \| \!+\!  \frac{13 \sqrt{n}}{2\sigma_n(\mathcal{N})} \|\hat{\mathcal{H}}_t^+ \!-\! \mathcal{H}^+ \| \!+\! \frac{10n \| \mathcal{N} \!-\! \hat{\mathcal{N}}_t\|^2}{\sigma_n(\mathcal{N}) } \!+\! (\|\bar{F} \| \!+\! \|\bar{C} \|)\| \mathcal{N} \!-\! \hat{\mathcal{N}}_t\| \sqrt{\frac{10n }{\sigma_n(\mathcal{N}) }  } \\
    &\leq \frac{31\sqrt{2n\Hes}  \|\mathcal{H} \|  }{2\sigma_n^{2}(\mathcal{N})} \! \|\wh{\mathcal{G}}_{\mathbf{y}} - \modelearn\| + \frac{13 \sqrt{n\Hes}}{2\sqrt{2}\sigma_n(\mathcal{N})} \|\wh{\mathcal{G}}_{\mathbf{y}} - \modelearn\| + \frac{20n\Hes \|\wh{\mathcal{G}}_{\mathbf{y}} - \modelearn\|^2}{\sigma_n(\mathcal{N}) } \\
    &\qquad + (\|\bar{F} \| \!+\! \|\bar{C} \|)\|\wh{\mathcal{G}}_{\mathbf{y}} - \modelearn\| \sqrt{\frac{20n\Hes }{\sigma_n(\mathcal{N}) }  }
\end{align*}
\end{proof}

\subsection{Bound on the Markov Parameters Estimation Errors} \label{apx:markovestimate}
Finally, we will consider the Markov parameter estimates that is constructed by using the parameter estimates. From Theorem \ref{ConfidenceSets}, for some unitary matrix $\mathbf{T}$, we denote $\Delta A \coloneqq \| \wh A_t - \mathbf{T}^\top A \mathbf{T} \| $, $\Delta B \coloneqq \| \wh B_t - \mathbf{T}^\top B \| = \| \wh C_t -  C \mathbf{T} \| $. Let $T_A = T_{\modelearn}\frac{4c_1^2 \left( \frac{\sqrt{n\Hes}(\|\mathcal{H}\| + \sigma_n(\mathcal{H}))}{\sigma_n^2(\mathcal{H})} \right)^2 }{(1-\rho(A))^2}$. For $t>\max \{T_A,T_B\}$, $\Delta A \leq \frac{1-\rho(A)}{2}$ and $\Delta B \leq 1$. Using this fact, we have
\begin{align*}
    &\sum_{j\geq 1}^H \|\wh C_t \wh A_t^{j-1} \wh B_t - CA^{j-1}B\|\\
    &\leq \Delta B (\|B \| \!+\! \| C \| \!+\! 1) \!+\! \sum_{i=1}^{H-1} \!\!\Phi(A) \rho^i(A) \Delta B (\|B \| \!+\! \| C \| \!+\! 1) + \|\wh A_t^i \!-\! \mathbf{T}^\top\! A^i \mathbf{T} \| (\|C\| \|B\| \!+\! \|B\| \!+\! \|C\| \!+\! 1 ) \\
    &\leq \left(1+ \frac{\Phi(A)}{1-\rho(A)}\right) \Delta B (\|B \| \!+\! \| C \| \!+\! 1) + \Delta A (\|C\| \|B\| \!+\! \|B\| \!+\! \|C\| \!+\! 1 ) \sum_{i=1}^{H-1} \sum_{j=0}^{i-1} \binom{i}{j} \| A^j\| (\Delta A)^{i-1-j} \\
    &\leq \left(1+ \frac{\Phi(A)}{1-\rho(A)}\right) \Delta B (\|B \| \!+\! \| C \| \!+\! 1) \\
    &\qquad + \Delta A \Phi(A) (\|C\| \|B\| \!+\! \|B\| \!+\! \|C\| \!+\! 1 )  \sum_{i=1}^{H-1} \sum_{j=0}^{i-1} \binom{i}{j} \rho^j(A) \left(\frac{1-\rho(A)}{2}\right)^{i-1-j} \\
    &\leq \left(1+ \frac{\Phi(A)}{1-\rho(A)}\right) \Delta B (\|B \| \!+\! \| C \| \!+\! 1) + \frac{2\Delta A \Phi(A)}{1-\rho(A)} (\|C\| \|B\| \!+\! \|B\| \!+\! \|C\| \!+\! 1 )  \sum_{i=1}^{H-1} \left[ \left(\frac{1+\rho}{2}\right)^i - \rho^i \right]
    \\
    &\leq \Delta B \left(1+ \frac{\Phi(A)}{1-\rho(A)}\right) (\|B \| \!+\! \| C \| \!+\! 1) + \frac{2\Delta A \Phi(A)}{(1-\rho(A))^2} (\|C\| \|B\| \!+\! \|B\| \!+\! \|C\| \!+\! 1 )
\end{align*}

\begin{equation*}
\gamma_{\mathbf{G}}=(\|B\|+\|C\|+1)\left(1+\frac{\Phi(A)}{1-\rho(A)}+\frac{2 \Phi(A)}{(1-\rho(A))^{2}}\right)+\frac{2 \Phi(A)}{(1-\rho(A))^{2}}\|C\|\|B\|
\end{equation*}
Assuming that $\|F\| + \| C\| > 1 $ for simplicity, from the exact expressions of Theorem \ref{ConfidenceSets}, we have $\Delta A > \Delta B$. For the given $\gamma_\Markov$ and $\gamma_\mathcal{H}$, we can upper bound the last expression above as follow, 
\begin{align}
    \sum_{j\geq 1}^H \|\wh C_t \wh A_t^{j-1} \wh B_t - CA^{j-1}B\| \leq \gamma_\Markov  \Delta A \leq \frac{c_1 \gamma_\Markov \gamma_\mathcal{H} \kappa_{e}}{\sigma_\star \sqrt{t}}, \label{LemmaWarmLast}
\end{align}
for  
\begin{align}
    \gamma_{\mathbf{G}}&\coloneqq(\|B\|+\|C\|+1)\left(1+\frac{\Phi(A)}{1-\rho(A)}+\frac{2 \Phi(A)}{(1-\rho(A))^{2}}\right)+\frac{2 \Phi(A)}{(1-\rho(A))^{2}}\|C\|\|B\|, \label{gamma_G}\\
    \kappa_e &\coloneqq \sqrt{m\| C \sig  C^\top + \sigma_z^2 I \|\left( \log(1/\delta) + \frac{\Hes(m+p)}{2} \log  \left(\frac{\lambda(m+p) + T \Upsilon^2}{\lambda(m+p) }\right)\right)} + S\sqrt{\lambda} +\frac{\sqrt{\Hes}}{T}, \\
    \gamma_{\mathcal{H}} &\coloneqq  \frac{\sqrt{n\Hes}(\|\mathcal{H}\| + \sigma_n(\mathcal{H}))}{\sigma_n^2(\mathcal{H})}. \label{gamma_H}
\end{align}
The proof of Theorem \ref{theo:sysid} is completed by noticing that $\|\wh{\Markov}(H) - \Markov(H) \| = \| [\wh G^{[1]}~\wh G^{[2]}~\ldots~\wh G^{[H]}] - [ G^{[1]}~ G^{[2]}~\ldots~ G^{[H]}] \leq \sqrt{\sum_{i=1}^H \|\wh G^{[i]} - G^{[i]} \|^2 }$.

\section{Persistence of Excitation}
\label{apx:persistence}

In this section, we provide the persistence of excitation of \alg inputs that is required for consistent estimation of system parameters as pointed out in Theorem \ref{theo:sysid}. We will first consider open-loop persistence excitation and introduce truncated open-loop noise evolution parameter $\Gol$, Appendix \ref{subsec:warmuppersist}. It represents the effect of noises in the system on the outputs. We will define $\Gol$ for $2\Hes$ time steps back in time and show that last $2\Hes$ process and measurement noises provide sufficient persistent excitation for the covariates in the estimation problem. Let $\sigma_o < \sigma_{\min}(\Gol)$. We will show that there exists a positive $\sigma_o$, \textit{i.e.,} $\Gol$ is full row rank. In the following, $\bar{\phi}_t = P \phi_t$ for a permutation matrix $P$ that gives 
\begin{equation*}
    \bar{\phi}_t = \left[ y_{t-1}^\top \enskip u_{t-1}^\top \ldots y_{t-H}^\top \enskip  u_{t-H}^\top \right]^\top \in \mathbb{R}^{(m+p)H}.
\end{equation*}

We assume that, throughout the interaction with the system, the agent has access to a convex compact set of \DFC{}s, $\Mcontrolset$ which is an $r$-expansion of $\Mcontrolset_\psi$, such that $\kappa_\Mcontrolset = \kappa_\psi + r$ and all controllers $\Mcontrol \in \Mcontrolset $ are persistently exciting the system $\Theta$. In Appendix \ref{subsec:persistence}, we formally define the persistence of excitation condition for the given set $\Mcontrolset$. Finally, in Appendix \ref{subsec:theo:persist_adaptive}, we show that persistence of excitation is achieved by the policies that \alg deploys. 
\subsection{Persistence of Excitation in Warm-up} \label{subsec:warmuppersist}

Recall the state-space form of the system,
\begin{align}
    x_{t+1} &= A x_t + B u_t + w_t \nonumber \\
    y_t &= C x_t + z_t. \label{system_apx}
\end{align}
During the warm-up period, $t \leq \Tburn$, the input to the system is $u_t \sim \mathcal{N}(0,\sigma_u^2 I)$. Let $f_t = [y_t^\top u_t^\top]^\top$. From the evolution of the system with given input we have the following:
\begin{equation*}
    f_t = \mathbf{G^o} \begin{bmatrix}
    w_{t-1}^\top & z_t^\top & u_t^\top & \ldots & w_{t-H}^\top & z_{t-H+1}^\top & u_{t-H+1}^\top
\end{bmatrix}^\top + \mathbf{r_t^o}
\end{equation*}
where 
\begin{align}
\!\!\!\mathbf{G^o}\!\! := \!\!
\begin{bmatrix}
0_{m\!\times\! n}~I_{m\!\times\! m}~0_{m\!\times\! p}  & C~0_{m\!\times\! m}~CB & CA~0_{m\!\times\! m}~CAB &\ldots & \quad CA^{\Hes-2}~0_{m\!\times\! m}~CA^{\Hes-2}B \\
0_{p\!\times\! n}~0_{p\!\times\! m}~I_{p\!\times\! p}  &0_{p\!\times\! n}~0_{p\!\times\! m}~0_{p\!\times\! p}  & 0_{p\!\times\! n}~0_{p\!\times\! m}~0_{p\!\times\! p} &\ldots& 0_{p\!\times\! n}~0_{p\!\times\! m}~0_{p\!\times\! p} 
\end{bmatrix}
\end{align}
and $\mathbf{r_t^o}$ is the residual vector that represents the effect of $[w_{i-1} \enskip z_i \enskip u_i ]$ for $0 \leq i<t-\Hes$, which are independent. Notice that $\mathbf{G^o}$ is full row rank even for $\Hes=1$, due to first $(m+p)\times (m+n+p)$ block. Using this, we can represent $\bar{\phi}_t$ as follows
\begin{align}
    \bar{\phi}_t &= \underbrace{\begin{bmatrix}
    f_{t-1} \\
    \vdots \\
    f_{t-\Hes}
\end{bmatrix}}_{\mathbb{R}^{(m+p)\Hes}} + 
\begin{bmatrix}
    \mathbf{r_{t-1}^o} \\
    \vdots \\
    \mathbf{r_{t-\Hes}^o}
\end{bmatrix}
= \Gol
\underbrace{\begin{bmatrix}
    w_{t-2} \\
    z_{t-1} \\
    u_{t-1}\\
    \vdots \\
    w_{t-2\Hes-1}\\
    z_{t-2\Hes} \\
    u_{t-2\Hes}
\end{bmatrix}}_{\mathbb{R}^{2(n+m+p)\Hes}} + 
\begin{bmatrix}
    \mathbf{r_{t-1}^o} \\
    \vdots \\
    \mathbf{r_{t-\Hes}^o}
\end{bmatrix} \quad \text{ where } \nonumber \\
\Gol &\coloneqq \!\!
\begin{bmatrix}
    [\qquad \qquad \mathbf{G^o} \qquad \qquad] \quad 0_{(m+p)\times (m+n+p)} \enskip 0_{(m+p)\times (m+n+p)} \enskip 0_{(m+p)\times (m+n+p)} \enskip \ldots \\
    0_{(m+p)\times (m+n+p)} \enskip [\qquad \qquad \mathbf{G^o} \qquad \qquad] \qquad  0_{(m+p)\times (m+n+p)}  \enskip 0_{(m+p)\times (m+n+p)} \enskip \ldots \\
    \ddots \\
     0_{(m+p)\times (m+n+p)}  \enskip 0_{(m+p)\times (m+n+p)} \enskip \ldots \quad [\qquad \qquad \mathbf{G^o} \qquad \qquad] \enskip 0_{(m+p)\times (m+n+p)} \\
    0_{(m+p)\times (m+n+p)} \enskip 0_{(m+p)\times (m+n+p)} \enskip 0_{(m+p)\times (m+n+p)} \enskip \ldots \qquad [\qquad \qquad \mathbf{G^o} \qquad \qquad]
\end{bmatrix}.
\label{Gol}
\end{align}
During warm-up period, from Lemma D.1 of \citet{lale2020regret}, we have that for all $1\leq t \leq \Tburn$, with probability $1-\delta/2$,
\begin{align}
    \label{exploration norms first}
    \| x_t \| &\leq X_{w} \coloneqq \frac{(\sigma_w + \sigma_u \|B\|)  \Phi(A) \rho(A)}{\sqrt{1-\rho(A)^2}}\sqrt{2n\log(12n\Tburn/\delta)} , \\
    \| z_t \| &\leq Z \coloneqq  \sigma_z \sqrt{2m\log(12m\Tburn/\delta)} , \\ 
    \| u_t \| &\leq U_{w} \coloneqq \sigma_u \sqrt{2p\log(12p\Tburn/\delta)}, \\
    \| y_t \| &\leq \|C\| X_w + Z
    \label{exploration norms last}.
\end{align}
Thus, during the warm-up phase, we have $\max_{i\leq t\leq \Tburn} \|\phi_i\| \leq  \Upsilon_w \sqrt{\Hes} $, where $\Upsilon_w = \|C\| X_w + Z + U_w $. Define 
\begin{equation*}
    T_o = \frac{32 \Upsilon_w^4 \log^2\left(\frac{2\Hes(m+p)}{\delta}\right)}{\sigma_{\min}^4(\Gol) \min \{ \sigma_w^4, \sigma_z^4, \sigma_u^4 \}}.
\end{equation*}

We have the following lemma that provides the persistence of excitation of inputs in the warm-up period.
\begin{lemma} \label{openloop_persistence_appendix}
If the warm-up duration $\Tburn \geq T_o$, then for $T_o \leq t\leq \Tburn$, with probability at least $1-\delta$ we have
\begin{equation}
\sigma_{\min}\left(\sum_{i=1}^{t} \phi_i \phi_i^\top \right) \geq t \frac{\sigma_{o}^2 \min \{ \sigmaWlow^2, \sigmaZlow^2, \sigma_u^2 \}}{2}.  
\end{equation}
\end{lemma}
\begin{proof}
Let $\Bar{\mathbf{0}} = 0_{(m+p)\times (m+n+p)}$. Since each block row is full row-rank, we get the following decomposition using QR decomposition for each block row: 
\begin{align*}
    \Gol = \underbrace{\begin{bmatrix}
   Q^{o} & 0_{m+p} & 0_{m+p} & 0_{m+p} & \ldots \\
    0_{m+p} & Q^{o} &  0_{m+p} & 0_{m+p} & \ldots \\
    & & \ddots & & \\
     0_{m+p}  & 0_{m+p} & \ldots & Q^{o} & 0_{m+p} \\
    0_{m+p} & 0_{m+p} & 0_{m+p} & \ldots & Q^{o}
\end{bmatrix} }_{\mathbb{R}^{(m+p)H \times (m+p)H }}
\underbrace{\begin{bmatrix}
   R^{o} \!&\! \Bar{\mathbf{0}} \!&\! \Bar{\mathbf{0}} \!&\! \Bar{\mathbf{0}} \!&\! \ldots \\
    \Bar{\mathbf{0}} \!&\! R^{o} \!&\!  \Bar{\mathbf{0}} \!&\! \Bar{\mathbf{0}} \!&\! \ldots \\
    \!&\! \!&\! \ddots \!&\! \!&\! \\
     \Bar{\mathbf{0}}  \!&\! \Bar{\mathbf{0}} \!&\! \ldots \!&\!R^{o} \!&\! \Bar{\mathbf{0}} \\
    \Bar{\mathbf{0}} \!&\! \Bar{\mathbf{0}} \!&\! \Bar{\mathbf{0}} \!&\! \ldots \!&\! R^{o}
\end{bmatrix}}_{\mathbb{R}^{(m+p)H \times 2(m+n+p)H }}
\end{align*}
where $R^{o} = \begin{bmatrix}
       \x & \x & \x & \x & \x & \x &\ldots \\ 0 & \x & \x & \x & \x & \x & \ldots \\ & \ddots & \\ 0 & 0 & 0 & \x & \x & \x & \ldots \end{bmatrix} \in \mathbb{R}^{(m+p) \times H(m+n+p) }$ 
where the elements in the diagonal are positive numbers. Notice that the first matrix with $Q^{0}$ is full rank. Also, all the rows of second matrix are in row echelon form and second matrix is full row-rank. Thus, we can deduce that $\Gol$ is full row-rank. Since $\Gol$ is full row rank, we have that 
\begin{equation*}
    \mathbb{E}[\bar{\phi}_t \bar{\phi}_t^\top] \succeq \Gol \Sigma_{w,z,u} \mathcal{G}^{ol \top}
\end{equation*}
where $\Sigma_{w,z,u} \in \R^{2(n+m+p)H \times 2(n+m+p)H} = \text{diag}(\sigma_w^2, \sigma_z^2, \sigma_u^2, \ldots,\sigma_w^2, \sigma_z^2, \sigma_u^2)$. This gives us 
\[ 
\sigma_{\min}(\mathbb{E}[\bar{\phi}_t \bar{\phi}_t^\top]) \geq \sigma_{\min}^2(\Gol) \min \{ \sigma_w^2, \sigma_z^2, \sigma_u^2 \}
\]
for $t \leq \Tburn$. As given in (\ref{exploration norms first})-(\ref{exploration norms last}), we have that $\|\phi_t\| \leq \Upsilon_w \sqrt{\Hes}$ with probability at least $1-\delta/2$. Given this holds, one can use Theorem \ref{azuma}, to obtain the following which holds with probability $1-\delta/2$:
\begin{align*}
    \lambda_{\max}\left (\sum_{i=1}^{t} \phi_i \phi_i^\top - \mathbb{E}[\phi_i \phi_i^\top]  \right) \leq 2\sqrt{2t} \Upsilon_w^2 \Hes \sqrt{\log\left(\frac{2\Hes(m+p)}{\delta}\right)}. 
\end{align*}
Using Weyl's inequality, during the warm-up period with probability $1-\delta$, we have 
\begin{align*}
    \sigma_{\min}\left(\sum_{i=1}^{t} \phi_i \phi_i^\top \right) \geq t \sigma_{o}^2\min \{ \sigma_w^2, \sigma_z^2, \sigma_u^2 \} - 2\sqrt{2t} \Upsilon_w^2 \Hes \sqrt{\log\left(\frac{2\Hes(m+p)}{\delta}\right)}.
\end{align*}
For all $t\geq T_{o} \coloneqq \frac{32 \Upsilon_w^4 \Hes^2 \log\left(\frac{2\Hes(m+p)}{\delta}\right)}{\sigma_{o}^4 \min \{ \sigma_w^4, \sigma_z^4, \sigma_u^4 \}}$, we have the stated lower bound.
\end{proof}

Combining Lemma \ref{openloop_persistence_appendix} with Theorem \ref{theo:closedloopid} gives 
\begin{equation*} 
    \| \modelearnf  - \modelearn \| \leq  \frac{\kappa_{e}}{ \sigma_{o} \sqrt{\Tburn} \sqrt{ \frac{\min \left\{ \sigmaWlow^2, \sigmaZlow^2, \sigma_u^2 \right\}}{2} }},
\end{equation*}
at the end of warm-up, with probability at least $1-2\delta$, in parallel with Appendix \ref{apx:2normbound}.

\subsection{Persistence of Excitation Condition of $\Mcontrol \in \Mcontrolset$} \label{subsec:persistence}
In this section, we will provide the precise condition of the \DFC policies in $\Mcontrolset$, which provides the persistence of excitation if the underlying system is known while designing the control input via any \DFC policy in $\Mcontrolset$. The condition is given in (\ref{precise_persistence}). Note that in the controller design of \alg, we don't have access to the actual system. In the next section, we show that even though we have errors in the estimates, if the errors are small enough, we can still have persistence of excitation in the inputs. 

Now consider when the underlying system is known. If that's the case, the following are the inputs and outputs of the system:
\begin{align*}
     u_t &=  \sum_{j=0}^{H'-1} M_t^{[j]}  \nature_{t-j}(\Markov)\\
     y_t &= [G^{[0]} \enskip G^{[1]} \ldots   \enskip G^{[H]}]  \left[ u_{t}^\top \enskip u_{t-1}^\top \ldots u_{t-H}^\top \right]^\top + \nature_t(\Markov) + \mathbf{r_{t}}
\end{align*}
where $\mathbf{r_{t}} = \sum_{k=H+1}^{t-1} G^{[k]} u_{t-k}$. For $\Hes$ defined in Section 4.2, $\Hes \geq \max\{ 2n+1, \frac{\log(c_H T^2 \sqrt{m} / \sqrt{\lambda})}{\log(1/\upsilon)} \}$, define 
\begin{equation*}
    \phi_t = \left[ y_{t-1}^\top \ldots y_{t-\Hes}^\top \enskip u_{t-1}^\top \ldots \enskip u_{t-\Hes}^\top \right]^\top \in \mathbb{R}^{(m+p)\Hes}.
\end{equation*}
We have the following decompositions for $\phi_t$:
\begin{align*}
\phi_t
&\!\!=\!\! 
\underbrace{\begin{bmatrix}
     G^{[0]}  \! & \! G^{[1]} \! & \! \ldots \! & \! \ldots \! & \! \ldots \! & \! G^{[H]} \! & \! 0_{m\times p} \! & \! 0_{m\times p} \! & \! \ldots \! & \! 0_{m\times p} \\
    0_{m \times p} \! & \! G^{[0]} \! & \! \ldots \! & \! \ldots \! & \! \ldots \! & \! G^{[H-1]} \! & \! G^{[H]} \! & \! 0_{m \times p} \! & \! \ldots \! & \! 0_{m\times p} \\
    \! & \! \! & \! \ddots \! & \! \! & \! \! & \! \! & \! \! & \! \ddots \! & \! \! & \! \\
    0_{m \times p} \! & \! \ldots \! & \! 0_{m \times p} \! & \! G^{[0]} \! & \! G^{[1]} \! & \! \ldots \! & \! \ldots \! & \! \ldots \! & \! G^{[H-1]} \! & \! G^{[H]} \\
    I_{p\times p} \! & \! 0_{p\times p} \! & \! 0_{p\times p} \! & \! 0_{p\times p} \! & \! 0_{p\times p} \! & \! 0_{p\times p} \! & \! \ldots \! & \! \ldots \! & \! \ldots \! & \! 0_{p\times p} \\
    0_{p\times p} \! & \! I_{p\times p} \! & \!  0_{p\times p} \! & \! 0_{p\times p} \! & \! 0_{p\times p} \! & \! 0_{p\times p} \! & \! \ldots \! & \! \ldots \! & \! \ldots \! & \! 0_{p\times p} \\
    \! & \! \! & \! \ddots \! & \! \! & \! \\
    0_{p\times p} \! & \! 0_{p\times p} \! & \! \ldots \! & \! I_{p\times p} \! & \! 0_{p\times p} \! & \! \ldots \! & \! \ldots \! & \! \ldots \! & \!\ldots  \! & \! 0_{p\times p} \\ 
\end{bmatrix}}_{\T_\Markov \in \R^{\Hes(m+p) \times (\Hes+H)p}}\! 
\underbrace{\begin{bmatrix}
    u_{t-1}\\
    \vdots \\
    u_{t-H}\\
    \vdots \\
    u_{t-H-\Hes} 
\end{bmatrix}}_{\mathcal{U}_t}\!+\!\! \underbrace{\begin{bmatrix}
    \nature_{t-1}\\
    \vdots \\
    \nature_{t-\Hes} \\
    0_p \\
    \vdots \\
    0_p 
\end{bmatrix}}_{B_y(\Markov)(t)} \!\!+\!\! 
\underbrace{\begin{bmatrix}
    \mathbf{r_{t-1}}\\
    \vdots \\
    \mathbf{r_{t-\Hes}} \\
    0_p \\
    \vdots \\
    0_p 
\end{bmatrix}}_{\mathbf{R}_t}
\end{align*}
\begin{align*}
&\mathcal{U}_t \!\!=\!\! \underbrace{\begin{bmatrix}
     M_{t-1}^{[0]}  \!&\! M_{t-1}^{[1]}  \!&\! \ldots \!&\! \ldots \!&\! M_{t-1}^{[H'-1]}  \!&\! 0_{p\times m} \!&\! 0_{p\times m} \!&\! \ldots \!&\! 0_{p\times m} \\
    0_{p \times m} \!&\! M_{t-2}^{[0]}  \!&\! \ldots \!&\! \ldots \!&\! M_{t-2}^{[H'-2]} \!&\! M_{t-2}^{[H'-1]} \!&\!\! 0_{p \times m} \!\!&\! \ldots \!&\! 0_{p\times m} \\
    \!&\! \!&\! \ddots \!&\! & & & & \ddots & & \\
    0_{p \times m} \!&\! \ldots \!&\! 0_{p \times m} \!&\! M_{t-\Hes-H}^{[0]} \!&\! \ldots \!&\! \ldots \!&\! \ldots \!\!&\!\! \ldots \!&\! M_{t-\Hes-H}^{[H'-1]}\!\!\!
\end{bmatrix}}_{\T_{\Mcontrol_{t}} \in \R^{(\Hes+H)p \times m(H+H'+\Hes-1) }}
\underbrace{\begin{bmatrix}
    \nature_{t-1}(\Markov) \\
    \nature_{t-2}(\Markov) \\
    \vdots \\
    \nature_{t-H'+1}(\Markov)\\
    \vdots \\
    \nature_{t-\Hes-H-H'+1}(\Markov)
\end{bmatrix}}_{B(\Markov)(t)}
\end{align*}
\begin{align*}
B(\Markov)(t) &\!=\!  \underbrace{\begin{bmatrix}
   I_m \!&\! 0_{m} \!&\! \ldots \!&\! 0_{m} \!&\! C \!&\! CA \!&\! \ldots \!&\! \ldots \!&\! \ldots\!&\!CA^{t-3} \\
   0_m \!&\! I_m \!&\!  \!&\! 0_m \!&\! 0_{m\times n} \!&\! C \!&\!\ldots \!&\!\ldots \!&\! \ldots \!&\! CA^{t-4}\\
    \!&\! \!&\! \ddots \!&\! \!&\! \!&\! \!&\!\ddots \!&\! \!&\! \ddots\!&\! \\
    0_m \!&\! 0_m \!&\! \ldots \!&\! I_m \!&\! 0_{m\times n} \!&\! \ldots  \!&\!\ldots \!&\!\!\! C \!&\! \ldots \!&\! CA^{t-\Hes-H-H'-1}  
\end{bmatrix}}_{\mathcal{O}_t}
\underbrace{\begin{bmatrix}
    z_{t-1} \\
    z_{t-2} \\
    \vdots \\
    z_{t-\Hes-H-H'+1} \\
    w_{t-2} \\
    w_{t-3} \\
    \vdots \\
    w_{1}
\end{bmatrix}}_{\boldsymbol{\eta}_{t}}
\end{align*}
and $
B_y(\Markov)(t) \!=\!  \underbrace{\begin{bmatrix}
   \begin{matrix}
       I_m \!&\! 0_m \!&\! \ldots \!&\! \ldots \!&\! 0_m \!&\! C \!&\!\ldots \!&\!\ldots \!&\! \ldots \!&\! CA^{t-3} \\
    \!&\! \ddots \!&\!  \!&\!  \!&\! \vdots \!&\! \!&\!\ddots \!&\! \!&\! \ddots\!&\! \\
    0_m \!&\! \ldots \!&\! I_m \!&\! \ldots \!&\! 0_{m} \!&\!0_{m\times n}  \!&\!\ldots \!&\!\!\! C \!&\! \ldots \!&\! CA^{t-\Hes-2}  
   \end{matrix} \\
   \begin{matrix}
      \bigzero_{ \left(p\Hes\right) \times \left((\Hes+H+H'-1)m + (t-2)n\right)}
   \end{matrix}
\end{bmatrix}}_{\bar{\mathcal{O}}_t} \boldsymbol{\eta}_{t}$. 

Combining all gives 
\begin{equation*}
    \phi_t = \left(\T_\Markov \T_{\Mcontrol_{t}} \mathcal{O}_t + \bar{\mathcal{O}}_t \right) \boldsymbol{\eta}_{t} + \mathbf{R}_t.
\end{equation*}
\vspace{0.3cm}
\begin{mdframed}
\begin{condition}
For the given system $\Theta$, for $t \geq H + H' + \Hes$, $\T_\Markov \T_{\Mcontrol_{t}} \mathcal{O}_t + \bar{\mathcal{O}}_t$ is full row rank for all $\Mcontrol \in \Mcontrolset$, \textit{i.e.},  
\begin{equation} \label{precise_persistence}
    \sigma_{\min} (\T_\Markov \T_{\Mcontrol_{t}} \mathcal{O}_t + \bar{\mathcal{O}}_t) > \sigma_c > 0.
\end{equation}
\end{condition}
\end{mdframed}

\vspace{0.3cm}

\subsection{Persistence of Excitation in Adaptive Control Period}
\label{subsec:theo:persist_adaptive}
In this section, we show that the Markov parameter estimates of \alg are well-refined that, the controller of \alg constructed by using a \DFC policy in $\Mcontrolset$ still provides persistence of excitation. In other words, we will show that the inaccuracies in the model parameter estimates do not cause lack of persistence of excitation in adaptive control period.

First we have the following lemma, that shows inputs have persistence of excitation during the adaptive control period. Let $d = \min\{m,p\}$. Using (\ref{gamma_G}) and (\ref{gamma_H}), define 
\begin{align*}
T_{\epsilon_\Markov} &= 4 c_1^2 \kappa_\Mcontrolset^2 \kappa_\Markov^2 \gamma_\Markov^2 \gamma_\mathcal{H}^2 T_{\modelearn} \quad T_{cl} = \frac{T_{\epsilon_\Markov}}{\left(\frac{3\sigma_c^2 \min \{ \sigmaWlow^2, \sigmaZlow^2 \} }{8\kappa_u^2\kappa_y \Hes} - \frac{1}{10T}\right)^2},\\
T_c &= \frac{2048 \Upsilon_c^4 \Hes^2 \log\left(\frac{\Hes(m+p)}{\delta}\right) + H'mp\log\left(\kappa_\Mcontrolset\sqrt{d} + \frac{2}{\epsilon} \right)}{\sigma_{c}^4 \min \{ \sigma_w^4, \sigma_z^4 \}}.
\end{align*}
for $$\epsilon = \min \left\{1, \frac{\sigma_{c}^2 \min \{ \sigmaWlow^2, \sigmaZlow^2 \} \sqrt{\min\{m,p\}} }{68 \kappa_\nature^3 \kappa_\Markov \Hes \left(2\kappa_\Mcontrolset^2 + 3\kappa_\Mcontrolset + 3 \right) } \right \}$$

\begin{lemma} \label{closedloop_persistence_appendix}
After $T_{c}$ time steps in the adaptive control period, with probability $1-3\delta$, we have persistence of excitation for the remainder of adaptive control period, 
\begin{equation}
\sigma_{\min}\left(\sum_{i=1}^{t} \phi_i \phi_i^\top \right) \geq t \frac{\sigma_{c}^2 \min \{ \sigma_w^2, \sigma_z^2\}}{16}. 
\end{equation}
\end{lemma}
\begin{proof}
During the adaptive control period, at time $t$, the input of \alg is given by
\begin{align*}
    u_t &= \sum_{j=0}^{H'-1}M_t^{[j]}  \nature_{t-j}(\Markov) + M_t^{[j]} \left( \nature_{t-j}(\wh \Markov_i) - \nature_{t-j}(\Markov)  \right) 
\end{align*}
where 
\begin{align}
    \nature_{t-j}(\Markov) &= y_{t-j} - \sum_{k=1}^{t-j-1} G^{[k]} u_{t-j-k} = z_{t-j}+\sum_{k=1}^{t-j-1}CA^{t-j-k-1}w_k \\
    \nature_{t-j}(\wh \Markov_i) &= y_{t-j} - \sum_{k=1}^{H} \wh G_i^{[k]} u_{t-j-k} 
\end{align}
Thus, we obtain the following for $u_t$ and $y_t$,
\begin{align*}
     u_t &=  \sum_{j=0}^{H'-1} M_t^{[j]}  \nature_{t-j}(\Markov) + \underbrace{\sum_{j=0}^{H'-1} M_t^{[j]} \left( \sum_{k=1}^{t-j-1} [G^{[k]} - \wh G_i^{[k]}] u_{t-j-k}  \right) }_{u_{\Delta \nature}(t)}\\
     y_t &= [G^{[0]} \enskip G^{[1]} \ldots   \enskip G^{[H]}]  \left[ u_{t}^\top \enskip u_{t-1}^\top \ldots u_{t-H}^\top \right]^\top + \nature_t(\Markov) + \mathbf{r_{t}}
\end{align*}
where $\mathbf{r_{t}} = \sum_{k=H+1}^{t-1} G^{[k]} u_{t-k}$ and $\sum_{k=H}^{t-1} \|G^{[k]}\| \leq  \psi_\Markov(H+1) \leq  1/10T$ which is bounded by the assumption. Notice that $\|u_{\Delta \nature}(t) \| \leq \kappa_\Mcontrolset \kappa_u \epsilon_\Markov(1, \delta)$ for all $t \in \Tburn$. 
Using the definitions from Appendix \ref{subsec:persistence}, $\phi_t$ can be written as,
\begin{equation}
    \phi_t = \left(\T_\Markov \T_{\Mcontrol_{t}} \mathcal{O}_t + \bar{\mathcal{O}}_t \right) \boldsymbol{\eta}_{t} + \mathbf{R}_t + \T_\Markov \mathcal{U}_{\Delta \nature}(t)
\end{equation}
where 
\begin{equation*}
    \mathcal{U}_{\Delta \nature}(t) = \begin{bmatrix}
    u_{\Delta \nature}(t-1) \\
    u_{\Delta \nature}(t-2) \\
    \vdots \\
    u_{\Delta \nature}(t-\Hes)\\
    \vdots \\
    u_{\Delta \nature}(t-\Hes-H)
\end{bmatrix}.
\end{equation*}
Consider the following,
\begin{align*}
    \mathbb{E}\left[\phi_t \phi_t^\top \right] &= \mathbb{E}\bigg[ \left(\T_\Markov \T_{\Mcontrol_{t}} \mathcal{O}_t + \bar{\mathcal{O}}_t \right) \boldsymbol{\eta}_{t} \boldsymbol{\eta}_{t}^\top  \left(\T_\Markov \T_{\Mcontrol_{t}} \mathcal{O}_t + \bar{\mathcal{O}}_t \right)^\top + \boldsymbol{\eta}_{t}^\top \left(\T_\Markov \T_{\Mcontrol_{t}} \mathcal{O}_t + \bar{\mathcal{O}}_t \right)^\top \left( \T_\Markov \mathcal{U}_{\Delta \nature}(t) + \mathbf{R}_t \right) \\
    &+ \left( \T_\Markov \mathcal{U}_{\Delta \nature}(t) + \mathbf{R}_t \right)^\top \left(\T_\Markov \T_{\Mcontrol_{t}} \mathcal{O}_t + \bar{\mathcal{O}}_t \right) \boldsymbol{\eta}_{t} + \left( \T_\Markov \mathcal{U}_{\Delta \nature}(t) + \mathbf{R}_t \right)^\top \left( \T_\Markov \mathcal{U}_{\Delta \nature}(t) + \mathbf{R}_t \right)  \bigg] 
\end{align*}
\begin{align*}
    \sigma_{\min}\left(\mathbb{E}\left[\phi_t \phi_t^\top \right]\right) &\geq \sigma_c^2 \min \{ \sigmaWlow^2, \sigmaZlow^2 \} \\
    &- 2\kappa_\nature \left(\kappa_\Mcontrolset+\kappa_\Mcontrolset\kappa_\Markov + 1 \right) \sqrt{\Hes} ((1+\kappa_\Markov) \kappa_\Mcontrolset \kappa_u \epsilon_\Markov(1, \delta) \sqrt{\Hes} + \sqrt{\Hes} \kappa_u /10T ) \\
    &\geq \sigma_c^2 \min \{ \sigmaWlow^2, \sigmaZlow^2 \} - 2\kappa_u^2\kappa_y \Hes (2\kappa_\Markov\kappa_\Mcontrolset \epsilon_\Markov(1, \delta) + 1/10T)
\end{align*}
Note that for $\Tburn \geq T_{cl}$, $\epsilon_\Markov(1, \delta) \leq \frac{1}{2\kappa_\Mcontrolset \kappa_\Markov} \left(\frac{3\sigma_c^2 \min \{ \sigmaWlow^2, \sigmaZlow^2 \} }{8\kappa_u^2\kappa_y \Hes} - \frac{1}{10T}\right)$ with probability at least $1-2\delta$. Thus, we get 
\begin{equation}
    \sigma_{\min}\left(\mathbb{E}\left[\phi_t \phi_t^\top \right]\right) \geq \frac{\sigma_{c}^2}{4} \min \{ \sigmaWlow^2, \sigmaZlow^2 \},
\end{equation}
for all $t\geq \Tburn$. Using Lemma \ref{lem:boundednature}, we have that for $\Upsilon_c \coloneqq (\kappa_y + \kappa_u)$, $\|\phi_t\| \leq \Upsilon_c \sqrt{\Hes}$ with probability at least $1-2\delta$. Therefore, for a chosen $\Mcontrol \in \Mcontrolset$, using Theorem \ref{azuma}, we have the following with probability $1-3\delta$:
\begin{align}
    \lambda_{\max}\left (\sum_{i=1}^{t} \phi_i \phi_i^\top - \mathbb{E}[\phi_i \phi_i^\top]  \right) \leq 2\sqrt{2t} \Upsilon_c^2 \Hes \sqrt{\log\left(\frac{\Hes(m+p)}{\delta}\right)}.
\end{align}

In order to show that this holds for any chosen $\Mcontrol \in \Mcontrolset$, we adopt a standard covering argument. We know that from Lemma 5.4 of \citet{simchowitz2020improper}, the Euclidean diameter of $\Mcontrolset$ is at most $2\kappa_\Mcontrolset\sqrt{\min\{m,p\}}$, \textit{i.e.} $\| \Mcontrol_t\|_F \leq \kappa_\Mcontrolset\sqrt{\min\{m,p\}}$ for all $\Mcontrol_t \in \Mcontrolset$. Thus, we can upper bound the covering number as follows, 
\begin{equation*}
    \mathcal{N}(B(\kappa_\Mcontrolset\sqrt{\min\{m,p\}}), \|\cdot \|_F, \epsilon) \leq \left(\kappa_\Mcontrolset\sqrt{\min\{m,p\}} + \frac{2}{\epsilon} \right)^{H'mp}.
\end{equation*}
The following holds for all the centers of $\epsilon$-balls in $\| \Mcontrol_t\|_F$, for all $t\geq \Tburn$, with probability $1-3\delta$:
\begin{align}
    \lambda_{\max}\left (\sum_{i=1}^{t} \phi_i \phi_i^\top - \mathbb{E}[\phi_i \phi_i^\top]  \right) \leq 2\sqrt{2t} \Upsilon_c^2 \Hes \sqrt{\log\left(\frac{\Hes(m+p)}{\delta}\right) + H'mp\log\left(\kappa_\Mcontrolset\sqrt{\min\{m,p\}} + \frac{2}{\epsilon} \right)}.
\end{align}
Consider all $\Mcontrol$ in the $\epsilon$-balls, \textit{i.e.} effect of epsilon perturbation in $\|\Mcontrol \|_F$ sets, using Weyl's inequality we have with probability at least $1-3\delta$,
\begin{align*}
    \sigma_{\min}\left(\sum_{i=1}^{t} \phi_i \phi_i^\top \right) &\geq t \left( \frac{\sigma_{c}^2}{4} \min \{ \sigmaWlow^2, \sigmaZlow^2 \} - \frac{8 \kappa_\nature^3 \kappa_\Markov \Hes \epsilon \left(2\kappa_\Mcontrolset^2 + 3\kappa_\Mcontrolset + 3 \right)}{\sqrt{\min\{m,p \}}} \left(1+\frac{1}{10T}\right) \right) \\
    &\quad\qquad- 2\sqrt{2t} \Upsilon_c^2 \Hes \sqrt{\log\left(\frac{\Hes(m+p)}{\delta}\right) + H'mp\log\left(\kappa_\Mcontrolset\sqrt{\min\{m,p\}} + \frac{2}{\epsilon} \right)}.
\end{align*}
for $\epsilon \leq 1$. Let $\epsilon = \min \left\{1, \frac{\sigma_{c}^2 \min \{ \sigmaWlow^2, \sigmaZlow^2 \} \sqrt{\min\{m,p\}} }{68 \kappa_\nature^3 \kappa_\Markov \Hes \left(2\kappa_\Mcontrolset^2 + 3\kappa_\Mcontrolset + 3 \right) } \right \}$. For this choice of $\epsilon$, we get 

\begin{align*}
    \sigma_{\min}\left(\sum_{i=1}^{t} \phi_i \phi_i^\top \right) &\geq t \left( \frac{\sigma_{c}^2}{8} \min \{ \sigmaWlow^2, \sigmaZlow^2 \} \right)\\
    &\quad\qquad- 2\sqrt{2t} \Upsilon_c^2 \Hes \sqrt{\log\left(\frac{\Hes(m+p)}{\delta}\right) + H'mp\log\left(\kappa_\Mcontrolset\sqrt{\min\{m,p\}} + \frac{2}{\epsilon} \right)}.
\end{align*}

For picking $\Tburn \geq T_c$, we can guarantee that after $T_c$ time steps in the first epoch we have the advertised lower bound. 
\end{proof}
Combining Lemma \ref{closedloop_persistence_appendix} with Theorem \ref{theo:closedloopid} gives 
\begin{equation} 
    \| \modelearni  - \modelearn \| \leq  \frac{\kappa_{e}}{ \sigma_{c} \sqrt{2^{i-1}\Tbase} \sqrt{ \frac{ \min \{ \sigmaWlow^2, \sigmaZlow^2\}}{16}}  },\label{persistent_in_effect}
\end{equation}
for all $i$, with probability at least $1-4\delta$. Setting
\[ 
\sigma_\star^2 \coloneqq \min \left\{ \frac{\sigma_o^2\sigmaWlow^2}{2},\frac{\sigma_o^2\sigmaZlow^2}{2}, \frac{\sigma_o^2\sigma_y^2}{2}, \frac{\sigma_c^2\sigmaWlow^2 }{16}, \frac{\sigma_c^2\sigmaZlow^2 }{16} \right \},
\]
provides the guarantee in Appendix \ref{apx:2normbound} for warm-up and adaptive control periods.

\section{Boundedness Lemmas } \label{apx:Boundedness}

\begin{lemma}[Bounded Nature's $y$]. For all $t \in [T]$, the following holds with probability at least $1-\delta$,
\[
\left\|\nature_{t}(\mathbf{G})\right\| \leq \kappa_{b}:=\sigmaZup \sqrt{2m \log \frac{6mT}{\delta}} + \|C\| \Phi(A) \sqrt{2n} \left(\rho(A)^t \sqrt{\|\Sigma\| \log \frac{6n}{\delta}} + \frac{\sigmaWup \sqrt{\log \frac{4nT}{\delta}}}{1-\rho(A)}\right).
\]
\end{lemma}
\begin{proof}
Using Lemma \ref{subgauss lemma}, the following hold for all $t \in [T]$, with probability at least $1-\delta$,

\begin{equation} \label{noisebounds}
    \|w_t\| \leq \sigmaWup \sqrt{2n \log \frac{6nT}{\delta}}, \qquad \| z_t \| \leq \sigmaZup \sqrt{2m \log \frac{6mT}{\delta}}, \qquad \| x_0\| \leq \sqrt{2n\|\Sigma\| \log \frac{6n}{\delta}}.
\end{equation}
Thus we have,
\begin{align} \label{naturedecomp}
    \|\nature_t(\Markov) \| \!&=\! \|z_t + CA^t x_0 + \sum_{i=0}^{t-1} CA^{t-i-1} w_i \| \!\leq\! \| z_t\| \!+\! \|C\| \left(\!\|A^t\| \|x_0\| \!+\! \left(\max_{1\leq t\leq T}\| w_t\|\right)\! \sum_{i=0}^\infty \| A^{i} \| \right) \nonumber\\
    &\leq \sigmaZup \sqrt{2m \log \frac{6mT}{\delta}} + \|C\| \Phi(A) \sqrt{2n} \left(\rho(A)^t \sqrt{\|\Sigma\| \log \frac{6n}{\delta}} + \frac{\sigmaWup \sqrt{\log \frac{4nT}{\delta}}}{1-\rho(A)}\right). 
\end{align}
\end{proof}

\begin{lemma}[Boundedness Lemma]\label{lem:boundednature} Let $\delta \in (0,1)$, $T > \Tburn \geq \Tmax$ and $\psi_\Markov(H+1)\leq 1/10T$. For \alg, we have the boundedness of the following with probability at least $1-2\delta$:\\
$\textbf{Nature's  }\mathbf{y:}$ $\|\nature_t(\Markov)\| \!\leq\! \kappa_\nature,$ $\forall t$,\\
\textbf{Inputs:} $\|u_t\| \leq \kappa_u \coloneqq 2\max \{ \kappa_{u_{b}},\kappa_\Mcontrolset \kappa_\nature \}$,$\forall t$, $\quad$ \textbf{Outputs:}
$\|y_t\|\leq \kappa_y \coloneqq \kappa_\nature + \kappa_\Markov\kappa_u$, $\forall t$, and \\ 
$\textbf{Nature's  }\mathbf{y}$ \textbf{estimates:} $\|\nature_t(\wh \Markov)\| \leq 2\kappa_\nature~~$, for all $t>\Tburn$.
\end{lemma}
Proof of this lemma follows similarly from the proof of Lemma 6.1 in \citet{simchowitz2020improper}.

\subsection{Additional Bound on the Markov Parameter Estimates} \label{subsec:markovbound}
Define $\alpha$, such that $\alpha \!\leq\! \strong \big(\sigma_z^2+\sigma_w^2\left(\frac{\sigma_{\min}\left(C\right)}{1+\|A\|^2}\right)^2\big)$, where right hand side is the effective strong convexity parameter. Define $ T_{cx} \coloneqq T_{\modelearn} \frac{16 c_1^2\kappa_\nature^2 \kappa_\Mcontrolset^2 \kappa_\Markov^2 H' \gamma_\Markov^2 \gamma_\mathcal{H}^2 \strong}{\alpha}, \quad T_{\epsilon_\Markov} \coloneqq 4 c_1^2 \kappa_\Mcontrolset^2 \kappa_\Markov^2 \gamma_\Markov^2 \gamma_\mathcal{H}^2 T_{\modelearn}$ and $T_r = c_1^2 \gamma_\Markov^2 \gamma_\mathcal{H}^2 \kappa_\psi^2 T_{\modelearn}/r^2$. 

\begin{lemma}[Additional Boundedness of Markov Parameter Estimation Error] \label{lem:concentrationgeneral} Let $\Tburn > \Tmax$, \textit{i.e.} $\Tburn > \max \{ T_{cx}, T_{\epsilon_\Markov}, T_r \}$ and $\psi_\Markov(H+1)\leq 1/10T$. Then 
\begin{equation*}
    \|\sum_{j\geq 1} \wh G_i^{[j]} - G^{[j]} \| \leq \epsilon_\Markov(i,\delta) \leq \min \left\{\frac{1}{4 \kappa_{b} \kappa_{\mathcal{M}} \kappa_{\mathbf{G}}} \sqrt{\frac{\alpha}{H^{\prime} \underline{\alpha}_{l o s s}}}, \frac{1}{2 \kappa_{\mathcal{M}} \kappa_ \mathbf{G}}, \frac{r}{\kappa_\psi}\right\}
\end{equation*}
with probability at least $1-4\delta$, where $\epsilon_\Markov(i,\delta) = \frac{2c_1 \gamma_\Markov \gamma_\mathcal{H} \kappa_{e}}{\sigma_\star \sqrt{2^{i-1}\Tbase}}$.
\end{lemma}
\begin{proof}
At the beginning of epoch $i$, using persistence of excitation with high probability in (\ref{LemmaWarmLast}), we get 
\begin{align}
    \sum_{j\geq 1}^H \|\wh C_i \wh A_i^{j-1} \wh B_i - CA^{j-1}B\| \leq \epsilon_\Markov(i,\delta)/2 =  \frac{c_1 \gamma_\Markov \gamma_\mathcal{H} \kappa_{e}}{\sigma_\star \sqrt{2^{i-1}\Tbase}}.
\end{align}

From the assumption that $\psi_\Markov(H+1)\leq 1/10T$, we have that $\sum_{j\geq H+1} \|\wh G^{[j]}_1-G^{[j]}\| \leq \epsilon_\Markov(1, \delta)/2$. The second inequality follows from the choice of  $T_{\epsilon_\Markov}$,$ T_{cx}$ and $T_r$.
\end{proof}

\section{Proofs for Regret Bound}\label{apx:Regret}

In order to prove Theorem \ref{thm:polylogregret}, we follow the proof steps of Theorem 5 of \citet{simchowitz2020improper}. The main difference is that, \alg updates the Markov parameter estimates in epochs throughout the adaptive control period which provides decrease in the gradient error in each epoch. These updates allow \alg to remove $\OO(\sqrt{T})$ term in the regret expression of Theorem 5. In the following, we state how the proof of Theorem 5 of \citet{simchowitz2020improper} is adapted to the setting of \alg. 

\begin{theorem}\label{thm:main}
Let $H'$ satisfy $H'\geq 3H\geq 1$, $\psi(\lfloor H'/2\rfloor-H)\leq \kappa_\Mcontrolset/T$ and $\psi(H+1)\leq 1/10T$. If (\ref{asm:lipschitzloss}) holds for the given setting, after a warm-up period time $\Tburn \geq \Tmax$, if \alg runs with step size $\eta_t = \frac{12}{\alpha t}$, then with probability at least $1-5\delta$, the regret of \alg is bounded as 
%
\begin{align*}
    &\reg(T) \lesssim ~ \Tburn L \kappa_y^2 ~+~ \frac{L^2{H'}^3\min\lbrace{m,p\rbrace}\kappa_\nature^4\kappa_{\Markov}^4\kappa_{\Mcontrolset}^2}{\min\lbrace \alpha, L\kappa_\nature^2\kappa_{\Markov}^2\rbrace}\bigg(\!1\!+\!\frac{\smooth}{\min\lbrace{m,p\rbrace} L\kappa_\Mcontrolset}\!\bigg)\log\bigg(\frac{T}{\delta}\bigg) \\
    &+ \sum\nolimits_{t=\Tburn+1}^{T}  \epsilon^2_\Markov \!\!\left( \Big \lceil \log_2 \big( \frac{t}{\Tburn}\big) \Big \rceil, \delta \right)  H' \kappa_\nature^2 \kappa_{\Mcontrolset}^2 \bigg( \frac{ \kappa_\Markov^2  \kappa_\nature^2 \left(\smooth + L\right)^2}{\alpha} + \kappa_y^2  \max \Big\{L, \frac{L^2}{\alpha} \Big\} \bigg).
\end{align*}
\end{theorem}

\begin{proof}
Consider the hypothetical ``true prediction'' y's, $y_t^{pred}$ and losses, $f_t^{pred}(M)$ defined in Definition 8.1 of \citet{simchowitz2020improper}. Up to truncation by $H$, they describe the true counterfactual output of the system for \alg inputs during the adaptive control period and the corresponding counterfactual loss functions. Lemma \ref{Lemma8.1}, shows that at all epoch $i$, at any time step $t\in[t_i,~ \ldots,~t_{i+1}-1]$, the gradient $f_t^{pred}(M)$ is close to the gradient of the loss function of \alg:

\begin{equation} \label{gradientdif}
\left\|\nabla f_t\left(\Mcontrol,\wh \Markov_i,\nature_1(\wh \Markov_i),\ldots,\nature_t(\wh \Markov_i)\right)-\nabla f_{t}^{\text {pred }}(\Mcontrol)\right\|_{\mathrm{F}} \leq C_{\text {approx}} \epsilon_\Markov(i, \delta),
\end{equation}
where $C_{\text {approx }} \coloneqq \sqrt{H'} \kappa_\Markov \kappa_{\Mcontrolset} \kappa_\nature^2 \left(16\smooth +24 L\right)$. For a comparing controller $\Mcontrol_{comp} \in \Mcontrolset(H', \kappa_\Mcontrolset )$ and the competing set $\Mcontrolset_\psi(H_0', \kappa_\psi)$, where $\kappa_\Mcontrolset = (1+r)\kappa_\psi$ and $H'_0 = \lfloor \frac{H'}{2} \rfloor - H$, we have the following regret decomposition: 
\begin{align} 
    \reg(T)& \leq \underbrace{ \left(\sum_{t=1}^{\Tburn} \ell_{t}\left(y_{t}, u_{t}\right)\right) }_{\text{warm-up regret}} + \underbrace{\left(\sum_{t=\Tburn+1}^{T} \ell_{t}\left(y_{t}, u_{t}\right)-\sum_{t=\Tburn+1}^{T} F_{t}^{\text{pred}}\left[\Mcontrol_{t: t-H}\right]\right)}_{\text{algorithm truncation error}} \nonumber \\
    &+ \underbrace{\left(\sum_{t=\Tburn+1}^{T} F_{t}^{\mathrm{pred}}\left[\mathbf{M}_{t: t-H}\right]-\sum_{t=\Tburn+1}^{T} f_{t}^{\mathrm{pred}}\left(\Mcontrol_{comp}\right)\right)}_{f^{\mathrm{pred}} \text{ policy regret}} \nonumber \\
    &+ \underbrace{\left( \sum_{t=\Tburn+1}^{T} f_{t}^{\mathrm{pred}}\left(\Mcontrol_{comp}\right)-\inf _{\Mcontrol \in \mathcal{M}_{\psi}} \sum_{t=\Tburn+1}^{T} f_{t}\left(\Mcontrol, \Markov, \nature_1( \Markov),\ldots,\nature_t(\Markov) \right) \right)}_{\text{comparator approximation error}} \nonumber \\
    &+ \underbrace{\left( \inf_{\Mcontrol \in \mathcal{M}_{\psi}} \sum_{t=\Tburn+1}^{T} f_{t}\left(\Mcontrol, \Markov, \nature_1( \Markov),\ldots,\nature_t(\Markov) \right) - \inf_{\Mcontrol \in \mathcal{M}_{\psi}} \sum_{t=\Tburn+1}^{T} \ell_{t}\left(y^\Mcontrol_t, u^\Mcontrol_t\right) \right) }_{\text{comparator truncation error}} \nonumber \\
    &+ \underbrace{\left(\inf_{\Mcontrol \in \mathcal{M}_{\psi}} \sum_{t=1}^{T} \ell_{t}\left(y^\Mcontrol_t, u^\Mcontrol_t\right) - \sum_{t=0}^T \ell(y^{\pi_\star},u^{\pi_\star}) \right)}_{\text{policy approximation error}}
\end{align}
Notice that the last term is only required to extend the Theorem \ref{thm:polylogregret} to Corollary \ref{regretLQG}. The result of Theorem \ref{thm:polylogregret} does not require the last term. We will consider each term separately. \\

\noindent \textbf{Warm-up Regret:} From (\ref{asm:lipschitzloss}) and Lemma \ref{lem:boundednature}, we get $\sum_{t=1}^{\Tburn} \ell_{t}\left(y_{t}, u_{t}\right) \leq \Tburn L \kappa_y^2 $. \\

\noindent \textbf{Algorithm Truncation Error:} From (\ref{asm:lipschitzloss}), we get 
\begin{align*}
    \sum_{t=\Tburn+1}^{T} \!\!\!\!\!\!\!\! \ell_{t}\left(y_{t}, u_{t}\right)-\!\!\!\!\!\sum_{t=\Tburn+1}^{T}\!\!\!\!\!\!\!\! F_{t}^{\text{pred}}\left[\Mcontrol_{t: t-H}\right] &\leq \sum_{t=\Tburn+1}^{T}\left|\ell_{t}\left(y_{t}, u_{t}\right)-\ell_{t}\left(\nature_t(\Markov) + \sum_{i=1}^H G^{[i]} u_{t-i}, u_{t}\right)\right| \\
    &\leq\sum_{t=\Tburn+1}^{T} L\kappa_y \left\|y_t - \nature_t(\Markov) + \sum_{i=1}^H G^{[i]} u_{t-i}  \right \| \\
    &\leq \sum_{t=\Tburn+1}^{T} L\kappa_y \left\| \sum_{i=H+1} G^{[i]} u_{t-i}  \right \| \\
    &\leq T L \kappa_y \kappa_u \psi_\Markov(H+1)
\end{align*}
Since $\psi_\Markov(H+1)\leq 1/10T$, we get $\sum_{t=\Tburn+1}^{T} \ell_{t}\left(y_{t}, u_{t}\right)-\sum_{t=\Tburn+1}^{T} F_{t}^{\text{pred}}\left[\Mcontrol_{t: t-H}\right] \leq L \kappa_y \kappa_u/10$.\\

\noindent \textbf{Comparator Truncation Error:} Similar to algorithm truncation error above,
\begin{align*}
    \inf_{\Mcontrol \in \mathcal{M}_{\psi}} \sum_{t=\Tburn+1}^{T} \!\!\!\!\!\!\!f_{t}\left(\Mcontrol, \Markov, \nature_1( \Markov),\ldots,\nature_t(\Markov) \right) - \inf_{\Mcontrol \in \mathcal{M}_{\psi}} \sum_{t=\Tburn+1}^{T} \!\!\!\!\!\!\!\ell_{t}\left(y^\Mcontrol_t, u^\Mcontrol_t\right) 
    &\leq T L\kappa_\Markov \kappa_\Mcontrolset^2 \kappa_\nature^2 \psi_\Markov(H+1) \\
    &\leq L\kappa_\Markov \kappa_\Mcontrolset^2 \kappa_\nature^2 / 10
\end{align*}

\noindent \textbf{Policy Approximation Error:} By the assumption that $M_\star$ lives in the given convex set $\Mcontrolset_\psi$ and (\ref{asm:lipschitzloss}), using Lemma \ref{lem:deviation}, we get 
\begin{align*}
    \inf_{\Mcontrol \in \mathcal{M}_{\psi}} \sum_{t=1}^{T} \ell_{t}\left(y^\Mcontrol_t, u^\Mcontrol_t\right) - \sum_{t=1}^T \ell_t(y_t^{\pi_\star},u_t^{\pi_\star})
    &\leq \sum_{t=1}^T \ell_{t}\left(y^{\Mcontrol_\star}_t, u^{\Mcontrol_\star}_t\right) - \ell_t(y_t^{\pi_\star},u_t^{\pi_\star})  \\
    &\leq T L \kappa_y \left( \psi(H_0') \kappa_\nature + \psi(H_0') \kappa_\Markov \kappa_\nature \right) \\
    &\leq 2 T L \kappa_y \kappa_\Markov \kappa_\nature \psi(H_0')
\end{align*}
Since $\psi(H_0') \leq \kappa_\Mcontrolset/T$, we get $\inf_{\Mcontrol \in \mathcal{M}_{0}} \sum_{t=1}^{T} \ell_{t}\left(y^\Mcontrol_t, u^\Mcontrol_t\right) - \sum_{t=1}^T \ell_t(y_t^{\pi_\star},u_t^{\pi_\star}) \leq  2 L \kappa_\Mcontrolset \kappa_y \kappa_\Markov \kappa_\nature$. \\

\noindent $\mathbf{f^{\mathrm{pred}} \textbf{ Policy Regret}:}$ In order to utilize Theorem \ref{theo:fpred}, we need the strong convexity, Lipschitzness and smoothness properties stated in the theorem. Due to Lemma \ref{lem:concentrationgeneral}, Lemmas \ref{lem:smooth}-\ref{lem:lipschitz} provide those conditions. Combining these with (\ref{gradientdif}), we obtain the following adaptation of Theorem \ref{theo:fpred}:
\begin{lemma}
For step size $\eta = \frac{12}{\alpha t}$, the following bound holds with probability $1-\delta$:
\begin{align*}
   &\mathbf{f^{\mathrm{pred}} \textbf{ policy regret } } + \frac{\alpha}{48} \sum_{t = \Tburn + 1}^T \!\!\!\!\!\!\!\|\Mcontrol_t - \Mcontrol_{comp}\|_F^2 \\
   &\lesssim \frac{L^2{H'}^3\min\lbrace{m,p\rbrace}\kappa_\nature^4\kappa_{\Markov}^4\kappa_{\Mcontrolset}^2}{\min\lbrace \alpha, L\kappa_\nature^2\kappa_{\Markov}^2\rbrace}\left(\!1\!+\!\frac{\smooth}{\min\lbrace{m,p\rbrace} L\kappa_\Mcontrolset}\!\right)\log\left(\frac{T}{\delta}\right) + \frac{1}{\alpha} \! \sum_{t=\Tburn+1}^T \!\!\!\!\!\!\! C_{\text{approx}}^2 \epsilon^2_\Markov \!\!\left( \left \lceil \log_2 \left( \frac{t}{\Tburn}\right) \right \rceil, \delta \right)
\end{align*}
\end{lemma}
\begin{proof}
Let $d = \min \{m, p\}$. We can upper bound the right hand side of Theorem \ref{theo:fpred} via following proof steps of Theorem 4 of \citet{simchowitz2020improper}: 
\begin{align}
    &\mathbf{f^{\mathrm{pred}} \textbf{p.r.} } \!\!-\!\! \left(\frac{6}{\alpha}\!\! \sum_{t=k+1}^{T}\!\!\left\|\boldsymbol{\epsilon}_{t}\right\|_{2}^{2}-\frac{\alpha}{48} \sum_{t=1}^{T}\left\|\Mcontrol_t - \Mcontrol_{comp} \right\|_{F}^{2}\right) \lesssim \frac{L^2{H'}^3 d \kappa_\nature^4\kappa_{\Markov}^4\kappa_{\Mcontrolset}^2}{\min\lbrace \alpha, L\kappa_\nature^2\kappa_{\Markov}^2\rbrace}\left(\!1\!+\!\frac{\smooth}{d L\kappa_\Mcontrolset}\!\right)\log\left(\frac{T}{\delta}\right) \nonumber  \\
    &\mathbf{f^{\mathrm{pred}} \textbf{p.r.} }  + \frac{\alpha}{48} \sum_{t=1}^{T}\left\|\Mcontrol_t - \Mcontrol_{comp} \right\|_{F}^{2} \lesssim \frac{L^2{H'}^3 d \kappa_\nature^4\kappa_{\Markov}^4\kappa_{\Mcontrolset}^2}{\min\lbrace \alpha, L\kappa_\nature^2\kappa_{\Markov}^2\rbrace}\left(\!1\!+\!\frac{\smooth}{d L\kappa_\Mcontrolset}\!\right)\log\left(\frac{T}{\delta}\right) \nonumber \\
    &\qquad\qquad\qquad\qquad\qquad\qquad\qquad\qquad\qquad + \frac{1}{\alpha} \sum_{t=\Tburn+1}^T  C_{\text{approx}}^2 \epsilon^2_\Markov\left( \left \lceil \log_2 \left( \frac{t}{\Tburn}\right) \right \rceil, \delta \right), \label{usegradientdiff} 
\end{align}
where (\ref{usegradientdiff}) follows from (\ref{gradientdif}). 
\end{proof}
\newpage
\noindent \textbf{Comparator Approximation Error:} 
\begin{lemma}
Suppose that $H' \geq 2H_{0}'-1+H,~\psi_\Markov(H\!+\!1) \leq 1 / 10T$. Then for all $\tau>0$,
\begin{align*}
&\textbf{Comp. \!app. \!err.} \leq 4L \kappa_y \kappa_u \kappa_\Mcontrolset \\
&\qquad +\!\!\!\!\! \sum_{t=\Tburn+1}^{T} \left[ \tau \left\|\Mcontrol_{t} \!-\! \Mcontrol_{\mathrm{comp}}\right\|_{F}^{2} + 8\kappa_y^2 \kappa_\nature^2 \kappa_{\Mcontrolset}^2 (H+H') \max \left\{L, \frac{L^2}{\tau} \right\}\epsilon^2_\Markov\left( \left \lceil \log_2 \left( \frac{t}{\Tburn}\right) \right \rceil, \delta \right) \right]
\end{align*}
\end{lemma}
\begin{proof}
The lemma can be proven using the proof of Proposition 8.2 of \citet{simchowitz2020improper}. Using Lemma E.3 and adapting Lemma E.4 in \citet{simchowitz2020improper} such that $\Mcontrol_{comp}^{[i]}= M_{*}^{[i]} I_{i \leq H_0'-1}+\sum_{a=0}^{H_0'-1} \sum_{b=0}^{H} \sum_{c=0}^{H_0'-1} M_{*}^{[a]} (\wh G^{[b]}_1-G^{[b]}) M_{*}^{[c]} \mathbb{I}_{a+b+c=i}$ for $\Mcontrol_{*} =  \argmin_{\Mcontrol \in \Mcontrolset_\psi} \sum_{t=\Tburn+1}^{T} \ell_t(y_{t}^\Mcontrol,u_{t}^\Mcontrol)$ and due to Lemma \ref{lem:concentrationgeneral} we have $\Mcontrol_{comp} \in \Mcontrolset$:
\begin{align*}
    &\sum_{t=\Tburn+1}^{T} f_{t}^{\mathrm{pred}}\left(\Mcontrol_{comp}\right)-\inf_{\Mcontrol \in \mathcal{M}_{0}} \sum_{t=\Tburn+1}^{T} f_{t}\left(\Mcontrol, \Markov, \nature_1( \Markov),\ldots,\nature_t(\Markov) \right) \\
    &\leq 4L\kappa_y \!\!\!\!\!\!\!\sum_{t=\Tburn+1}^{T}\!\!\!\!\! \epsilon^2_\Markov\left( \left \lceil \log_2 \left( \frac{t}{\Tburn}\right) \right \rceil, \delta \right) \kappa_{\Mcontrolset}^{2} \kappa_\nature \left(\kappa_{\Mcontrolset}\!+\!\frac{\kappa_\nature}{4 \tau}\right) \!+\! \kappa_u \kappa_\Mcontrolset \psi_\Markov(H\!+\!1) \!+\! (H\!+\!H')\tau \left\|\Mcontrol_{t} \!-\! \Mcontrol_{\mathrm{comp}}\right\|_{F}^{2}  \\
    &\leq \sum_{t=\Tburn+1}^{T} \left[ \tau \left\|\Mcontrol_{t} \!-\! \Mcontrol_{\mathrm{comp}}\right\|_{F}^{2} + 8\kappa_y^2 \kappa_\nature^2 \kappa_{\Mcontrolset}^2 (H+H') \max \left\{L, \frac{L^2}{\tau} \right\}\epsilon^2_\Markov\left( \left \lceil \log_2 \left( \frac{t}{\Tburn}\right) \right \rceil, \delta \right) \right] \\
    &+4T L \kappa_y \kappa_u \kappa_\Mcontrolset \psi_\Markov(H\!+\!1) \\
    &\leq \! 4L \kappa_y \kappa_u \kappa_\Mcontrolset \!+\!\!\!\!\!\!\!\! \sum_{t=\Tburn+1}^{T} \left[ \tau \left\|\Mcontrol_{t} \!-\! \Mcontrol_{\mathrm{comp}}\right\|_{F}^{2} + 8\kappa_y^2 \kappa_\nature^2 \kappa_{\Mcontrolset}^2 (H+H') \max \left\{L, \frac{L^2}{\tau} \right\}\epsilon^2_\Markov\left( \left \lceil \log_2 \left( \frac{t}{\Tburn}\right) \right \rceil, \delta \right) \right]
\end{align*}
\end{proof}

Combining all the terms bounded above, with $\tau = \frac{\alpha}{48}$ gives 
\begin{align*}
    &\reg(T) \\
    &\lesssim \Tburn L \kappa_y^2 + L \kappa_y \kappa_u/10 + L\kappa_\Markov \kappa_\Mcontrolset^2 \kappa_\nature^2 / 10 + 2 L \kappa_\Mcontrolset \kappa_y \kappa_\Markov \kappa_\nature + 4L \kappa_y \kappa_u \kappa_\Mcontrolset \\
    & + \frac{L^2{H'}^3\min\lbrace{m,p\rbrace}\kappa_\nature^4\kappa_{\Markov}^4\kappa_{\Mcontrolset}^2}{\min\lbrace \alpha, L\kappa_\nature^2\kappa_{\Markov}^2\rbrace}\left(\!1\!+\!\frac{\smooth}{\min\lbrace{m,p\rbrace} L\kappa_\Mcontrolset}\!\right)\log\left(\frac{T}{\delta}\right) + \frac{1}{\alpha} \! \sum_{t=\Tburn+1}^T \!\!\!\!\!\!\! C_{\text{approx}}^2 \epsilon^2_\Markov \!\!\left( \left \lceil \log_2 \left( \frac{t}{\Tburn}\right) \right \rceil, \delta \right) \\
    &+ \sum_{t=\Tburn+1}^{T} 8\kappa_y^2 \kappa_\nature^2 \kappa_{\Mcontrolset}^2 (H+H') \max \left\{L, \frac{48L^2}{\alpha} \right\}\epsilon^2_\Markov\left( \left \lceil \log_2 \left( \frac{t}{\Tburn}\right) \right \rceil, \delta \right)  \\
    &\lesssim \Tburn L \kappa_y^2 \\
    &+ \frac{L^2{H'}^3\min\lbrace{m,p\rbrace}\kappa_\nature^4\kappa_{\Markov}^4\kappa_{\Mcontrolset}^2}{\min\lbrace \alpha, L\kappa_\nature^2\kappa_{\Markov}^2\rbrace}\left(\!1\!+\!\frac{\smooth}{\min\lbrace{m,p\rbrace} L\kappa_\Mcontrolset}\!\right)\log\left(\frac{T}{\delta}\right) \\
    &+ \sum_{t=\Tburn+1}^{T} \epsilon^2_\Markov \!\!\left( \left \lceil \log_2 \left( \frac{t}{\Tburn}\right) \right \rceil, \delta \right) \left\{ \frac{H' \kappa_\Markov^2 \kappa_{\Mcontrolset}^2 \kappa_\nature^4 \left(\smooth + L\right)^2}{\alpha} + \kappa_y^2 \kappa_\nature^2 \kappa_{\Mcontrolset}^2 (H+H') \max \left\{L, \frac{48L^2}{\alpha} \right\} \right\}
\end{align*}
\end{proof}

Following the doubling update rule of \alg for the epoch lengths, after $T$ time steps of agent-environment interaction, the number of epochs is $\mathcal{O}\left(\log T\right)$. From Lemma~\ref{lem:concentrationgeneral}, at any time step $t$ during the $i$'th epoch, \textit{i.e.}, $t\in[ t_i,\ldots, t_{i}-1]$, $\epsilon^2_\Markov(i, \delta) = \OO(polylog(T)/2^{i-1} \Tbase)$. Therefore, update rule of \alg yields,
\begin{align}\label{eq:logG}
    \sum\nolimits_{t=\Tbase+1}^{T} \epsilon^2_\Markov \bigg( \big \lceil \log_2 \big( \frac{t}{\Tburn} \big) \big \rceil, \delta \bigg) =\sum\nolimits_{i=1}^{\OO \left(\log T\right)} 2^{i-1} \Tbase \epsilon^2_\Markov \left( i, \delta \right) \leq \OO \left(polylog(T)\right)
\end{align}

Using the result of (\ref{eq:logG}), we can bound the third term of the regret upper bound in Theorem~\ref{thm:main} with a $polylog(T)$ bound which gives the advertised result in Theorem \ref{thm:polylogregret} and using the policy approximation error term we obtain Corollary \ref{regretLQG}.

\null\hfill$\square$

\section{Additional Results} \label{apx:convex}

Consider the case where the condition on persistence of excitation of $\Mcontrolset$ does not hold. In order to efficiently learn the model parameters and minimize the regret, one can add an additional independent Gaussian excitation to the control input $u_t$ for each time step $t$. This guarantees the concentration of Markov parameter estimates, but it also results in an extra regret term in the bound of Theorem~\ref{thm:main}. If the variance of the added Gaussian vector is set to be $\wt\sigma^2$, exploiting the Lipschitzness of the loss functions, the additive regret of the random excitation is $\Tilde{\OO}(T\wt\sigma)$. Following the results in Lemma~\ref{lem:concentrationgeneral}, the additional random excitation helps in parameter estimation and concentration of Markov parameters up to the error of $\OO(polylog(T)/\sqrt{\wt\sigma^2t)}$. Since the contribution of the error in the Markov parameter estimates in the Theorem~\ref{thm:main} is quadratic, the contribution of this error in the regret through $R_3$ will be $\OO(polylog(T)/\wt\sigma^2)$. 

\begin{corollary}\label{corr:T2/3}
When the condition on persistent excitation of all $\Mcontrol $ is not fulfilled, adding independent Gaussian vectors with variance of $\OO(1/T^{1/3})$ to the inputs in adaptive control period results in the regret upper bound of $\Tilde{\OO}(T^{2/3})$. 
\end{corollary}


\section{Technical Lemmas and Theorems} \label{technical}
\begin{theorem}[Matrix Azuma \citep{tropp2012user}]\label{azuma} Consider a finite adapted sequence $\left\{\boldsymbol{X}_{k}\right\}$ of self-adjoint matrices
in dimension $d,$ and a fixed sequence $\left\{\boldsymbol{A}_{k}\right\}$ of self-adjoint matrices that satisfy
\begin{equation*}
\mathbb{E}_{k-1} \boldsymbol{X}_{k}=\mathbf{0} \text { and } \boldsymbol{A}_{k}^{2} \succeq \boldsymbol{X}_{k}^{2}  \text { almost surely. }
\end{equation*}
Compute the variance parameter
\begin{equation*}
\sigma^{2}\coloneqq \left\|\sum_{k} \boldsymbol{A}_{k}^{2} \right\|
\end{equation*}
Then, for all $t \geq 0$
\begin{equation*}
\mathbb{P}\left\{\lambda_{\max }\left(\sum_{k} \boldsymbol{X}_{k} \right) \geq t\right\} \leq d \cdot \mathrm{e}^{-t^{2} / 8 \sigma^{2}}
\end{equation*}

\end{theorem}

\begin{theorem}[Self-normalized bound for vector-valued martingales~\citep{abbasi2011improved}]
\label{selfnormalized}
Let $\left(\mathcal{F}_{t} ; k \geq\right.$
$0)$ be a filtration, $\left(m_{k} ; k \geq 0\right)$ be an $\mathbb{R}^{d}$-valued stochastic process adapted to $\left(\mathcal{F}_{k}\right),\left(\eta_{k} ; k \geq 1\right)$
be a real-valued martingale difference process adapted to $\left(\mathcal{F}_{k}\right) .$ Assume that $\eta_{k}$ is conditionally sub-Gaussian with constant $R$. Consider the martingale
\begin{equation*}
S_{t}=\sum_{k=1}^{t} \eta_{k} m_{k-1}
\end{equation*}
and the matrix-valued processes
\begin{equation*}
V_{t}=\sum_{k=1}^{t} m_{k-1} m_{k-1}^{\top}, \quad \overline{V}_{t}=V+V_{t}, \quad t \geq 0
\end{equation*}
Then for any $0<\delta<1$, with probability $1-\delta$
\begin{equation*}
\forall t \geq 0, \quad\left\|S_{t}\right\|^2_{V_{t}^{-1}} \leq 2 R^{2} \log \left(\frac{\operatorname{det}\left(\overline{V}_{t}\right)^{1 / 2} \operatorname{det}(V)^{-1 / 2}}{\delta}\right)
\end{equation*}

\end{theorem}

\begin{theorem}[Theorem 8 of \citet{simchowitz2020improper}] \label{theo:fpred}
Suppose that $\mathcal{K} \subset \mathbb{R}^{d}$ and $h \geq 1$. Let $F_{t}:=\mathcal{K}^{h+1} \rightarrow \mathbb{R}$ be a sequence of $L_{c}$ coordinatewise-Lipschitz functions with the induced unary functions
$f_{t}(x):=F_{t}(x, \ldots, x)$ which are $L_{\mathrm{f}}$-Lipschitz and $\beta$-smooth. Let $f_{t ; k}(x):=\mathbb{E}\left[f_{t}(x) | \mathcal{F}_{t-k}\right]$ be $\alpha$-strongly convex on $\mathcal{K}$ for a filtration $\left(\mathcal{F}_{t}\right)_{t \geq 1}$. Suppose that $z_{t+1}=\Pi_{\mathcal{K}}\left(z_{t}-\eta \boldsymbol{g}_{t} \right)$, where $\boldsymbol{g}_{t}=\nabla f_{t}\left(z_{t}\right)+\epsilon_{t}$ for  $\left\|\boldsymbol{g}_{t}\right\|_{2} \leq L_{\mathbf{g}},$ and $\operatorname{Diam}(\mathcal{K}) \leq D$. Let the gradient descent iterates be applied for $t \geq t_{0}$ for some $t_{0} \leq k,$ with $z_{0}=z_{1}=\cdots=z_{t_{0}} \in \mathcal{K}$ for $k\geq 1$. Then with step size $\eta_{t}=\frac{3}{\alpha t},$ the
following bound holds with probability $1 - \delta$ for all comparators $z_{\star} \in \mathcal{K}$ simultaneously:
\begin{align*}
\sum_{t=k+1}^{T} &f_{t}\left(z_{t}\right)-f_{t}\left(z_{\star}\right)-\left(\frac{6}{\alpha} \sum_{t=k+1}^{T}\left\|\boldsymbol{\epsilon}_{t}\right\|_{2}^{2}-\frac{\alpha}{12} \sum_{t=1}^{T}\left\|z_{t}-z_{\star}\right\|_{2}^{2}\right) \\ & \lesssim \alpha k D^{2}+\frac{\left(k L_{\mathrm{f}}+h^{2} L_{\mathrm{c}}\right) L_{\mathrm{g}}+k d L_{\mathrm{f}}^{2}+k \beta L_{\mathrm{g}}}{\alpha} \log (T)+\frac{k L_{\mathrm{f}}^{2}}{\alpha} \log \left(\frac{1+\log \left(e+\alpha D^{2}\right)}{\delta}\right)
\end{align*}
\end{theorem}

\begin{lemma}[Regularized Design Matrix Lemma \citep{abbasi2011improved}]
When the covariates satisfy $\left\|z_{t}\right\| \leq c_{m},$ with some $c_{m}>0$ w.p.1 then
\[
\log \frac{\operatorname{det}\left(V_{t}\right)}{\operatorname{det}(\lambda I)} \leq d \log \left(\frac{\lambda d +t c_{m}^{2}}{\lambda d}\right)
\]
where $V_t = \lambda I + \sum_{i=1}^t z_i z_i^\top$ for $z_i \in \mathbb{R}^d$.
\end{lemma}

\begin{lemma}[Norm of a subgaussian vector \citep{abbasi2011regret}]\label{subgauss lemma}
Let $v\in \R^d$ be a entry-wise $R$-subgaussian random variable. Then with probability $1-\delta$, $\|v\| \leq R\sqrt{2d\log(2d/\delta)}$.
\end{lemma}

\begin{lemma}[Lemma 8.1 of \citet{simchowitz2020improper}] \label{Lemma8.1}
For any $\Mcontrol \in \mathcal{M}$, let $f_{t}^{\text {pred }}(\Mcontrol)$ denote the unary counterfactual loss function induced by true truncated counterfactuals (Definition 8.1 of \citet{simchowitz2020improper}). During the $i$'th epoch of adaptive control period, at any time step $t\in[t_i,~ \ldots,~t_{i+1}-1]$, for all $i$, we have that
\begin{equation*}
\left\|\nabla f_t\left(\Mcontrol,\wh \Markov_i,\nature_1(\wh \Markov_i),\ldots,\nature_t(\wh \Markov_i)\right)-\nabla f_{t}^{\text {pred }}(\Mcontrol)\right\|_{\mathrm{F}} \leq C_{\text {approx }} \epsilon_\Markov(i, \delta),
\end{equation*}
where $C_{\text {approx }} \coloneqq \sqrt{H'} \kappa_\Markov \kappa_{\Mcontrolset} \kappa_\nature^2 \left(16\smooth +24 L\right)$.
\end{lemma}

\begin{lemma}[Lemma 8.2 of \citet{simchowitz2020improper}] \label{lem:smooth}
For any $\Mcontrol \in \mathcal{M}$, $f_{t}^{\text {pred }}(\Mcontrol)$ is $\beta$-smooth, where $\beta = 16H'\kappa_\nature^2 \kappa_\Markov^2\smooth$.
\end{lemma}

\begin{lemma}[Lemma 8.3 of \citet{simchowitz2020improper}] \label{lem:strongcvx}
For any $\Mcontrol \in \mathcal{M}$, given $\epsilon_\Markov(i, \delta) \leq \frac{1}{4\kappa_\nature\kappa_\Mcontrolset\kappa_{\Markov}}\sqrt{\frac{\alpha}{H'\strong}}$, conditional unary counterfactual loss function induced by true counterfactuals are $\alpha/4$ strongly convex. 
\end{lemma}

\begin{lemma}[Lemma 8.4 of \citet{simchowitz2020improper}] \label{lem:lipschitz}
Let $L_f = 4L\sqrt{H'}\kappa_\nature^2\kappa_\Markov^2\kappa_\Mcontrolset$. For any $\Mcontrol \in \mathcal{M}$ and for $\Tburn \geq \Tmax$, $f_{t}^{\text{pred}}(\Mcontrol)$ is $4 L_{f}$-Lipschitz, $f_{t}^{\text{pred}}\left[\Mcontrol_{t:t-H}\right]$ is $4 L_{f}$ coordinate Lipschitz. Moreover,
$\max_{\Mcontrol \in \mathcal{M}}\left\| \nabla f_t\left(\Mcontrol,\wh \Markov_i,\nature_1(\wh \Markov_i),\ldots,\nature_t(\wh \Markov_i)\right) \right\|_{2} \leq 4 L_{f}$.
\end{lemma}

\end{document}